%% file: Minimax_LQR.tex
\documentclass[12pt]{article}
 \usepackage[margin=1in]{geometry}

\newcommand{\issue}[1]{{\color{red}#1}}

\usepackage{kz_style} 

\usepackage{mathrsfs}
\usepackage{comment}

\makeatletter
\def\@fnsymbol#1{\ensuremath{\ifcase#1\or  \natural \or \dagger\or * \or \ddagger\or
   \mathsection\or \mathparagraph\or \|\or **\or \dagger\dagger
   \or \ddagger\ddagger \else\@ctrerr\fi}}
\makeatother

\title{\LARGE  Policy Optimization Provably Converges to Nash \\ Equilibria in Zero-Sum Linear Quadratic Games}

\begin{document}
\author{Kaiqing Zhang\thanks{Department of Electrical and Computer Engineering \&  Coordinated Science Laboratory, University of Illinois at Urbana-Champaign.} \and Zhuoran Yang\thanks{Department of Operations Research and Financial Engineering, Princeton University.} \and Tamer Ba\c{s}ar$^\natural$}
\date{}
%\date{\today}
\maketitle

\begin{abstract} 
We study  the global convergence of policy optimization for finding the Nash equilibria (NE) in  zero-sum linear quadratic (LQ) games.  
%In particular, we develop three projected PG methods that  are shown to have theoretical convergence guarantees to the Nash equilibrium (NE)  of  LQ games. 
To this end, we first investigate the landscape of LQ games, viewing it as a nonconvex-nonconcave saddle-point  problem in the policy space. Specifically, we show that despite its nonconvexity and nonconcavity, zero-sum LQ games have the property that the stationary point of the objective function  with respect to the linear  feedback control policies constitutes the  NE of the game. Building upon this, we develop three projected \emph{nested-gradient}  methods that are guaranteed to converge to the NE  of the game. Moreover, we show that all of these algorithms enjoy both   globally sublinear and locally  linear convergence rates. Simulation results are also provided to illustrate the satisfactory convergence properties of the algorithms.  
%, which have  been shown to have superior convergence rates even without projection.   
To the best of our knowledge, this work appears to be the first one to investigate the optimization landscape of LQ games, and provably show the convergence of policy optimization  methods to the Nash equilibria. 
Our work serves 
as an initial step toward  understanding the  
theoretical aspects of 
policy-based reinforcement learning algorithms for  zero-sum Markov games in general. 
%We believe  the results  set theoretical foundations  for  developing model-free policy-based reinforcement learning algorithms for zero-sum LQ games.
%On the other hand, we also propose approximate   PI algorithms  using batch data to find the control policy at the NE, with finite-sample analysis. 
%To our knowledge, our work appears to be the first that develop policy-based RL algorithms for  zero-sum LQ games with theoretical convergence guarantees. 
 
%\issue{ 
%In this note, we study the direct policy gradient methods for    zero-sum linear quadratic (LQ) games.  We first provide the concrete form of the gradient, showing that although a nonconvex-nonconcave saddle point problem, the zero-sum LQ game has the property that the stationary point of the objective with respect to the feedback control gains constitutes the Nash equilibrium (NE) of the game. Based on such a landscape property, we develop policy gradient-based algorithms for NE-seeking, which is shown to converge to the NE with linear rate. Building upon this, we then  develop model-free version of the algorithms, i.e., reinforcement learning algorithms, for NE-seeking, which are shown to be efficient in their  sample and computational complexities.  }
\end{abstract}

\input{introduction}

\input{background}

\input{landscape}

\input{algorithm}

\input{theory}

\input{main_proof}

\input{simulations}

\input{conclusion}

\section*{Acknowledgements}
K. Zhang and T. Ba\c{s}ar were supported in part by the US Army Research Laboratory (ARL) Cooperative Agreement W911NF-17-2-0196, and in part by the Office of Naval Research (ONR) MURI
Grant N00014-16-1-2710. Z. Yang was supported by Tencent PhD Fellowship.  The authors would like to thank Renyuan Xu and  Xiangyuan (Rocker) Zhang for the careful reading, and pointing out several typos. The authors also appreciate the  valuable  feedback from the  anonymous NeurIPS reviewers.

%\newpage

\bibliographystyle{ims}
\bibliography{Minimax_LQR} 

\appendix 

\input{appendix_psedo.tex}

\input{appendix_aux_proof}

\input{appendix_simulation}

\end{document}

%% file: introduction.tex
%!TEX root =Minimax_LQR.tex

\section{Introduction}
Reinforcement learning (RL)  \citep{sutton2018reinforcement} 
has achieved sensational   progress recently in several prominent   decision-making problems, e.g., playing the game of Go \citep{silver2016mastering, silver2017mastering} and playing  real-time strategy games  \citep{OpenAI_dota,alphastarblog}. 
%combined with deep neural networks \citep{goodfellow2016deep} has  achieved  astonishing  progress recently in solving  
% large-scale decision-making problems. Some prominent examples  include AlphaGo  and AlphaZero \citep{silver2016mastering, silver2017mastering}    for the game of Go,   OpenAI Five \citep{OpenAI_dota}, and AlphaStar \citep{alphastarblog} for real-time strategy games. 
 Interestingly, all of these problems can be formulated  as zero-sum Markov  games involving  two opposing  players  or teams. Moreover, 
 their algorithmic frameworks  are all based upon \emph{policy optimization} (PO)   methods  such as actor-critic \citep{konda2000actor} and proximal policy optimization (PPO) \citep{schulman2017proximal}, where the policies are parametrized and   iteratively updated.
 Such popularity of PO methods are mainly attributed to the facts that: (i) they are easy to implement and can 
handle high-dimensional and continuous action spaces; (ii) they  can readily incorporate advanced optimization results to facilitate the algorithm design \citep{schulman2017proximal,schulman2015trust,mnih2016asynchronous}.  Moreover, empirically, some observations have shown that PO methods usually converge faster than value-based ones   \citep{mnih2016asynchronous,o2016combining}.

 In contrast to the tremendous empirical  success,  theoretical understanding of policy optimization methods for the multi-agent RL settings \citep{littman1994markov, hu2003nash,conitzer2007awesome, perolat2016use,zhang2018fully,zhang18cdc}, especially the  zero-sum Markov game setting, lags behind.  
Although the convergence of policy optimization algorithms to \emph{locally  optimal}  policies   has been   established in the  classical RL setting  with a \emph{single-agent/player}  \citep{sutton2000policy,  konda2000actor, kakade2002natural,schulman2017proximal,papini2018stochastic,zhang2019policy}, extending those theoretical guarantees to 
\emph{Nash equilibrium} (NE) policies, a common solution concept in game theory  also known as the saddle-point equilibrium (SPE) in the zero-sum setting \citep{bacsar2008h},   suffers from the following two caveats.

First, since the players simultaneously determine their actions in the games, the decision-making  problem faced by each player becomes non-stationary. As a result,  single-agent algorithms fail    to work due to  lack of Markov property \citep{hernandez2017survey}. Second,  with parametrized  policies, the   policy optimization for finding NE in a function space is reduced to solving for NE in the policy parameter space, where the underlying game is in general nonconvex-nonconcave. 
%a minimax saddle-point problem where the decision variables are the  policy parameters and the objective is the cumulative costs of one player, which is  in general  nonconvex-nonconcave.  
Since nonconvex optimization problems are NP-hard \citep{murty1987some} in the worst case, so is finding  NE in  nonconvex-nonconcave saddle-point problems \citep{chen2017robust}. In fact, it has been showcased recently that  vanilla gradient-based algorithms might have cyclic behaviors and fail to converge to any NE \citep{balduzzi2018mechanics, mazumdar2018convergence,adolphs2018local} in both zero-sum and general-sum games.
 
%for these problems, even convergence to a stationary point can be  challenging -- vanilla gradient-based algorithms might have cyclic behaviors and not converge to any NE \citep{balduzzi2018mechanics, mazumdar2018convergence,adolphs2018local}. 

 As an initial attempt in merging the gap between theory and practice, we study the performance of PO methods  on a simple but  quintessential  example of  zero-sum Markov games, namely, zero-sum  linear quadratic (LQ) games.  In LQ games, the system evolves following linear dynamics controlled by both players, while the cost function is quadratically dependent on the states and joint control actions.  Zero-sum LQ games find broad applications in $\cH_{\infty}$-control for   robust control synthesis \citep{bacsar2008h,zhang2018policymixed}, and risk-sensitive control \citep{jacobson1973optimal,whittle1981risk}. 
 In fact, such an LQ setting can be used for studying general  continuous control problems with adversarial disturbances/opponents, by linearizing  the  system of interest  around the operational point \citep{bacsar2008h}. Therefore, developing theory for the LQ setting may provide some insights into the \emph{local} property of the general control settings.     
 Our study is pertinent to the recent efforts on policy optimization for linear quadratic regulator (LQR) problems \citep{fazel2018global,malik2018derivative, tu2018gap}, a single-player    counterpart of  LQ games. As to be shown later, 
 LQ games are more challenging to solve using PO methods, since they are not only nonconvex in the policy space for one player (as LQR), but also nonconcave for the other. Compared to PO for LQR, such nonconvexity-nonconcavity has caused technical difficulties in showing the stabilizing properties along the iterations, an essential requirement for the iterative PO algorithms to be feasible. Additionally,  in contrast to the recent non-asymptotic analyses on gradient  methods for nonconvex-nonconcave saddle-point problems  \citep{nouiehed2019solving}, the objective function lacks smoothness in LQ games, as the main challenge identified in \citep{fazel2018global} for LQR.

% It has been shown that  policy gradient methods converges to the optimal policy in linear quadratic regulator (LQR), a single-agent   counterpart of  zero-sum LQ game \citep{fazel2018global,malik2018derivative, tu2018gap}, However, compared with LQR, zero-sum LQ game has   additional difficulties. {\color{red} Add stability argument, maybe say something about smoothness?} 

To address these technical challenges, we first investigate the optimization   landscape  of LQ games, showing that the stationary point of the objective function constitutes the NE of the game, despite  its nonconvexity and nonconcavity. We then propose three projected \emph{nested-gradient} methods, which separate the updates into two loops with both gradient-based iterations. Such a nested-loop update mitigates the inherent non-stationarity of learning in games. The projection  ensures the stabilizing property of the control along the iterations. The algorithms are guaranteed to converge to the NE, with provably globally sublinear and locally  linear rates.

%Interesting simulation results are also provided to demonstrate their superior convergence property. 
% Our results set theoretical foundations for developing model-free policy-based reinforcement learning algorithms for zero-sum LQ games.

% For zero-sum LQR games,  we provide finite-time convergence analysis of three policy gradient algorithms. Our algorithm is based on a characterization of the optimization landscape, which shows that the stationary point of the objective constitutes the Nash equilibrium policy. 
% Moreover,  to handle the non-stationarity, after each update of the outer maximizaton probem, we  perform gradient for inner minimization   problem to approximate  stationary point. Such a nested update can be viewed as the discretization of  two-timescale approach for finding a local Nash equilibrium \citep{mazumdar2019finding}. More importantly, we show that our algorithm converge to the global Nash equilibrium with {\color{red} global sublinear rate and local linear rate}. 
%  {\color{red} Our results are  based   carefully  handling the stability and smoothness? Main contribution (2)}  
%  
%To the best of our knowledge, {\color{red} for the first time, policy gradient methods are shown to converge to the global Nash equilibrium in a class of zero-sum Markov  games. }

\vspace{8pt}
{\noindent \bf Related Work.}
%{\noindent \bf Solving Zero-Sum Markov Games:}
There is a huge body of literature on applying \emph{value-based} methods to solve  zero-sum Markov games; see, e.g, \citep{littman1994markov, lagoudakis2002value, conitzer2007awesome, perolat2016use, zhang2018finite,  zou2019finite} and the references therein. 
Specially, for the linear quadratic setting, \cite{al2007model} proposed a Q-learning approximate dynamic programming approach.   
In contrast, the study of PO methods for zero-sum Markov games is limited, which are either empirical  without any theoretical guarantees \citep{pinto2017robust}, or developed  only  for the tabular setting  \citep{bowlingrational,banerjee2003adaptive, perolat2018actor,srinivasan2018actor}. 
Within the LQ setting,   our work is related to  the recent work on the global convergence of policy gradient (PG) methods for LQR \citep{fazel2018global,malik2018derivative}. However, our setting is  more challenging since it concerns a saddle-point problem with not only nonconvexity on the minimizer, but also nonconcavity on the maximizer. 

Our work also falls into the realm of solving \emph{nonconvex-(non)concave saddle-point} problems \citep{cherukuri2017saddle,rafique2018non,daskalakis2018limit,mertikopoulos2019optimistic,mazumdar2019finding,jin2019minmax}, which has recently drawn great attention due to the popularity of training generative  adversarial networks (GANs) \citep{heusel2017gans,nagarajan2017gradient,rafique2018non,lu2018understand}. 
However, most of the existing results are either for the nonconvex but concave minimax setting  \citep{grnarova2017online,rafique2018non,lu2018understand}, or only have \emph{asymptotic} convergence results \citep{cherukuri2017saddle,heusel2017gans,nagarajan2017gradient,daskalakis2018limit,mertikopoulos2019optimistic}. 
Two recent pieces of  results on non-asymptotic analyses for solving this problem have been  established under strong assumptions that the objective function is either  weakly-convex and weakly-concave \citep{lin2018solving}, or smooth \citep{sanjabi2018solving,nouiehed2019solving}. However, LQ games  satisfy neither of these  assumptions. 
In addition, even asymptotically, basic gradient-based approaches may not  converge to (local) Nash equilibria \citep{mazumdar2019finding,jin2019minmax}, not even  to stationary points, due to the oscillatory behaviors \citep{mazumdar2018convergence}.
In contrast to \cite{mazumdar2019finding,jin2019minmax}, our results show the \emph{global convergence} to \emph{actual NE} (instead of any surrogate as \emph{local minimax} in \cite{jin2019minmax}) of the game. 

\vspace{8pt}
\noindent \textbf{Contribution.}
Our contribution is two-fold:  i) we investigate the optimization landscape of zero-sum LQ games in the parametrized feedback control policy space, showing its desired    property that  stationary points constitute the Nash equilibria; ii) we develop  projected nested-gradient methods that are proved to converge to the NE with globally sublinear and locally linear rates. We also provide several interesting simulation findings on solving this problem with PO methods.  To the best of our knowledge,  for the first time, policy-based methods are shown to converge to the \emph{global Nash equilibria} in a class of zero-sum Markov  games, and also with convergence rate guarantees. 
%Our work serves as an initial step of understanding the  theoretical aspects of policy-based reinforcement learning algorithms for zero-sum Markov games in general.

%two-fold: i) geometry of the problem; ii) analyze the convergence of the algorithm, finite-time analysis, novelty. 
%Emphasize: for the first time, we show the convergence of policy gradient algorithm to Nash. 
%
% In sum, policy optimization in zero-sum LQ game captures three   fundamental  challenges: 
% \begin{enumerate}
% 	\item [(ii)] the non-stationary environment in multi-agent decision making, 
% 	\item [(ii)] the  lack of convex-concave and smooth structure in minimax optimization, 
% 	\item [(iii)] the  algorithmic instability of gradient-based algorithms for zero-sum games.
% \end{enumerate} 
% 
%
%\noindent XXXXXXXXXXXXXXXXXXXXXXXXXXXXXXXXXX

%\subsection{Guideline}
%XXXXXXXXXXXXXXXXXXXXXXXXXXXXXXXXX
%
%Key Words: 
%
%\noindent \textbf{Motivation:}
%1. Policy gradient for Nash equilibrium finding is non-existing; this motivates LQ.
%
%\noindent \textbf{Challenge:}
%2. For LQ games policy gradient is nonconvex-nonconcave: Gradient for nonconvex-nonconcave is NP hard;
%3. Compared with PG for LQR: stability argument.
%4. Hardness from optimization perspective: objective non-smoothness
%
%\noindent \textbf{Methodology:}
%5. We first study the geometry of LQ games; then propose   three types of PG methods that are guaranteed to converge to NE with rate;

\vspace{8pt}
{\noindent \bf Notation.}
For any vector $x\in\RR^n$ and matrix $Y\in\RR^{m\times n}$, 
we use $\|x\|$, $\|Y\|$, and $\|Y\|_F$ to denote the Euclidean norm of $x$,  the induced  $2$-norm,  and the  Frobenius norm of $Y$, respectively. We use $\vect(Y)\in\RR^{mn}$ to denote the vectorization of the matrix $Y$.
For any symmetric  matrix $M\in\RR^{n\times n}$, we use $M\geq 0$  and $M> 0$ to denote the nonnegative-definiteness and positive 
definiteness of $M$, respectively. 
For any set $\cS$, we use $\cS^c$ to denote the complement set of $\cS$. For any square matrix $A$, we use $\rho(A)$ to denote its spectral radius, i.e., the largest absolute value of its  eigenvalues, of matrix $A$.  For any matrix $M\in\RR^{m\times n}$, we use $\sigma_{\min}(M)$ and $\sigma_{\max}(M)$ to denote its smallest and largest singular values, respectively. 
For any real symmetric matrix $M\in\RR^{n\times n}$, we use $\lambda_{\min}(M)$ and $\lambda_{\max}(M)$ to denote its smallest and largest eigenvalues, respectively. 
We use $\otimes$ to denote the Kronecker product. For any positive integer $m$, we use $[m]$ to denote the set of integers $\{1,\cdots,m\}$.
We use 
%$\bm{0}_{m\times n}$ to denote the all-zero matrix with dimension $m\times n$, and 
$\Ib$ to denote the identity matrix with proper dimensions.

%% file: background.tex
\section{Background}
Consider a zero-sum LQ game, where the system dynamics   are  characterized by a    linear dynamical system 
\$
x_{t+1}=Ax_t+Bu_t+Cv_t, 
\$
where the system state is  $x_t\in\RR^d$, the control inputs  of players $1$ and $2$ are   
$u_t\in\RR^{m_1}$ and $v_t\in\RR^{m_2}$, respectively. The matrices satisfy $A\in\RR^{d\times d}$, $B\in\RR^{d\times m_1}$, and $C\in\RR^{d\times m_2}$. The objective of  player $1$ (player $2$) is to minimize (maximize) the infinite-horizon value function,
\#\label{equ:minimax_def}
\inf_{\{u_t\}_{t\geq 0}}\sup_{\{v_t\}_{t\geq 0}}\quad\EE_{x_0\sim\cD}\bigg[\sum_{t=0}^\infty c_t(x_t,u_t,v_t)\bigg]=\EE_{x_0\sim\cD}\bigg[\sum_{t=0}^\infty (x_t^\top Qx_t+u_t^\top R^uu_t-v_t^\top R^vv_t)\bigg],
\#
where $x_0\sim \cD$ is the initial state  drawn from a distribution $\cD$, the matrices  $Q\in\RR^{d\times d}$, $R^u\in\RR^{m_1\times m_1}$, and $R^v\in\RR^{m_2\times m_2}$ are all positive definite. 
%In the $\cH$-infinity control problem, for example,  we can choose  $R^v=\bar{\gamma}^2\Ib$ with $\bar{\gamma}$ being the upper bound on the desired $\ell_2$ gain disturbance attenuation. 
%In addition, we have the following inequality by definition
%\#\label{equ:weak_duality}
%\inf_{\{u_t\}}\sup_{\{v_t\}}~~\EE_{x_0\sim\cD}\bigg[\sum_{t=0}^\infty c_t(x_t,u_t,v_t)\bigg]\geq \sup_{\{v_t\}}\inf_{\{u_t\}}~~\EE_{x_0\sim\cD}\bigg[\sum_{t=0}^\infty c_t(x_t,u_t,v_t)\bigg],
%\#
%where we refer the left and right-hand  sides of \eqref{equ:weak_duality} as the \emph{upper-value} and \emph{lower-value} of the game \eqref{equ:minimax_def}, respectively. If the upper and  lower values are equal, we refer to it as the \emph{value} of the game. 
If the solution to \eqref{equ:minimax_def} exists and the infimum and supremum in  \eqref{equ:minimax_def} can be interchanged, we refer to the solution value in \eqref{equ:minimax_def} as the \emph{value} of the game.

To investigate  the property of the solution to   \eqref{equ:minimax_def}, we first introduce the generalized  algebraic Riccati equation (GARE) as follows 
\small
\#\label{equ:P_GARE}
P^*=A^\top P^*A +Q-
\begin{bmatrix}
A^\top P^*B & A^\top P^*C
\end{bmatrix}
\begin{bmatrix}
R^u+B^\top P^*B & B^\top P^*C\\
C^\top P^*B & -R^v+C^\top P^*C  
\end{bmatrix}^{-1}
\begin{bmatrix}
B^\top P^*A \\
C^\top P^*A 
\end{bmatrix},
\#
\normalsize
where $P^*$ denotes   the minimal  non-negative definite  solution to \eqref{equ:P_GARE}. 
Under some standard assumptions to be specified shortly, the value exists and can be  characterized by a matrix $P^* \in \RR^{d\times d} $ \citep{bacsar2008h} satisfying 
% In particular, for any $x_0\in\RR^d$, 
\#
\forall~~ x_0\in\RR^d,\quad 
%V^*(x_0) &= x_0^\top P^* x_0, \label{equ:def_value}\\
x_0^\top P^* x_0 =\inf_{\{u_t\}_{t\geq 0}}\sup_{\{v_t\}_{t\geq 0}}~\sum_{t=0}^\infty c_t(x_t,u_t,v_t)= \sup_{\{v_t\}_{t\geq 0}}\inf_{\{u_t\}_{t\geq 0}}~\sum_{t=0}^\infty c_t(x_t,u_t,v_t). \label{equ:strong_duality}
\#
Moreover, 
there exists a pair of linear feedback stabilizing polices that attain the  equality in \eqref{equ:strong_duality}, i.e., the optimal actions  $\{ u_t^*\}_{t\geq 0}$ and $\{v_t ^*\}_{t\geq 0}$ in \eqref{equ:minimax_def} can be written as  
%that make the equality in \eqref{equ:strong_duality} hold,  which have the following form
\#\label{equ:u_v_t_gen}
u_t^*=-K^*x_t,\qquad v_t^*=-L^*x_t,
\#
where $K^*\in\RR^{m_1\times d}$ and $L^*\in\RR^{m_2\times d}$ are called the control gain matrices for the minimizer and the maximizer, respectively.
The values of $K^*$ and $L^*$ can be given by  
\#
K^*=&[R^u+B^\top P^*B-B^\top P^*C(-R^v+C^\top P^*C)^{-1}C^\top P^*B]^{-1}\notag\\
&\quad\times [B^\top P^*A-B^\top P^*C(-R^v+C^\top P^*C)^{-1}C^\top P^*A],\label{equ:K_GARE}\\
L^*=&[-R^v+C^\top P^*C-C^\top P^*B(R^u+B^\top P^*B)^{-1}B^\top P^*C]^{-1}\notag\\
&\quad\times [C^\top P^*A-C^\top P^*B(R^u+B^\top P^*B)^{-1}B^\top P^*A].\label{equ:L_GARE}
\#
Since the controller pair $(K^*,L^*)$ achieves the value \eqref{equ:strong_duality} for any $x_0$, the value of the game is thus $\EE_{x_0\sim\cD} \big(x_0^\top P^* x_0 \big)$. 
Now we introduce the following  assumption that guarantees the arguments above to hold. 

\begin{assumption}\label{assum:invertibility}
The following conditions hold: i) 
	there exists a minimal positive definite solution $P^*$ to the GARE \eqref{equ:P_GARE}  that satisfies $R^v-C^\top P^*C>0$; ii) $L^*$  satisfies $Q-(L^*)^\top R^v L^* >0$. 
%	 the control gain pair $(K^*,L^*)$ defined in \eqref{equ:K_GARE}-\eqref{equ:L_GARE} is stabilizing,  i.e., $\rho(A-BK^*-CL^*)<1$, and   $L^*$  satisfies $Q-(L^*)^\top R^v L^* >0$. 
%	, where $R^{v,*}\in\RR^{m_2\times m_2}$ is a positive definite matrix. 
\end{assumption}

The condition i) in  
Assumption \ref{assum:invertibility} is a standard sufficient condition  that ensures the existence of the value of the game    \citep{bacsar2008h,al2007model,stoorvogel1994discrete}.    
In addition,  condition ii) leads to the saddle-point property of the control pair $(K^*,L^*)$, i.e., the controller  sequence $(\{u_t^*\}_{t\geq 0},\{v_t^*\}_{t\geq 0})$ generated by \eqref{equ:u_v_t_gen} constitutes the NE of the game \eqref{equ:minimax_def}, which is also unique. 
We formally state the arguments regarding \eqref{equ:P_GARE}-\eqref{equ:L_GARE}  in the following  lemma, whose proof is deferred to \S\ref{subsec:proof_saddle_point_formal}.
 
\begin{lemma}\label{lemma:saddle_point_formal}
	Under Assumption \ref{assum:invertibility} i), for any $x_0\in\RR^d$, the value of the minimax game 
	\#\label{equ:def_state_game}
	\inf_{\{u_t\}_{t\geq 0}}\sup_{\{v_t\}_{t\geq 0}}\quad\sum_{t=0}^\infty c_t(x_t,u_t,v_t)
	\#
	exists, i.e.,  
	\eqref{equ:strong_duality} holds, and $(K^*,L^*)$ is stabilizing. Furthermore, under Assumption \ref{assum:invertibility} ii), the  controller  sequence $(\{u_t^*\}_{t\geq 0},\{v_t^*\}_{t\geq 0})$ generated from \eqref{equ:u_v_t_gen}  constitutes the saddle-point of \eqref{equ:def_state_game},  i.e., the NE of the game, and it is  unique.   
\end{lemma}
  
Lemma \ref{lemma:saddle_point_formal} implies that the solution to \eqref{equ:minimax_def} can be found by searching for $(K^*,L^*)$ in the matrix space $\RR^{m_1\times d}\times \RR^{m_2\times d}$,  given by  \eqref{equ:K_GARE}-\eqref{equ:L_GARE} for  some $P^*>0$ satisfying \eqref{equ:P_GARE}. 
%that satisfies  \eqref{equ:P_GARE} and \eqref{equ:u_v_t_gen}-\eqref{equ:L_GARE} with some $P^*>0$.   
Next, we aim to develop policy optimization  methods that provably converge to the NE $(K^*,L^*)$.

%% file: landscape.tex
\section{Policy Gradient and Landscape}\label{sec:landscape}
%We focus on  the control with  the form of  state-feedback. 
By Lemma \ref{lemma:saddle_point_formal}, 
we focus on finding the state feedback policies of players parameterized by   $u_t=-Kx_t$, and  $v_t=-Lx_t,
$ such that $\rho(A-BK-CL)<1$. Accordingly, 
we denote the  corresponding expected cost in  \eqref{equ:minimax_def} as 
%value with initial state $x_0$ by $V_{K,L}(x_0)$, and its  expected value over $x_0$ by $\cC(K,L)$, respectively, which are defined as
\$
%V_{K,L}(x_0)
%:=&~\sum_{t=0}^\infty \big[x_t^\top Qx_t+(Kx_t)^\top R^u(Kx_t)-(Lx_t)^\top R^v(Lx_t)\big],\label{equ:def_V_K_L}\\
\cC(K,L)
:=&~\EE_{x_0\sim\cD}\bigg\{\sum_{t=0}^\infty \big[x_t^\top Qx_t+(Kx_t)^\top R^u(Kx_t)-(Lx_t)^\top R^v(Lx_t)\big]\bigg\}.
\$
Also, define $P_{K,L}$ as the unique    solution to the   Lyapunov equation 
\#\label{equ:P_KL_Sol}
P_{K,L}=Q+K^\top R^u K-L^\top R^v L+(A-BK-CL)^\top P_{K,L}(A-BK-CL). 
\#
Then for any \emph{stablilizing}  control pair $(K,L)$, it follows that 
$
\cC(K,L)=\EE_{x_0\sim\cD}\big(x_0^\top P_{K,L} x_0\big). 
$
Also,  we  define $\Sigma_{K,L}$ as the state correlation matrix, i.e., $
\Sigma_{K,L}:=\EE_{x_0\sim\cD}\sum_{t=0}^\infty x_tx_t^\top$. 
Our goal is to find the NE  $(K^*,L^*)$ using  policy optimization  methods that solve the following minimax problem 
\#\label{equ:nonconvex_concave_def}  
\min_K\max_L ~~\cC(K,L)  
\# 
%In other words, we aim to find the Nash equilibrium control gain $(K^*,L^*)$, the saddle point of the expected value function $\cC(K,L)$, 
such that for any  
%non-negative definite 
 $K\in\RR^{m_1\times d}$ and $L\in\RR^{m_2\times d}$, $
\cC(K^*,L) \leq \cC(K^*,L^*)\leq \cC(K,L^*)$.

As has been recognized in \cite{fazel2018global}  that the LQR problem is nonconvex with respect to (w.r.t.) the control gain $K$, 
we note that in general, for some  given $L$ (or $K$), the minimization (or maximization) problem is not convex (or concave). 
%Thus, the common and standard results on  convex-concave saddle-point problems \citep{fan1953minimax,sion1958general,nemirovski2004prox,nedic2009subgradient} do not apply here. 
This has in fact caused the main challenge for the design of equilibrium-seeking algorithms for zero-sum LQ games. 
We formally state this in the following lemma, which is proved in \S\ref{subsec:proof_nonconvex_nonconcave}.

%the expected value $\cC(K,L)$ is nonconvex with respect to $K$ and nonconcave with respect to  $L$, which makes finding the saddle-point using gradient-based methods extremely challenging. 
%In fact, it has been shown in \cite{fazel2018global} that the single-agent linear quadratic regulator (LQR) problem is  nonconvex w.r.t. the parameter $K$, which justifies the nonconvexity of $\cC(K,L)$ w.r.t. $K$, since the minimization w.r.t. $K$ is essentially an LQR problem  for any given $L$. By similar argument, $\cC(K,L)$ is nonconcave w.r.t. $L$, as stated in the following lemma. 

{
\begin{lemma}[Nonconvexity-Nonconcavity of $\cC(K,L)$]\label{lemma:nonconvex_nonconcave}
Define a subset $\underline{\Omega}\subset \RR^{m_2\times d}$ as
\#\label{equ:def_Omega_underline}
\underline{\Omega}:=\big\{L\in\RR^{m_2\times d} \colon  Q-L^\top R^v L>0\big\}.
\#
Then there exists  $L\in\underline{\Omega}$ such that $\min_{K}\cC(K,L)$ is a  nonconvex minimization problem; there exists  $K$ such that $\max_{L\in\underline{\Omega}}\cC(K,L)$ is a nonconcave maximization problem. 
\end{lemma}

}

To facilitate the algorithm design, we  establish the explicit expression of the policy gradient w.r.t. the parameters $K$ and $L$  in the following lemma, with a proof provided in \S\ref{subsec:proof_policy_grad}.

\begin{lemma}[Policy Gradient Expression]\label{lemma:policy_grad}
	The policy gradients of $\cC(K,L)$ have the form
\#
\nabla_K \cC(K,L)=&~2[(R^u+B^\top P_{K,L}B)K-B^\top P_{K,L}(A-CL)]\Sigma_{K,L}\label{equ:policy_grad_K_form}\\
\nabla_L \cC(K,L)=&~2[(-R^v+C^\top P_{K,L}C)L-C^\top P_{K,L}(A-BK)]\Sigma_{K,L}.\label{equ:policy_grad_L_form}
\#
\end{lemma}

%At the Nash equilibrium, the tuple $(K^*,L^*)$ constitutes a saddle point of the expected value function $\cC(K,L)$, i.e., for any 
%%non-negative definite 
%matrices $K\in\RR^{m_1\times d}$ and $L\in\RR^{m_2\times d}$
%\#\label{equ:saddle_point}
%\cC(K^*,L) \leq \cC(K^*,L^*)\leq \cC(K,L^*).
%\#

To study the landscape  of this nonconvex-nonconcave problem, 
% shown in Lemma \ref{lemma:nonconvex_nonconcave},  it may still be possible to attain the Nash equilibrium with gradient-based methods. In particular, 
 we first examine the  property of the stationary points of $\cC(K,L)$, which are  the points that gradient-based  methods converge to. 

\begin{lemma}[Stationary Point Property]\label{lemma:stationary_point}
	For a stabilizing control pair $(K,L)$, i.e., $\rho(A-BK-CL)<1$,  suppose $\Sigma_{K,L}$ is full-rank and $(-R^v+C^\top P_{K,L}C)$ is invertible. If      
	$
	\nabla_K \cC(K,L)=\nabla_L \cC(K,L)=0
	$ 
	and the induced matrix $P_{K,L}$ defined in \eqref{equ:P_KL_Sol} is positive definite, then 
	  $(K,L)$ constitutes the  control gain pair at the  Nash equilibrium.
\end{lemma}

Lemma \ref{lemma:stationary_point}, proved in  \S\ref{subsec:proof_stationary_point}, shows that  the stationary point of $\cC(K,L)$ suffices to characterize the NE of the game under certain conditions. 
%, provided that the corresponding matrix $P_{K,L}>0$, and   the state correlation matrix $\Sigma_{K,L}$ is full-rank. 
In fact, for $\Sigma_{K,L}$ to be full-rank, 
it suffices to let $\EE_{x_0\sim\cD} x_0x_0^\top$ be full-rank, i.e., to use a random initial state $x_0$ whose covariance matrix is non-degenerate. 
%,  to ensure $\Sigma_{K,L}$ is full-rank. 
This can be easily satisfied in practice.  
%Moreover, by Lemma \ref{lemma:stationary_point}, if the algorithm  starts from a stable control $(K,L)$ that ensures  $(-R^v+C^\top P_{K,L}C)$, and 
%However, ensuring the nonnegative definiteness of $P_{K,L}$ for any stationary point policy $(K,L)$ is non-trivial, which needs to be taken into account in the algorithm design.  

%% file: algorithm.tex
\section{Policy Optimization Algorithms}\label{sec:algorithms}
In this section, we propose three PO methods, based on policy gradients, to find the global NE of the  LQ game. 
In particular, we develop \emph{nested-gradient} (NG)  methods, which first solve   the inner  optimization by policy-gradient methods, and then use the stationary-point solution to perform gradient-update for  the outer optimization. 
One way to solve for the NE is to directly address the minimax problem \eqref{equ:minimax_def}. 
Success of this procedure, as pointed out in    \cite{fazel2018global} for LQR,  
requires the stability guarantee of the system along the outer policy-gradient updates. 
%\footnote{The  necessity of establishing stability of the algorithm updates was discussed via personal  communications. The stability argument was missing previously and has been  supplemented in the latest  version of \cite{fazel2018global}.}.  
However, unlike LQR,   it is not clear so far if there exists a stepsize and/or condition on $K$  that ensures such stability of the system along 
the outer-loop policy-gradient update.  Instead, if we solve  the maximin problem, which has the same value as \eqref{equ:minimax_def} (see Lemma \ref{lemma:saddle_point_formal}), then a simple projection step on the iterate $L$,  as to be shown later, can guarantee  the stability of the updates. 
Therefore, we 
aim to solve   $
\max_L\min_K ~~\cC(K,L)$.

For some given $L$,  the inner minimization problem becomes an LQR problem with equivalent cost matrix $\tilde Q_L=Q-L^\top R^v L$, and state transition matrix $\tilde{A}_L=A-CL$. Motivated by  \cite{fazel2018global}, we propose to find the stationary point of the inner problem, 
%i.e., to find $K$ such that $\nabla_K {\cC}(K,L)=0$ for given $L$,  
since the stationary point suffices  to be the global optimum under certain conditions (see Corollary $4$ in \cite{fazel2018global}). Let the stationary-point solution be $K(L)$. 
By setting $\nabla_K {\cC}(K,L)=0$  and  by Lemma \ref{lemma:policy_grad}, we have     
 \#\label{equ:KL_def}
 K(L)=(R^u+B^\top P_{K(L),L}B)^{-1}B^\top P_{K(L),L}(A-CL). 
 \#
%  which uses the full-rankness of $\Sigma_{K,L}$. 
  We then substitute 
\eqref{equ:KL_def}  into  \eqref{equ:P_KL_Sol} to obtain the Riccati equation for the inner  problem:
\#\label{equ:Inner_Riccati}
P_{K(L),L}=\tilde Q_L+\tilde{A}_L^\top P_{K(L),L}\tilde{A}_L -\tilde{A}_L^\top P_{K(L),L} C (R^u+B^\top P_{K(L),L}B)^{-1}C^\top P_{K(L),L}\tilde{A}_L. 
\#
%where we let $\tilde Q_L=Q-L^\top R^v L$ and $\tilde{A}_L=A-CL$. 
%Note that  the algorithm should only search over the $L$ that admits a positive definite  $P_{K(L),L}$ of \eqref{equ:Inner_Riccati}, since at the Nash equilibrium  $P_{K^*,L^*}> 0$ by Assumption \ref{assum:invertibility}. 
%Note that $P_{K(L),L}$ is the nonnegative definite solution to the Bellman equation \eqref{equ:P_KL_Sol} given $(K(L),L)$.   
% Also n
Note that as in  \cite{fazel2018global}, $K(L)$ can be obtained  using gradient-based algorithms.   For example, one can use the basic  policy gradient update in the inner-loop, i.e.,
  \#\label{equ:inner_gd}
 K' = K-\alpha \nabla _{K} \cC(K, L)=K-2\alpha[(R^u+B^\top P_{K,L}B)K-B^\top P_{K,L}\tilde{A}_L]\Sigma_{K,L}, 
 \#
 where $\alpha>0$ denotes the stepsize,  $P_{K,L}$ denotes the solution to \eqref{equ:P_KL_Sol} for given $(K,L)$, and $\nabla_K \cC(K,L)$ denotes the partial gradient w.r.t. $K$ given in \eqref{equ:policy_grad_K_form}. 
 Alternatively,   one can also use  the approximate second-order information to accelerate the update, which yields   the   \emph{natural} policy gradient  update   
 \#\label{equ:inner_natural_gd}
 K' = K-\alpha \nabla _{K} \cC(K, L)\Sigma_{K,L}^{-1}=K-2\alpha[(R^u+B^\top P_{K,L}B)K-B^\top P_{K,L}\tilde{A}_L], 
 \#
 that utilizes the Fisher's information, and the  \emph{Gauss-Newton}  update 
  \#\label{equ:inner_gauss_newton}
 K' &= K-\alpha (R^u+B^\top P_{K,L} B)^{-1}\nabla _{K} \cC(K, L)\Sigma_{K,L}^{-1}\notag\\
 &=K-2\alpha (R^u+B^\top P_{K,L} B)^{-1}[(R^u+B^\top P_{K,L}B)K-B^\top P_{K,L}\tilde{A}_L]. 
 \#

Suppose $K(L)$ in \eqref{equ:KL_def} can be obtained, regardless of the algorithms used. Then, we  substitute  $K(L)$ back to the gradient of  $\tilde{\cC}(L):= {\cC}(K(L),L)$ to obtain the \emph{nested-gradient}:  
\$
 \nabla_L \tilde{\cC}(L) & =\nabla_L {\cC}(K(L),L)\\
%=&~2\Big\{(-R^v+C^\top P_{K(L),L}C)L-C^\top P_{K(L),L}\big[A-B(R^u+B^\top P_{K(L),L}B)^{-1}B^\top P_{K(L),L}(A-CL)\big]\Big\}\Sigma_{K(L),L}\\
&  = 2\Big\{\big[-R^v+C^\top P_{K(L),L}C-C^\top P_{K(L),L}B(R^u+B^\top P_{K(L),L}B)^{-1}B^\top P_{K(L),L}C\big]L\\
&\qquad -C^\top P_{K(L),L}\big[A-B(R^u+B^\top P_{K(L),L}B)^{-1}B^\top P_{K(L),L}A\big]\Big\}\Sigma_{K(L),L}, 
\$
where  $\nabla_L \tilde{\cC}(L)$ denotes the  nested-gradient for the outer-loop. 
Note that the stationary-point condition of the outer-loop that  $\nabla_L \tilde{\cC}(L)=0$ is identical to that of $\nabla_L {\cC}(K(L),L)=0$, since 
\$
\nabla_L \tilde{\cC}(L)=\nabla_L {\cC}(K(L),L)+\nabla_L K(L)\cdot\nabla_K {\cC}(K(L),L)=\nabla_L {\cC}(K(L),L),
\$ 
where $\nabla_K {\cC}(K(L),L)=0$ by definition of $K(L)$. Thus, the convergent point $(K(L),L)$ that makes $\nabla_L \tilde{\cC}(L)=0$ satisfy  both conditions  $\nabla_K {\cC}(K(L),L)=0$ and $\nabla_L {\cC}(K(L),L)=0$, which implies from Lemma \ref{lemma:stationary_point} that the convergent control pair $(K(L),L)$ constitutes  the  Nash equilibrium.

Thus, we propose   the following  projected nested-gradient update  in the outer-loop to find the pair $(K(L),L)$:
\begin{flalign}
&{\rm \textbf{Projected Nested-Gradient:}}\qquad \qquad\quad~~  L'=\PP^{GD}_{\Omega}[L+	\eta \nabla_L \tilde{\cC}(L)],&\label{equ:algorithm_nested_L}
\end{flalign} 
where $\Omega$ is some convex set in $\RR^{m_2\times d}$, and  $\PP^{GD}_{\Omega}[\cdot]$ is the projection  operator onto $\Omega$  that is    defined~as
\#\label{equ:def_proj_GD}
\PP^{GD}_{\Omega}[\tilde L]=\argmin_{L\in\Omega} ~\tr\Big[\big(L-\tilde L\big)\big(L-\tilde L\big)^\top\Big],
\#
i.e., the minimizer of the  distance between $\tilde L$ and $L$ in Frobenius norm. 
It is assumed that the set $\Omega$ is large enough such that it contains  the Nash equilibrium $(K^*,L^*)$. 
  Under  Assumption \ref{assum:invertibility}, there exists a constant $\zeta$  with  $0<\zeta<\sigma_{\min}(\tilde Q_{L^*})$, with  
%$
%Q-(L^*)^\top R^v L^*\geq \zeta\cdot \Ib
%$. Thus, we can define 
  one example of  $\Omega$  that serves the purpose is  
\#\label{equ:def_Omega}
\Omega:=\big\{L\in\RR^{m_2\times d}\given  Q-L^\top R^v L\geq \zeta\cdot \Ib\big\},
\#
which contains $L^*$ at the NE. Thus, the projection  does not exclude the convergence to the NE. The following lemma, proved in  \S\ref{subsec:proof_convex_Omega},  shows that  $\Omega$ is indeed convex and compact.  
%The proof is provided in \S\ref{subsec:proof_convex_Omega}. 

\begin{lemma}\label{lemma:convex_Omega}
	The subset $\Omega\subset  \RR^{m_2\times d}$ defined in \eqref{equ:def_Omega} is a convex and compact set. 
\end{lemma}

The projection is mainly for the purpose of theoretical analysis, and is not necessarily used in the  implementation of the algorithm in practice. In fact, the   simulation results in \S\ref{sec:simulations} show that the algorithms converge without this projection in many   cases.  
Such a projection is also implementable,  since the set to project on is convex, and the constraint is directly imposed  on the policy parameter iterate $L$ (not on some derivative quantities, e.g.,  $P_{K(L),L}$).  Similarly, we  develop the following projected natural  nested-gradient update: 
\begin{flalign}
&{\rm \textbf{Projected Natural Nested-Gradient:}}\qquad  ~~  L'=\PP_{\Omega}^{NG}\big[L+	\eta \nabla_L \tilde{\cC}(L)\Sigma_{K(L),L}^{-1}\big],&\label{equ:algorithm_natural_nested_L}
\end{flalign} 
where the projection operator $\PP_{\Omega}^{NG}[\cdot]$ for natural nested-gradient is defined as
\#\label{equ:def_proj_NG}
\PP^{NG}_{\Omega}[\tilde L]=\argmin_{\check{L}\in\Omega} ~\tr\Big[\big(\check{L}-\tilde L\big)\Sigma_{K(L),L} \big(\check{L}-\tilde L\big)^\top\Big]. 
\#
Here a weight matrix $ \Sigma_{K(L),L}$ is added   for the convenience of subsequent  theoretical analysis.  We note that the weight  matrix $\Sigma_{K(L),L}$ depends on the current iterate $L$ in \eqref{equ:algorithm_natural_nested_L}.   

Moreover, we can  develop  the projected nested-gradient algorithm with preconditioning matrices. For example, if we assume that 
%both $R^v- C^\top P_{K,L(K)}C$ and 
$R^v - C^\top P_{K(L),L}C$ is positive definite, and 
define
\#
%W_{K} &= R^u + B^\top P_{K,L(K)} B + B^\top  P_{K,L(K)} C ( R^v- C^\top P_{K,L(K)}C)^{-1}C^\top P_{K,L(K)}B\label{eq:define_wkk},\\
W_{L} 
%&= R^v - C^\top P_{K(L),L}C + C^\top  P_{K(L),L}B (R^u+B^\top P_{K(L),L} B)^{-1} B^\top P_{K(L),L}C\\
&=R^v - C^\top \big[P_{K(L),L}-P_{K(L),L}B (R^u+B^\top P_{K(L),L} B)^{-1} B^\top P_{K(L),L}\big]C,\label{eq:define_wll}
\#
%where 
%\$
%\tilde P_{K(L),L}=P_{K(L),L}-P_{K(L),L}B (R^u+B^\top P_{K(L),L} B)^{-1} B^\top P_{K(L),L},
%\$ 
 we obtain the  projected \emph{Gauss-Newton nested-gradient} update 
\begin{flalign}
%&{\rm \textbf{Gauss-Newton NG for minimax:}}\qquad \qquad K'=K-\eta W_{K}^{-1}\nabla_K \tilde{\cC}(K)\Sigma_{K,L(K)}^{-1}&\label{equ:algorithm_nested_precond_K}\\
&{\rm \textbf{Projected Gauss-Newton Nested-Gradient:}}\notag\\
 &\qquad\qquad\qquad\qquad\qquad\qquad\qquad L'=\PP^{GN}_{\Omega}\big[L+	\eta W_{L}^{-1}\nabla_L \tilde{\cC}(L)\Sigma_{K(L),L}^{-1}\big],&\label{equ:algorithm_nested_precond_L}
\end{flalign}
where the projection operator $\PP^{GN}_{\Omega}[\cdot]$ is defined as
\#\label{equ:def_proj_GN}
\PP^{GN}_{\Omega}[\tilde L]=\argmin_{\check{L}\in\Omega} ~\tr\Big[W_L^{1/2}\big(\check{L}-\tilde L\big)\Sigma_{K(L),L} \big(\check{L}-\tilde L\big)^\top W_L^{1/2}\Big]. 
\#
The weight matrices $\Sigma_{K(L),L}$ and $W_L$ both depend on the current iterate $L$ in \eqref{equ:algorithm_nested_precond_L}.  
%This is referred to  as a variant of the projected \emph{Gauss-Newton nested-gradient} update. 

Based on the updates above, it is straightforward to develop model-free versions of NG algorithms using sampled data. In particular,  we propose to first use zeroth-order optimization algorithms to find the stationary point of the inner LQR problem after a finite number of iterations.  
Since the 
%preconditioning matrices in the 
Gauss-Newton  update cannot be estimated  via sampling, only the PG and natural PG updates are converted to model-free versions.  
The approximate   stationary point is then substituted into the outer-loop to perform the projected (natural) NG updates.  
Details of our model-free version updates are provided in   \S\ref{sec:append_alg}. Building upon our theory next, high-probability convergence guarantees for these model-free counterparts can be established as in  the LQR setting in   \cite{fazel2018global}.  

%% Note that the projection operators for natural NG and Gauss-Newton NG  depend 
%We first provide the subroutine  in Algorithm \ref{alg:est_grad_corre} for estimating the gradient and the correlation matrix evaluated at $K$, for any given $L$. 
%% for the inner and outer loops  in \issue{Algorithms \ref{XXX} and \ref{XXX}}, respectively. 
% Then,  we  summarize  the  model-free  updates for finding $K(L)$ in Algorithm \ref{alg:model_free_inner_NPG}, and the model-free projected NG algorithm   in {Algorithm \ref{alg:model_free_outer_NPG}}. Their high-probability convergence guarantees can be  established following the behavior of the model-based counterparts, as for the LQR setting in   \citep{fazel2018global}. 

%% file: theory.tex
\section{Convergence Results}\label{sec:theory}  
We start by showing the convergence    results  for the inner optimization problem as follows, which establishes the \emph{globally linear} convergence rates of the inner-loop policy gradient  updates in  \eqref{equ:inner_gd}-\eqref{equ:inner_gauss_newton}. 
%The following proposition shows that the policy gradient methods \eqref{equ:inner_gd}-\eqref{equ:inner_gauss_newton} all  converge to the solution $K(L)$ (see definition in \eqref{equ:KL_def}) of the inner problem for given $L$, with linear rates.  

\begin{proposition}[Global Convergence Rate of Inner-Loop Update]\label{prop:inner_full_grad_conv}
Suppose $\EE_{x_0\sim\cD}x_0x_0^\top>0$ and  Assumption \ref{assum:invertibility} holds.  
	For any $L\in\underline{\Omega}$, where  $\underline{\Omega}$ is defined in \eqref{equ:def_Omega_underline},  
it follows that: i) the inner-loop LQR problem always admits a solution,  with a positive definite $P_{K(L),L}$ and a stabilizing control pair $(K(L),L)$; ii)  there exists a constant stepsize $\alpha>0$ for   each of 
the  updates \eqref{equ:inner_gd}-\eqref{equ:inner_gauss_newton} such that the generated  control pair sequences $\{(K_\tau,L)\}_{\tau\geq 0}$  are always stabilizing; 
iii) the  updates \eqref{equ:inner_gd}-\eqref{equ:inner_gauss_newton} enables the convergence of the cost  value sequence $\{\cC(K_\tau,L)\}_{\tau\geq 0}$   to the optimum $\cC(K(L),L)$   with linear rate. 
\end{proposition}

Proof of 
Proposition \ref{prop:inner_full_grad_conv}, deferred to \S\ref{subsec:proof_inner_full_grad_conv}, 
%establishes the \emph{global linear} convergence rates of the inner-loop policy gradient  updates \eqref{equ:inner_gd}-\eqref{equ:inner_gauss_newton}, whose proof 
primarily  follows that for  Theorem $7$ in \cite{fazel2018global}. 
%is consistent with the result in   \cite{fazel2018global} (see Theorem $7$ therein).
%Nonetheless,  compared to \cite{fazel2018global}, we note that an additional argument on the stability of the update is provided here. 
%In fact,  for given $L$, the equivalent cost  matrix $\tilde Q_L=Q-L^\top R^v L$ of the inner LQR may not be positive definite, and thus  $\sigma_{\min}(\tilde Q_L)$ may be zero and cannot be lower-bounded. Hence, the stability argument in the updated version of \cite{fazel2018global}, which is based on Lemma $16$ therein, does not apply here. 
% our results differ in the following aspects: i) global \emph{sublinear} instead of \emph{linear} rate is established here since for given $L$, the equivalent cost  matrix $\tilde Q_L=Q-L^\top R^v L$ may not be positive definite, and thus we may lose the lower-boundedness of  $\sigma_{\min}(\tilde Q_L)$, which was   essential in characterizing the convergence rate in \cite{fazel2018global}; ii) additional argument on the stability of the update is provided, in advance of  analyzing its convergence. 
However, we provide additional stability arguments 
%\footnote{Note that the stability argument in \cite{fazel2018global} was supplemented  during the time of preparing this paper. But still, we provide a different approach to show the stability for the  updates, which may  be of independent interest.} 
for  the control pair $(K_\tau,L)$ 
as the inner loop update proceeds. 
%along the iteration  of $\tau$. 
%,  where $\{K_\tau\}$ is  generated from  \eqref{equ:inner_gd}-\eqref{equ:inner_gauss_newton}. 
% The proof   is deferred to \S\ref{subsec:proof_inner_full_grad_conv}. 

%To establish the global convergence of the projected 
% nested-gradient algorithms given in \S\ref{subsec:exact_nested_grad}, we first make the following assumption on the  Nash equilibrium $(K^*,L^*)$. 
%
%
%
%
%\begin{assumption}\label{assum:Omega_L_s}
%	Suppose the Nash equilibrium $(K^*,L^*)$ satisfies that $Q-(L^*)^\top R^v L^*>0$, i.e., $\sigma_{\min}(\tilde Q_{L^*})>0$.  
%\end{assumption}
% 
% XXXX JUSTIFY the assumption XXXX

We then establish the global convergence  of the projected NG updates \eqref{equ:algorithm_nested_L}, \eqref{equ:algorithm_natural_nested_L}, and \eqref{equ:algorithm_nested_precond_L}. Before we state the results,  we define the \emph{gradient mapping}  for all three projection operators $\PP_{\Omega}^{GN},\PP_{\Omega}^{NG},$ and $\PP_{\Omega}^{GD}$ at any $L\in\Omega$ as follows
\#
&\hat{G}_{L}^*:=\frac{\PP_{\Omega}^{GN}\big[L + \eta W_{L}^{-1}\nabla_L \tilde{\cC}(L)\Sigma_{K(L),L}^{-1}\big]-L}{2\eta}\quad  \qquad 
\tilde{G}_{L}^*:=\frac{\PP_{\Omega}^{NG}\big[L + \eta \nabla_L \tilde{\cC}(L)\Sigma_{K(L),L}^{-1}\big]-L}{2\eta}\notag\\ 
&\qquad\qquad\qquad\qquad\qquad\qquad\check{G}_{L}^*:=\frac{\PP_{\Omega}^{GD}\big[L + \eta \nabla_L \tilde{\cC}(L)]\big]-L}{2\eta}.\label{equ:def_check_G}
\#
Note that {gradient mappings} have been commonly adopted in the analysis of projected gradient descent methods in constrained optimization   \citep{nesterov2013introductory}.

\begin{theorem}[Global Convergence Rate of  Outer-Loop Update]\label{thm:full_grad_conv}
	Suppose $\EE_{x_0\sim\cD}x_0x_0^\top>0$,  Assumption \ref{assum:invertibility} holds, and the  initial maximizer  control $L_0\in\Omega$, where $\Omega$ is defined in \eqref{equ:def_Omega}. 
%	 for the maximizer admits a  solution to  the inner Riccati equation \eqref{equ:Inner_Riccati} such that  $P_{K(L_0),L_0}> 0$, and $\EE_{x_0\sim\cD}x_0x_0^\top>0$.
%	control pair $(K(L_0),L_0)$ is stable, i.e., the cost 
%	$\cC(K_0,L(K_0))$ and 
%	$\cC(K(L_0),L_0)$ is finite,  
%	  $R^v- C^\top P_{K_0}^*C>0$ holds, 
Then  it follows that: 
i) at   iteration $t$  of the projected NG updates   \eqref{equ:algorithm_nested_L}, \eqref{equ:algorithm_natural_nested_L}, and \eqref{equ:algorithm_nested_precond_L},   the inner-loop   updates   
\eqref{equ:inner_gd}-\eqref{equ:inner_gauss_newton} converge   to    $K(L_t)$   with linear rate; 
ii) the  control pair sequences $\{(K(L_t),L_t)\}_{t\geq 0}$ generated from  \eqref{equ:algorithm_nested_L}, \eqref{equ:algorithm_natural_nested_L}, and \eqref{equ:algorithm_nested_precond_L} are always stabilizing (regardless of the stepsize choice $\eta$); 
%ii) 
%there exists a constant stepsize $\eta>0$ for   each of 
%the projected NG updates \eqref{equ:algorithm_nested_L}-\eqref{equ:algorithm_nested_precond_L} such that the generated  control pair sequence $\{(K(L_t),L_t)\}$ stabilizes the system; 
iii) with proper choices of the stepsize $\eta$, the  updates   \eqref{equ:algorithm_nested_L}, \eqref{equ:algorithm_natural_nested_L}, and \eqref{equ:algorithm_nested_precond_L} 
%	, and the nested-gradient updates   \eqref{equ:algorithm_4_minimax}-\eqref{equ:algorithm_4_maximin} all 
 all converge to the Nash equilibrium $(K^*,L^*)$ of the zero-sum LQ game  \eqref{equ:nonconvex_concave_def} with $\cO(1/{t})$ rate, in the sense that the sequences  $\big\{t^{-1}\sum_{\tau=0}^{t-1}\big\|\hat{G}_{L_\tau}^*\big\|^2\big\}_{t\geq 1}$, $\big\{t^{-1}\sum_{\tau=0}^{t-1}\big\|\tilde{G}_{L_\tau}^*\big\|^2\big\}_{t\geq 1}$, and $\big\{t^{-1}\sum_{\tau=0}^{t-1}\big\|\check{G}_{L_\tau}^*\big\|^2\big\}_{t\geq 1}$ all converge to zero with $\cO(1/{t})$ rate. 
%iii) at   iteration $t$  of the outer-loop NG  updates,   the inner-loop   updates   
%\eqref{equ:inner_gd}-\eqref{equ:inner_gauss_newton} converge   to    $K(L_t)$   with linear rate. 
\end{theorem}

%\begin{theorem}[Global Convergence of Nested-gradient Algorithms  for Minimax]\label{thm:full_grad_conv_minimax}
%	Suppose $\cC(K_0,L(K_0))$ is finite, $\mu>0$, and $R^v- C^\top P_{K_0}^*C>0$, then the Gauss-Newton update  \eqref{equ:algorithm_nested_precond_K}, the Natural NG update \eqref{}
%	
%	
%	\begin{itemize}
%		\item Gauss-Newton update \eqref{equ:algorithm_nested_precond_K}: For a stepsize $\eta=1$ and for 
%		\$
%		N\geq XXXX,
%		\$
%		the Gauss-Newton NG update \eqref{equ:algorithm_nested_precond_K} enjoys the following performance bound
%		\$
%		\cC(K_N,L(K_N))-\cC(K^*,L^*)\leq \epsilon,
%		\$
%		i.e., the convergent point $(K_N,L(K_N))$ constitutes an $\epsilon$-Nash equilibrium. 
%		\item XXXXXX
%		\item XXXXXX
%	\end{itemize}
%\end{theorem}

%Theorem \ref{thm:full_grad_conv} shows that 
%if the nested-gradient algorithms start from an  
%\emph{stable} 
%initialization $(K(L_0),L_0)$ that makes $P_{K(L_0),L_0}>0$, then 
Since $\Omega\subset  \underline{\Omega}$, 
the first two arguments follow directly  from Proposition \ref{prop:inner_full_grad_conv}. 
The last argument shows that   the iterate $(K(L_t),L_t)$ generated from the projected NG updates   converges  with a sublinear rate. 
%In fact, 
%as to be illustrated later (see Corollary \ref{coro:stability_maximin} in \S\ref{subsec:proof_full_grad_conv}), the stability argument follows from the fact that $P_{K(L_t),L_t}>0$ holds for all $t\geq 0$ along the iteration of the  nested-gradient algorithms, provided that   $P_{K(L_0),L_0}>0$ holds. 
%Moreover, this also implies that the assumption in Proposition \ref{prop:inner_full_grad_conv}, i.e., the observability of $(\tilde A_{L_t},\tilde Q_{L_t})$, holds along the iteration  of the  NG algorithms (see Lemma $21.6$ in \cite{zhou1996robust}), which leads to the linear convergence rate of the inner-loop updates  to $K(L_t)$. 
%Note that the requirement of $P_{K(L_0),L_0}>0$ is not restricted, and can be easily satisfied by simply choosing $L_0=\bm{0}_{m_2\times d}$ under the standard  assumption that $(A,Q)$ is observable. This is because when $L_0=\bm{0}_{m_2\times d}$, we have $\tilde A_{L_0}=A$ and $\tilde Q_{L_0}=Q$,   and thus  $P_{K(L_0),L_0}>0$ (see Proposition $4.4.1$ in \cite{bertsekas2005dynamic}). 
Detailed proof of Theorem \ref{thm:full_grad_conv} is provided  in  \S\ref{subsec:proof_full_grad_conv}.

Due to the nonconvexity-nonconcavity of the problem (see Lemma \ref{lemma:nonconvex_nonconcave}),   our result is pertinent to   the  recent  work on finding a  first-order stationary point   for nonconvex-nonconcave minimax games  under the {P}olyak-{\L}ojasiewicz (P{\L})-condition for one of the players \citep{sanjabi2018solving}.  Interestingly,   the LQ games considered here also satisfy the one-sided P{\L}-condition in \cite{sanjabi2018solving}, since for a given $L\in\underline{\Omega}$,  the inner problem is an LQR, which enables the use of Lemma $11$ in \cite{fazel2018global} to show this.  However, as recognized by  \cite{fazel2018global} for LQR problems, the main challenge of the LQ games here in contrast to the   minimax game  setting in \cite{sanjabi2018solving} is coping  with the lack of smoothness in the objective function. 

%We note that 
%the sublinear global convergence  rate for the nested-gradient (first-order) method in Theorem \ref{thm:full_grad_conv} matches the iteration lower-bound of using gradient methods for solving nonconvex optimization problems \citep{nesterov2013introductory} and nonconvex-nonconcave minimax problems \citep{sanjabi2018solving}. 
%our    $\cO(1/t)$ global rate is  for the convergence of the average gradient mapping (see definition in \S\ref{sec:append_main_proof} and the proof of Theorem \ref{thm:full_grad_conv}) norm square. 
This $O(1/t)$ rate  
   matches  the sublinear convergence rate to first-order stationary points, instead of (local) Nash equilibrium,  in  \cite{sanjabi2018solving,nouiehed2019solving}. 
In contrast, by the landscape of zero-sum LQ games shown in Lemma \ref{lemma:stationary_point}, our convergence is to the \emph{global NE} of the game, if the projection is not effective at the accumulation point. In fact, in this case, the convergence rate can be improved to be linear, as to be introduced next in Theorem \ref{thm:full_grad_conv_local}.  
%   \footnote{Note that the criteria for approximate stationarity in \citep{sanjabi2018solving} is the gradient norm (no square), thus the $\cO(1/\sqrt{T})$ rate therein matches our  $\cO(1/T)$ rate here.}. 
   In addition, our rate also matches      the (worst-case) \emph{global}  convergence  rate of gradient descent and second-order algorithms for nonconvex optimization, 
either under the smoothness  assumption of the objective \citep{cartis2010complexity,cartis2017worst}, or for a certain class of non-smooth objectives \citep{khamaru2018convergence}.

Compared to \cite{fazel2018global}, the nested-gradient algorithms cannot be shown to have  \emph{globally linear} convergence rates so far, owing  to the additional nonconcavity   on $L$ added to the standard LQR problems. Nonetheless, the P{\L}  property of the LQ games still enables linear convergence  rate near the Nash equilibrium.
% similar as other approximate second-order algorithms but with smoothness assumptions for nonconvex optimization. 
 We formally establish the local convergence results in the following theorem, whose proof is provided in \S\ref{subsec:proof_full_grad_conv_local}.

% In this subsection, we explore the case where the inner maximization is solved to its stationary point. 
% Our goal is to show that, when the inner problem is solved to the stationary point for given $K$, the policy gradient update over $K$  solves the LQR game, i.e., the two variables converge to the Nash equilibrium.
 
% To this end, for any $K$, recall that $L(K)$ is the solution to 
% \#\label{eq:inner_sp}
% \nabla _{L} C(K, L)  = 2[(-R^v+C^\top P_{K,L}C)L(K)-C^\top P_{K,L}(A-BK)]\Sigma_{K,L}=0. 
% \#
% If $\Sigma_{K,L}$ is full-rank, which can be guaranteed by letting $\EE_{x_0\sim\cD} x_0x_0^\top$ be full-rank,  we have that the $F_{K,L(K)}$ defined in \eqref{equ:def_E_F_mu_2} satisfies
% \$
%F_{K, L(K)} = \bigl( - R^v + C ^\top P_{K, L(K) } C \bigr )L(K) - C^\top P_{K, L(K)} ( A - BK) = 0. 
% \$
% Thus,  $L(K)$ has the following form  
% \#\label{equ:LK_def}
% L(K)=(-R^v+C^\top P_{K,L(K)}C)^{-1}C^\top P_{K,L(K)}(A-BK),
% \#
% provided the matrix $-R^v+C^\top P_{K,L(K)}C$ is invertible. 
 
% \#
% L(K) = \argmax_{L} C(K, L ),
% \#
% which satisfies $\nabla _{L} C(K, L) \vert_{L = L(K)} = 0$ for any $K$. Moreover, by Lemma \ref{lemma:Grad_Dom}, we obtain that 
% \$
% \nabla _{L} C(K, L ) = 0 \quad \text{if and only if}\quad L = L(K).
% \$
% Thus, for any $K$, if $\Sigma_{K, L(K) }$ is full-rank, we have 
%XXXX

%Now, we establish  the local linear convergence rate of the nested-gradient algorithms. 

\begin{theorem}[Local Convergence Rate   of  Outer-Loop Update]\label{thm:full_grad_conv_local}
	Under the conditions of Theorem \ref{thm:full_grad_conv}, 
	the projected NG updates   \eqref{equ:algorithm_nested_L}, \eqref{equ:algorithm_natural_nested_L}, and \eqref{equ:algorithm_nested_precond_L}
%	there exists a constant stepsize $\eta>0$ for each of the NG updates \eqref{equ:algorithm_nested_L}-\eqref{equ:algorithm_nested_precond_L}, such that 
%	%	, and the nested-gradient updates  \eqref{equ:algorithm_4_minimax}-\eqref{equ:algorithm_4_maximin} all 
%%	both 
%the updates 
all 
	have locally linear convergence rates   around the  Nash equilibrium $(K^*,L^*)$ of the   LQ game  \eqref{equ:nonconvex_concave_def}, in the sense that   the cost  value sequence $\{\cC(K(L_t),L_t)\}_{t\geq 0}$  converges to  $\cC(K^*,L^*)$, and the nested gradient norm square sequence  $\{\|\nabla_L \tilde{\cC}(L_t)\|^2\}_{t\geq 0}$  converges to zero,  both  with linear rates. 
%	In particular, there exists constants $N_{NG},N_{NNG},N_{GNNG}>0$, such that $$
\end{theorem}

Theorem \ref{thm:full_grad_conv_local} shows that when the proposed NG updates \eqref{equ:algorithm_nested_L}, \eqref{equ:algorithm_natural_nested_L}, and \eqref{equ:algorithm_nested_precond_L} get closer to the  NE $(K^*,L^*)$, the local convergence rates can be improved from sublinear (see Theorem \ref{thm:full_grad_conv}) to linear. This resembles the convergence property of (Quasi)-Newton methods for nonconvex optimization, with globally sublinear and locally  linear convergence rates. To  the best of our knowledge, this appears to be the first such  result on  equilibrium-seeking for  
nonconvex-nonconcave  minimax games, even with the smoothness assumption as in \cite{sanjabi2018solving}.

We note that for the class of zero-sum LQ games that  Assumption \ref{assum:invertibility} ii) fails to  hold, there may not exists a  set $\Omega$ of the form  \eqref{equ:def_Omega} that contains the NE $(K^*,L^*)$. Even then, our global  convergence results in  Proposition \ref{prop:inner_full_grad_conv} and Theorem \ref{thm:full_grad_conv}  still hold. 
%\footnote{In this case, the statement should be changed from  global convergence to the \emph{NE}, to global convergence to the   \emph{projected point of NE onto $\Omega$}. }.  
This is because  the convergence is established in the sense of gradient mappings. However, this may invalidate the statements on local convergence in Theorem \ref{thm:full_grad_conv_local}, as the proof relies on the ineffectiveness of the projection operator.

%% file: main_proof.tex
\section{Proofs of Main Results}\label{sec:append_main_proof}

In this section, we provide  proofs for the main results on the convergence of the nested-gradient algorithms stated in \S\ref{sec:theory}. 

For notational convenience, we (re-)define the following functions
\$
\text{value:}&\qquad V_{K,L}(x)=x^\top P_{K,L}x,\\
\text{action-value:}&\qquad Q_{K,L}(x,u,v)=x^\top Qx+u^\top R^uu-v^\top R^vv+V_{K,L}(Ax+Bu+Cv),\\
\text{advantage:}&\qquad A_{K,L}(x,u,v)=Q_{K,L}(x,u,v)-V_{K,L}(x).
\$
Also, we define 
\#
E_{K,L}&=(R^u+B^\top P_{K,L} B)K-B^\top P_{K,L}(A-CL),\label{equ:def_E_F_mu_1}\\
F_{K,L}&=(-R^v+C^\top P_{K,L} C)L-C^\top P_{K,L}(A-BK),\label{equ:def_E_F_mu_2}\\
\mu &=\sigma_{\min}\big(\EE_{x_0\sim\cD}x_0x_0^\top\big),\quad \nu=\sigma_{\min}\big(W_{L^*}\big), \label{equ:def_E_F_mu_3}
\# 
where we recall the definitions of $P_{K,L}$    and  $W_{L}$ in \eqref{equ:P_KL_Sol} and  \eqref{eq:define_wll}, respectively.
To simplify the notation, we denote   $\zeta_{K(L), L}$ by $\zeta_{L}^*$, for any notation $\zeta_{K,L}$, for example, $V_{K,L}$, $Q_{K,L}$, $A_{K,L}$, $P_{K,L}$, etc. 
%In addition, we recall the definitions of gradient mappings in \eqref{equ:def_tilde_G}-\eqref{equ:def_check_G} as follows
%\small
%\$
%&\hat{G}_{L}^*:=\frac{\PP_{\Omega}^{GN}\big[L + 2\eta \cdot W_L^{-1}F_{L}^*\big]-L}{2\eta},~~
%\tilde{G}_{L}^*:=\frac{\PP_{\Omega}^{NG}\big[L + 2\eta \cdot F_{L}^*\big]-L}{2\eta}, ~~ \check{G}_{L}^*:=\frac{\PP_{\Omega}^{GD}\big[L + 2\eta \cdot F_{L}^*\Sigma_L^*\big]-L}{2\eta}. 
%\$
%\normalsize

\subsection{Auxiliary Lemmas}\label{subsec:helper_lemmas}

To proceed with the analysis, we first establish several   lemmas that are useful in the ensuing analysis. 
The first  lemma  links the value function $V_{K,L}$ and the advantage function $A_{K,L}$, when varying  $K$ and  $L$, which  plays a similar role as Lemma $7$ in \cite{fazel2018global}. 

\begin{lemma}[Cost Difference Lemma]\label{lemma:Cost_Diff}
	Suppose both $(K,L)$ and $(K',L')$ are  stabilizing.  
%	and  $(K',L')$
	%, $(K,L)$, and $K,L'$ all 
%	both have 
%	has finite expected value.
	   Let $\{x_t'\}_{t\geq 0}$ and $\{(u'_t,v'_t)\}_{t\geq 0}$ be the sequences of state and action pairs generated by $(K',L')$, i.e., starting from  $x_0'=x$ and satisfying $u'_t=-K'x_t', ~v'_t=-L'x_t'$.
%	   , where $x$ does not lie in the null space of $A-BK'-CL'$. 
Then, it follows that
	\#
	V_{K',L'}(x)-V_{K,L}(x)=\sum_{t\geq 0}A_{K,L}(x'_t,u'_t,v'_t).\label{equ:Cost_Diff}
	\#
	Moreover, we have
	\#\label{equ:Advantage_Diff}
	A_{K,L}(x,-K'x,-L'x)=&2x^\top(K'-K)^\top E_{K,L}x+x^\top (K'-K)^\top (R^u+B^\top P_{K,L}B)(K'-K)x\notag\\
	&\quad +2x^\top(L'-L)^\top F_{K,L}x+x^\top (L'-L)^\top (-R^v+C^\top P_{K,L}C)(L'-L)x \notag\\
	&\quad +2x^\top (L'-L)^\top C^\top P_{K,L}B(K'-K)x.
	\# 
\end{lemma}

\begin{proof}
	Let  the sequence of costs generated under $(K',L')$ be denoted by $c'_t$. Then
	\$
	V_{K',L'}(x)-V_{K,L}(x)=&~\sum_{t\geq 0}c'_t-V_{K,L}(x)=\sum_{t\geq 0}\big[c'_t+V_{K,L}(x'_t)-V_{K,L}(x'_t)\big]-V_{K,L}(x)\\
	=&~\sum_{t\geq 0}\big[c'_t+V_{K,L}(x'_{t+1})-V_{K,L}(x'_t)\big]=\sum_{t\geq 0}A_{K,L}(x'_t,u'_t,v'_t).
	\$
	Thus, we establish the first argument.
	
	Moreover, for the second claim, let $u= - K'x$ and $v=- L'x$. Then
	\small
	\$
	&A_{K,L}(x,u,v)=Q_{K,L}(x,u,v)-V_{K,L}(x)\\
	&=x^\top \big[Q+(K')^\top R^u K'-(L')^\top R^v L'\big]x+x^\top(A-BK'-CL')^\top P_{K,L} (A-BK'-CL') x-V_{K,L}(x)\\
	&=2x^\top(K'-K)^\top \big[(R^u+B^\top P_{K,L}B)K-B^\top P_{K,L}(A-CL)\big]x+x^\top (K'-K)^\top (R^u+B^\top P_{K,L}B)\\
	&\quad\cdot(K'-K)x +2x^\top(L'-L)^\top \big[(-R^v+C^\top P_{K,L}C)L-C^\top P_{K,L}(A-BK)\big]x\\
	&\quad+2x^\top (L'-L)^\top C^\top P_{K,L}B(K'-K)x+x^\top (L'-L)^\top (-R^v+C^\top P_{K,L}C)(L'-L)x\\
	&=2x^\top(K'-K)^\top E_{K,L}x+x^\top (K'-K)^\top (R^u+B^\top P_{K,L}B)(K'-K)x+2x^\top(L'-L)^\top F_{K,L}x\\
	&\quad +x^\top (L'-L)^\top (-R^v+C^\top P_{K,L}C)(L'-L)x  +2x^\top (L'-L)^\top C^\top P_{K,L}B(K'-K)x,
	\$
	\normalsize
	which completes the proof. 
\end{proof} 

For any $L\in\underline{\Omega}$, recall that  $P_L^*$ is the solution to the inner-loop Riccati equation \eqref{equ:Inner_Riccati}, and $K(L)$ is the stationary point solution defined in \eqref{equ:KL_def}. We have the following properties of $P_L^*$ and $K(L)$.
% The proof of the lemma is provided in \S\ref{subsec:proof_optim_of_KL}.  

 \begin{lemma}[Optimality of   $K(L)$ and Boundedness of $P_L^*$]\label{lemma:optim_of_KL}
Suppose  $\Sigma_{K,L}$ is full-rank for any $K$ and $L$.  
%For any $K\in\cK$ with $\cK$ being defined in \eqref{equ:def_cK_set}, let $L(K)$ be the stationary-point solution of the inner maximization given by \eqref{equ:LK_def}, then it follows that for any $x\in\RR^d$, $V_{K}^*(x)\geq V_{K,\tilde{L}}(x)$  for any $\tilde{L}$ such that $(K,\tilde{L})$ is stable. Taking expectation on both sides  further gives that $\cC(K,L(K))\geq \cC(K,\tilde{L})$.  
%Similarly,  
%	f
Recall the definition of   $\underline{\Omega}$  in \eqref{equ:def_Omega_underline}. 
%and that  the stationary-point solution of the inner minimization $K(L)$ is  given by \eqref{equ:KL_def}. 
Then under Assumption \ref{assum:invertibility}, for any  $L\in\underline{\Omega}$, 
%where $\underline{\Omega}$ is defined in \eqref{equ:def_Omega}, 
%	$L\in\cL$ with $\cL$ being defined in \eqref{equ:def_cL_set},
%	 let $K(L)$ be the stationary-point solution of the inner minimization given by \eqref{equ:KL_def}. 
 the inner-loop Riccati equation \eqref{equ:Inner_Riccati} always admits a solution 	 $P_{L}^*>0$, and the control pair $(K(L),L)$ is stabilizing. Moreover,   
	  for any $x\in\RR^d$, $V_{L}^*(x)\leq V_{\tilde K,L}(x)$  for any $\tilde K\in\RR^{m_1\times d}$. Taking expectation on both sides further yields that $\cC(K(L),L)\leq \cC(\tilde{K},L)$.  
	  In addition, $P_L^*$ is bounded and satisfies $Q-L^\top R^vL\leq P_L^*\leq P^*$, which implies that $\cC(K(L),L)\leq \cC(K^*,L^*)$. 
\end{lemma} 
%We first prove the optimality of the stationary point $K(L)$  and the stability of the control pair $(K(L),L)$, for  given $L\in\underline{\Omega}$.  

%\begin{lemma}[Lemma \ref{lemma:optim_of_KL} restated]\label{lemma:optim_of_KL_re}
%Suppose  $\Sigma_{K,L}$ is full-rank for any $K$ and $L$.  
%%For any $K\in\cK$ with $\cK$ being defined in \eqref{equ:def_cK_set}, let $L(K)$ be the stationary-point solution of the inner maximization given by \eqref{equ:LK_def}, then it follows that for any $x\in\RR^d$, $V_{K}^*(x)\geq V_{K,\tilde{L}}(x)$  for any $\tilde{L}$ such that $(K,\tilde{L})$ is stable. Taking expectation on both sides  further gives that $\cC(K,L(K))\geq \cC(K,\tilde{L})$.  
%%Similarly,  
%%	f
%Recall the definition of the set $\underline{\Omega}$  in \eqref{equ:def_Omega}, and that  the stationary-point solution of the inner minimization $K(L)$ is  given by \eqref{equ:KL_def}. 
%Then for any  $L\in\underline{\Omega}$, 
%%where $\underline{\Omega}$ is defined in \eqref{equ:def_Omega}, 
%%	$L\in\cL$ with $\cL$ being defined in \eqref{equ:def_cL_set},
%%	 let $K(L)$ be the stationary-point solution of the inner minimization given by \eqref{equ:KL_def}. 
% the inner Riccati equation \eqref{equ:Inner_Riccati} always admits a solution 	 $P_{L}^*>0$, and the control pair $(K(L),L)$ is stabilizing. Moreover,   
%	  for any $x\in\RR^d$, $V_{L}^*(x)\leq V_{\tilde K,L}(x)$  for any $\tilde K\in\RR^{m_1\times d}$. Taking expectation on both sides further gives that $\cC(K(L),L)\leq \cC(\tilde{K},L)$.  
%\end{lemma}
\begin{proof}

Since $\tilde Q_L=Q-L^\top R^v L>0$, it follows that $(\tilde A_L,\tilde Q_L)$ is observable. Moreover, 
 Lemma \ref{lemma:saddle_point_formal} shows the existence of the saddle-point $(K^*,L^*)$, which implies that for any $L\in\underline{\Omega}$ and any $x_0\in\RR^d$
 \#\label{equ:bounded_P_trash_1}
 V_{K^*,L}(x_0)\leq V_{K^*,L^*}(x_0)<\infty,
 \#
 which further implies that  $0\leq P_{K^*,L}\leq P_{K^*,L^*}$. Thus, for the inner LQR problem with any $L\in\underline{\Omega}$, there always exists a stabilizing control $K^*$, i.e., $(\tilde A_L,B)$ is always stabilizable \citep{kwakernaak1972linear}. Hence, by Proposition $4.4.1$ in \cite{bertsekas2005dynamic}, we know that the inner-loop Riccati equation \eqref{equ:Inner_Riccati} always admits a solution 	 $P_{L}^*>0$, and the control pair $(K(L),L)$ is stabilizing. Moreover, $K(L)$ yields the optimal cost, i.e.,
 \#\label{equ:bounded_P_trash_2}
 V_{K(L),L}(x_0)\leq V_{\tilde K,L}(x_0),
 \#
 for any $K$. Taking expectation over \eqref{equ:bounded_P_trash_2} on $x_0\sim \cD$ yields $\cC(K(L),L)\leq \cC(\tilde{K},L)$.
% completes the proof of Lemma \ref{lemma:optim_of_KL}. 
 
 Furthermore,   combining \eqref{equ:bounded_P_trash_1} and \eqref{equ:bounded_P_trash_2} yields 
  \#\label{equ:bounded_P_trash_3}
 V_{K(L),L}(x_0)\leq V_{ K^*,L}(x_0)\leq V_{ K^*,L^*}(x_0),
 \#
 for any $x_0$. As a result, we have $P_L^*\leq P^*$. Taking expectation over \eqref{equ:bounded_P_trash_3} further gives $\cC(K(L),L)\leq \cC(K^*,L^*)$.
 Also, since $P_L^*$ is a solution to Lyapunov equation 
 \$
 P_L^*=\tilde Q_L+K^\top R^u K+[\tilde A_L-BK(L)]^\top P_L^* [\tilde A_L-BK(L)],
 \$
 it holds  that $P_L^*\geq Q_L$,  
  which completes the proof. 
\end{proof}

Moreover, we also need the  
following  lemma that  characterizes the property of the projection operator   in the projected NG updates   \eqref{equ:algorithm_nested_L}, \eqref{equ:algorithm_natural_nested_L}, and \eqref{equ:algorithm_nested_precond_L}.  	The proof of the lemma is provided in \S\ref{subsec:proof_lemma:proj_prop}. 

\begin{lemma}\label{lemma:proj_prop}
For any $L_1,L_2\in\RR^{m_2\times d}$, 
	the projection operators defined in \eqref{equ:def_proj_GD}, \eqref{equ:def_proj_NG}, and \eqref{equ:def_proj_GN} at iterate $L$ have the following  properties:
	\footnotesize
	\$
	&\tr\Big[\big(L_1-L_2\big)\Sigma_L^*\big(\PP_{\Omega}^{GN}[L_1]-\PP_{\Omega}^{GN}[L_2]\big)^\top W_L\Big] \geq \tr\Big[\big(\PP_{\Omega}^{GN}[L_1]-\PP_{\Omega}^{GN}[L_2]\big)\Sigma_L^*\big(\PP_{\Omega}^{GN}[L_1]-\PP_{\Omega}^{GN}[L_2]\big)^\top W_L\Big],\\ 
	&\tr\Big[\big(L_1-L_2\big)\Sigma_L^*\big(\PP_{\Omega}^{NG}[L_1]-\PP_{\Omega}^{NG}[L_2]\big)^\top\Big] \geq \tr\Big[\big(\PP_{\Omega}^{NG}[L_1]-\PP_{\Omega}^{NG}[L_2]\big)\Sigma_L^*\big(\PP_{\Omega}^{NG}[L_1]-\PP_{\Omega}^{NG}[L_2]\big)^\top\Big],\\
	&\tr\Big[\big(L_1-L_2\big) \big(\PP_{\Omega}^{GD}[L_1]-\PP_{\Omega}^{GD}[L_2]\big)^\top\Big] \geq \tr\Big[\big(\PP_{\Omega}^{GD}[L_1]-\PP_{\Omega}^{GD}[L_2]\big) \big(\PP_{\Omega}^{GD}[L_1]-\PP_{\Omega}^{GD}[L_2]\big)^\top\Big]. 
	\$
	\normalsize
\end{lemma}

Another important result used later is the continuity of $P_{L}^*$   w.r.t. $L$, for any $L\in\underline{\Omega}$, whose proof is deferred to \S\ref{subsec:lemma_conti_PLs}.

%To this end, we first provide the following lemma showing that $P_{L}^*$ is continuous w.r.t. $L$. 

\begin{lemma}\label{lemma:conti_PLs}
For any $L\in\underline{\Omega}$, 
	let $P_{L}^*>0$ be the solution to the inner-loop Riccati equation   \eqref{equ:Inner_Riccati}.  Then $P_{L}^*$ is a continuous function  w.r.t $L$. 
%	 for all $L$ that makes  the control pair $(K(L),L)$  stabilizing.
\end{lemma}

Similarly, we also establish the following lemma on the continuity of the correlation matrix $\Sigma_{K,L}$ and $P_{K,L}$ w.r.t. $K$ and $L$, respectively. 

\begin{lemma}\label{lemma:conti_Sigma}
	For any stabilizing control pair $(K,L)$,  the correlation matrix $\Sigma_{K,L}$, and the solution $P_{K,L}$ to Lyapunov equation \eqref{equ:P_KL_Sol} are both  continuous w.r.t. $K$ and $L$.
\end{lemma}
\begin{proof}
 For stabilizing $(K,L)$, $\Sigma_{K,L}$ is the unique solution to the  Lyapunov equation
	\#\label{equ:cont_Sigma_trash_1}
	(A-BK-CL)\Sigma_{K,L}(A-BK-CL)^\top +\Sigma_0=\Sigma_{K,L},
	\#
	where we denote $\EE_{x_0\sim\cD}x_0x_0^\top>0$ by $\Sigma_0$. By vectorizing both sides, we can rewrite \eqref{equ:cont_Sigma_trash_1}~as 
	\$
	&\Psi\big(\vect(\Sigma_{K,L}),K,L\big)=\vect(\Sigma_{K,L}),
	\$
	where  the operator $\Psi:\RR^{d^2}\times \RR^{m_1\times d}\times \RR^{m_2\times d}\to \RR^{d^2}$ is defined as 
	\$
\Psi\big(\vect(\Sigma_{K,L}),K,L\big):=\big[(A-BK-CL)\otimes(A-BK-CL)\big]\cdot\vect(\Sigma_{K,L})+\vect(\Sigma_0).
	\$
	Notice that
	\$
	\frac{\partial \big[\Psi\big(\vect(\Sigma_{K,L}),K,L\big)-\vect(\Sigma_{K,L})\big]}{\partial \vect^\top(\Sigma_{K,L})}=\big[(A-BK-CL)\otimes(A-BK-CL)\big]-I,
	\$
	which is invertible for stabilizing $(K,L)$, since the eigenvalues of $\big[(A-BK-CL)\otimes(A-BK-CL)\big]$ have absolute values smaller than one. Hence, 	by the implicit function theorem \citep{krantz2012implicit}, $\vect(\Sigma_{K,L})$ is continuously differentiable, and also continuous, w.r.t. $K$ and $L$, 
	 which completes the proof.  
	The proof for $P_{K,L}$ is almost identical, which is omitted here for brevity. 
\end{proof}

In addition, recalling the definition of $\Omega$ in \eqref{equ:def_Omega}, we have $\Omega\subset \underline{\Omega}$. Hence, by Lemma \ref{lemma:optim_of_KL}, for any $L\in\Omega$, $P_L^*$ exists and $(K(L),L)$ is stabilizing. Hence, $\Sigma_L^*$ also exists. We can then bound  the spectral norm of $P_L^*$ and $\Sigma_L^*$. Also, since $P_L^*\leq P^*$, we can also bound  $W_L$ (see definition in \eqref{eq:define_wll}) as follows.

\begin{lemma}[Bounds for $\|P_{K,L}\|,\|\Sigma_{K,L}\|$, and $W_L$]\label{lemma:bound_P_L_Sigma_L}
Recalling the definition of $\Omega$ in  \eqref{equ:def_Omega} as
\$
\Omega:=\big\{L\in\RR^{m_2\times d}\given  Q-L^\top R^v L\geq \zeta\cdot \Ib\big\},
\$
	 it follows that for any $L\in\Omega$ and any $K$ that makes $(K,L)$ stabilizing 
	\$
	&\|P_{K,L}\|\leq {\cC(K,L)}/{\mu},\qquad\qquad\qquad\qquad\qquad\qquad\qquad \|\Sigma_{K,L}\|\leq {\cC(K,L)}/{\zeta}, \\
	&0<R^v-C^\top P^*C\leq W_{L^*}\leq W_{L}\leq R^v - C^\top \big[\xi^{-1}\cdot\Ib +   B(R^u)^{-1}B^\top\big]^{-1}C\leq R^v.
	\$
\end{lemma}
\begin{proof}
	Since $(K,L)$ is stabilizing,   $\cC(K,L)$ can be bounded as
	\$
	\cC(K,L)=\EE_{x_0\sim\cD} x_0^\top P_{K,L} x_0\geq \|P_{K,L}\|\sigma_{\min}(\EE x_0x_0^\top),
	\$
	since $P_{K,L}\geq P_L^*>0$ is positive definite by Lemma \ref{lemma:optim_of_KL}. Moreover,  $\cC(K,L)$ can also be bounded as
		\$\cC(K,L)&=\tr[\Sigma_{K,L}(Q+K^\top R^u K-L^\top R^v L)]\geq\tr(\Sigma_{K,L})\sigma_{\min}(Q-L^\top R^v L)\\
		&\geq \|\Sigma_{K,L}\|\sigma_{\min}(Q-L^\top R^v L)\geq \|\Sigma_{K,L}\|\cdot\zeta,
	\$
	where the first inequality uses the fact that $Q-L^\top R^v L$ is positive definite, and the last inequality is due to the definition of the set $\Omega$.  
	
	In addition, by matrix inversion lemma, $W_L$ can be written as 
%	\#\label{equ:w_ll_re_written}
	\$
W_{L} 
&= R^v + C^\top \big[-P_{L}^* + P_{ L}^*B (R^u+B^\top P_{L}^* B)^{-1} B^\top P_{L}^*\big]C\notag\\
&=R^v - C^\top \big[(P_{L}^*)^{-1} +   B(R^u)^{-1}B^\top\big]^{-1}C.
\$
Since Lemma \ref{lemma:optim_of_KL} shows that $\xi\cdot\Ib\leq P_{L}^*\leq P^*$, we know that
\$
0&<R^v-C^\top P^*C\leq R^v - C^\top \big[(P^*)^{-1} +   B(R^u)^{-1}B^\top\big]^{-1}C\leq W_{L}\\
&\leq R^v - C^\top \big[\xi^{-1}\cdot\Ib +   B(R^u)^{-1}B^\top\big]^{-1}C\leq R^v,
\$
which completes the proof. 
\end{proof}

Next, we provide  proofs for  the convergence of the proposed algorithms.

\subsection{Proof of Proposition \ref{prop:inner_full_grad_conv}}\label{subsec:proof_inner_full_grad_conv}

We first prove the global convergence of the inner-loop updates in  \eqref{equ:inner_gd}-\eqref{equ:inner_gauss_newton} for given $L\in\underline{\Omega}$.
Note that the proof roughly follows that of Theorem $7$ in \cite{fazel2018global}, but requires additional arguments on  the stability of the control pair $(K_\tau,L)$, where $\{K_\tau\}_{\tau\geq 0}$ is  generated by the updates in \eqref{equ:inner_gd}-\eqref{equ:inner_gauss_newton}\footnote{Note that the stability argument has been supplemented in the latest version of \cite{fazel2018global}, during the time of preparation of  this paper. But still, we provide a different approach to show the stability for the Gauss-Newton and natural nested-gradient updates, which may  be of independent  interest.}. 
From Lemma \ref{lemma:optim_of_KL}, we know that under Assumption \ref{assum:invertibility}, for any $L\in\underline{\Omega}$, the inner LQR problem always has a solution, and 
 $K(L)$ is such an optimal solution.  Thus,   there always exists some $K$ such that $(K,L)$ is stabilizing, namely, $(K(L),L)$,  which proves the first argument of  Proposition \ref{prop:inner_full_grad_conv}.
  Suppose the updates in  \eqref{equ:inner_gd}-\eqref{equ:inner_gauss_newton}  all start with such a stabilizing   $K$. Thus we have  
\#\label{equ:Bellman_trash}
(\tilde A_L-BK)^\top P_{K,L}(\tilde A_L-BK)-P_{K,L}=-\tilde Q_L-K^\top R^uK.
\#
%which further implies from 
By Lemma \ref{lemma:optim_of_KL}, $P_{K,L}\geq P_{L}^*>0$. Hence, $P_{K,L}$ is  invertible,    
%$(\tilde A_L,\tilde Q_L)$ is observable and $\tilde Q_L\geq 0$, by Proposition  $4.4.1$ in \cite{bertsekas2005dynamic}, 
%we know that for the optimal control $K(L)$ of the inner LQR problem, $P_{K(L),L}$ satisfies the algebraic Riccati equation
%\#\label{equ:inner_Riccati_recall}
%P_{K(L),L}=\tilde Q_L+\tilde{A}_L^\top P_{K(L),L}\tilde{A}_L -\tilde{A}_L^\top P_L^* B (R^u+B^\top P_{K(L),L}B)^{-1}B^\top P_{K(L),L}\tilde{A}_L,
%\#
%and is positive definite. Therefore, for any stabilizing  $K$,  $P_{K,L}\geq P_{K(L),L} >0$ and thus $P_{K,L}$ is invertible. 
and \eqref{equ:Bellman_trash} can be rewritten as 
\$
P_{K,L}^{-\frac{1}{2}}(\tilde A_L-BK)^\top P_{K,L}^{\frac{1}{2}}P_{K,L}^{\frac{1}{2}}(\tilde A_L-BK)P_{K,L}^{-\frac{1}{2}}=I-P_{K,L}^{-\frac{1}{2}}(\tilde Q_L+K^\top R^uK)P_{K,L}^{-\frac{1}{2}},
\$
which gives that
\#\label{equ:stability_trick}
\big[\rho(\tilde A_L-BK)\big]^2= 1-\sigma_{\min}\big[P_{K,L}^{-\frac{1}{2}}(\tilde Q_L+K^\top R^uK)P_{K,L}^{-\frac{1}{2}}\big]\leq 1-\sigma_{\min}\big(P_{K,L}^{-\frac{1}{2}}\tilde Q_LP_{K,L}^{-\frac{1}{2}}\big)< 1,	
\#
where the equation is due to that  $P_{K,L}^{-\frac{1}{2}}(\tilde A_L-BK)^\top P_{K,L}^{\frac{1}{2}}$ has identical spectrum as $\tilde A_L-BK$, the last  inequality is due to that $\tilde Q_L>0$. 
% and  the last strict inequality is due to that $\tilde Q_L\neq \bm{0}_{d\times d}$, i.e.,  $\rho(\tilde Q_L)>0$, since otherwise $(\tilde A_L,\tilde Q_L)$ will  not be observable.  
Also noticing that 
\$
\sigma_{\min}(P_{K,L}^{-1/2}\tilde Q_LP_{K,L}^{-1/2})=\sigma_{\min}(\tilde Q_L^{1/2}P_{K,L}^{-1}\tilde Q_L^{1/2}),
\$ 
we can thus assert that, if $P_{K',L}\leq P_{K,L}$, we have  
\#\label{equ:monotone_rho_bnd}
1-\sigma_{\min}(P_{K',L}^{-1/2}\tilde Q_LP_{K',L}^{-1/2})\leq 1-\sigma_{\min}(P_{K,L}^{-1/2}\tilde Q_LP_{K,L}^{-1/2}).
\#

Note that for all the  inner updates in  \eqref{equ:inner_gd}-\eqref{equ:inner_gauss_newton}, as long as $K\neq K(L)$, it holds that $\|\nabla_K\cC(K,L)\|>0$, i.e., there exists a constant $\epsilon_{K}>0$ such that $\|\nabla_K\cC(K,L)\|\geq \epsilon_{K}$. Moreover, the gradient norm $\|\nabla_K\cC(K,L)\|$ must also be upper bounded, since $K$ is stabilizing, and thus both $\|K\|$ and $\|P_K\|$ are bounded. Also note that both matrices $(R^u+B^\top P_{K,L} B)^{-1}$ and $\Sigma_{K,L}^{-1}$ have upper and lower-bounds, since $R^u+B^\top P_{K,L} B\geq R^u>0$ and $\Sigma_{K,L}\geq \EE_{x_0\sim\cD}x_0x_0^\top>0$, and $P_{K,L}$ is bounded. Therefore, at each $K\neq K(L)$, there exist constants $\text{Upper}_K,\text{Lower}_K>0$ such that 
\$
\alpha\cdot\text{Lower}_K\leq \|K'-K\|\leq \alpha\cdot\text{Upper}_K,
\$
where $K'$ is obtained from the one-step updates in  of any of  \eqref{equ:inner_gd}-\eqref{equ:inner_gauss_newton}. We thus define a  set $\Omega^1_K$, which depends on $K$,  as 
\$
\Omega^1_K:=\Big\{K'\biggiven \|K'-K\|\leq \alpha\cdot\text{Upper}_K\Big\}, 
\$
which is compact. 
On the other hand, define $\Omega^2_K$, the lower-level set of $K'$ as
\$
\Omega^2_K:=\Big\{K'\biggiven\rho(\tilde A_L-BK')\leq [1-\sigma_{\min}(P_{K,L}^{-{1}/{2}}\tilde Q_LP_{K,L}^{-{1}/{2}})]^{1/2}<1\Big\},
\$
which is closed by the continuity and lower-boundedness of   $\rho(\tilde A_L-BK)$  w.r.t. $K$ \citep{tyrtyshnikov2012brief}. 
% and lower bounded, the lower-level set of $K'$ that ensures  
%\$
%\rho(\tilde A_L-BK')\leq [1-\sigma_{\min}(P_{K,L}^{-{1}/{2}}\tilde Q_LP_{K,L}^{-{1}/{2}})]^{1/2}<1. 
%\$ 
Hence, the intersection $\Omega_K=\Omega^1_K\bigcap \Omega^2_K$ 
is  compact. Note that   $\Omega_K\neq \emptyset$, since it at least contains $K$. 
Also, the upper-level set that ensures $\rho(\tilde A_L-BK')\geq 1$ is closed. Thus, by Lemma \ref{lemma:dist_compact}, there exists   a positive distance between the two disjoint sets. 
%from the level-set that makes $\rho(\tilde A_L-BK')\geq 1$, due to the fact that $\rho(\tilde A_L-BK)$ is continuous w.r.t. $K$ \citep{tyrtyshnikov2012brief}. 
Denote this distance by $\delta_{K}$. Then  any $K'$ such  that $\|K'-K\|\leq \delta_K$ is stabilizing.

Now we take the analysis for Gauss-Newton update \eqref{equ:inner_gauss_newton} as an example.  
If $\alpha\cdot\text{Upper}_K\leq \delta_K$ for any $\alpha\in[0,1/2]$, i.e., the range of $\alpha$ in Lemma $14$  of \cite{fazel2018global}     that ensures the contraction of the cost, then both $K'$ and $K$ are stabilizing. 
By further applying Lemma $10$ in \cite{fazel2018global} and the form of \eqref{equ:inner_gauss_newton}, we have that for any $\alpha\in[0,1/2]$
\#\label{equ:contract_P_K_gn}
V_{K',L}(x)-V_{K,L}(x)&=(-4\alpha+4\alpha^2)\tr\bigg[\sum_{t\geq0}(x'_t)(x'_t)^\top E_{K,L}^\top(R^u+B^\top P_{K,L}B)^{-1}E_{K,L}\bigg]\notag\\
%&\quad+4\alpha^2\tr\bigg[\sum_{t\geq0}(x'_t)(x'_t)^\top E_{K,L}^\top(R^u+B^\top P_{K,L}B)^{-1}E_{K,L}\bigg]\leq 0,
&\leq -2\alpha\tr\bigg[\sum_{t\geq0}(x'_t)(x'_t)^\top E_{K,L}^\top(R^u+B^\top P_{K,L}B)^{-1}E_{K,L}\bigg]\leq 0,
\#
where $\{x'_t\}_{t\geq 0}$ is the sequence of states generated by $(K',L)$ with $x'_0=x$ 
for any $x\in\RR^d$. Hence, we show the monotonicity of $P_{K',L}$, i.e., $P_{K',L}\leq P_{K,L}$, after one-step update of \eqref{equ:inner_gauss_newton}. 

If $\alpha\cdot\text{Upper}_K> \delta_K$ for some  $\alpha\in[0,1/2]$, the one-step update \eqref{equ:inner_gauss_newton} may go beyond the stabilizing region with radius $\delta_K$. However, we can show as follows that for all the $\alpha$ changing from $0$ to $1/2$, the updated $K'$ remains to be stabilizing. First, there must exist some stepsize $\beta\in(0,1/2)$ such that $\beta\cdot\text{Upper}_K\leq \delta_K$. Let the arrived  control gain be $K'_{\beta}$. Then  by the argument in the previous paragraph, we know that $P_{K'_{\beta},L}\leq P_{K,L}$. Thus, any $K'$ such that $\|K'-K'_{\beta}\|\leq \delta_K$ is also stabilizing, including the control gain $K{''}_{\beta}$ updated from $K$ using stepsize $2\beta$. If $2\beta\geq 1/2$, then simply choosing  $\alpha\in[0,1/2]$ ensures the stability of $K'$; if $2\beta< 1/2$, then $K{''}_{\beta}$ can also be shown to lead to that $P_{K{''}_{\beta},L}\leq P_{K,L}$ using the argument in \eqref{equ:contract_P_K_gn}, which further implies that any $K'$ such that $\|K'-K{''}_{\beta}\|\leq \delta_K$ is also stabilizing. This enables the choice of stepsize $3\beta$ starting from $K$. Repeating the argument concludes that any choice of $\alpha\in[0,1/2]$ guarantees the stability of the update. Thus, the linear convergence rate of Gauss-Newton update can be obtained by the proof of Theorem $7$ in   \cite{fazel2018global}. In particular, along the iteration $\tau\geq 0$, the sequence $\{P_{K_{\tau},L}\}_{\tau\geq 0}$ satisfies $P_{K_{\tau},L}\geq P_{K_{\tau+1},L}\geq P_{K(L),L}$.

The proof for  natural PG update is similar, except that the upper bound for the stepsize choice is changed from $1/2$ to $1/\|R^u+B^\top P_{K,L}B\|$ (see Lemma $15$ in \cite{fazel2018global}), which can also be covered by finite times  of some $\beta>0$.

For the stability proof of the  gradient update, such an idea of using \eqref{equ:stability_trick} and  the monotonicity of $P_{K,L}$ to upper bound the spectral radius $\rho(\tilde A_L-BK)$ does not apply, since only the monotonicity of $\cC(K,L)$ instead of $P_{K,L}$ can be shown. Hence, we follow the stability argument in \cite{fazel2018global} for the gradient update; see Appendix \S C.4 therein. 

%The proof for gradient update is more involved, since the upper bound for the 
% stepsize $\alpha$ that ensures the decrease of $P_K$ (see the proof of Lemma $24$ in \cite{fazel2018global}) depends on $\sigma_{\min}(\tilde Q_L)$, which may be zero for given $L$. Hence, we need to re-establish the one step progress  of the gradient update. 
% XXXXXXXXXX
 
%To this end, 

With the stability  arguments  verified as above, the last two arguments of the proposition  on the algorithm convergence then follow from  Theorem $7$ in   \cite{fazel2018global},  which completes the proof. \hfil\qed

\subsection{Proof of Theorem \ref{thm:full_grad_conv}}\label{subsec:proof_full_grad_conv}
%%%%%%%%%%%%%%%%%%%%%%%%%%%%%%%%%%%%%%%%%%%%%%%%%%
%%%%%%%%%%%%%%%%PROOF%%%%%%%%%%%%%%%%%%%%%%%%%%%%%
%%%%%%%%%%%%%%%%%%%%%%%%%%%%%%%%%%%%%%%%%%%%%%%%%%

We now prove the global convergence of the nested-gradient algorithms. 
First, since the projection set $\Omega\subseteq \underline{\Omega}$, we have from Lemma \ref{lemma:optim_of_KL} that the control pair sequence $\{K(L_t),L_t\}_{t\geq 0}$ generated by the projected updates are always stabilizing, namely, the stability argument holds regardless of the choice of the stepsize $\eta$. Moreover,  since  $\Omega\subseteq \underline{\Omega}$,  
the inner-loop   updates   in 
\eqref{equ:inner_gd}-\eqref{equ:inner_gauss_newton} converge   to    $K(L_t)$   with linear rate by Proposition \ref{prop:inner_full_grad_conv}. 

To establish the global convergence result, we first need the following lemma that  characterizes the difference in value functions for any two pairs of control gains  
 $(K(L), L)$ and $(K(L'), L')$ when  $L,L'\in\Omega$.

 \begin{lemma}[Value Difference Between $(K(L), L)$ and $(K(L'), L')$] \label{lemma:Cost_Diff3}
 	For any  matrices $L,L'\in\Omega$, 
% 	the control pairs $(K(L), L)$ and  $(K(L'),L')$
% 	%, $(K,L)$, and $K,L'$ all 
% 	are both stablizing.
 	 recalling the definition of $W_L$ in \eqref{eq:define_wll},    it follows that
 	\$
 	 V_{L'}^*(x) - V_{L}^*(x)  &\geq 2 \tr \biggl [ \sum_{t\geq 0}x_t'^* (x_t'^*)^\top (L' - L)^\top F_{L}^*  \biggr ] -  \tr \biggl [ \sum_{t\geq 0}x_t'^* (x_t'^*)^\top (L' - L)^\top   W_{L} (L'- L)  \biggr ], 
 	\$
 	where  $\{x_t'^*\}_{t\geq 0}$ is the sequence of states generated by the control pairs   $( K(L'), L') $ with 
 	 $x_0'^*=x$.  Also,  
% 	and the matrix $R^v - C^\top P_{L }^*C>0$. 
% 	, and the matrix $R^v - C^\top P_{L }^*C>0$.  
%Recall the definition of $W_L$ in \eqref{eq:define_wll} as 
% 	$$
% 	W_{L} = R^v - C^\top P_{L }^*C + C^\top  P_{L}^*B (R^u+B^\top P_{L}^* B)^{-1} B^\top P_{L}^*C. 
% 	 $$ 
 	 letting $\tilde K(L,L')=   K(L)  -   (R^u+B^\top P_{L}^* B)^{-1} B^\top P_{L}^*C(L'-L)$, we have that for any $x$
 	 \$
 	  	   V_{ L'}^*(x) - V_{L}^*(x) & \leq 2 \tr \biggl [ \sum_{t\geq 0}\tilde{x}'_t \tilde{x}'^\top_t (L' - L)^\top F_{L}^*  \biggr ] -  \tr \biggl [  \sum_{t\geq 0}\tilde{x}'_t \tilde{x}'^\top_t (L' - L)^\top  W_{L} (L'- L)  \biggr ], 
 	 \$
 	 where $\{\tilde{x}'_t\}_{t\geq 0}$ is the  sequence of states generated by the control pairs $(\tilde K(L,L'), L' )$, with 
 	 $\tilde{x}'_0=x$. 
% 	 where we define $\Sigma_{ \tilde K(L,L'), L' } $ and $ \Sigma_{L' } ^*$ as the covariance matrices computed using trajectories sampled from the control gain $(\tilde K, L' )$ and $( K(L'), L') $, respectively.
 	 \end{lemma}

\begin{proof}
%The proof is similar to the proof of Lemma \ref{lemma:Cost_Diff2}. 
%We thus only emphasize the difference between them.
First   by Lemma \ref{lemma:optim_of_KL}, both $P_L^*>0$ and $P_{L'}^*>0$,   $(K(L),L)$ and $(K(L'),L')$ are stabilizing. 
Also, from Lemma \ref{lemma:Cost_Diff}, 
 we have that for any stabilizing control pair $(K',L')$ and  any $x \in \RR^d$   
\$
	V_{K',L'}(x)-V_{L}^* (x)=\sum_{t\geq 0}A_{ L}^*(x'_t,u'_t,v'_t),
\$	
with $x'_0=x$, $u'_t=-K'x'_t$, and $v'_t=-L'x'_t$. 
Moreover,  by definitions of $E_{K,L}$ in \eqref{equ:def_E_F_mu_2} and $K(L)$ in \eqref{equ:KL_def},  we have   $E_{L}^* = 0$, which combined with \eqref{equ:Advantage_Diff}  further gives  that
\#\label{equ:Advantage_Diff3}
A_{L }^* (x,-K'x,-L'x)&=   x^\top (K'-K(L) )^\top (R^u+B^\top P_{L}^* B)(K'-K(L) )x\\
&\qquad + 2x^\top(L' - L)^\top F_{L}^*x  +  x ^\top (L'-L)^\top (-R^v+C^\top P_{L }^*C)(L'-L)x  \notag \\
&\qquad  +2x^\top (L'-L)^\top C^\top P_{L}^*B(K'-K(L))x.\notag
\# 
Completing the squares w.r.t. $K'$ in \eqref{equ:Advantage_Diff3} yields 
%\issue{!!It is true that $\cC(K^*,L^*)-\cC(K(L_t),L_t)$ is finite, but it cannot be bounded by the norm of $F_{L_t}^*$!! We can only do this locally.} 
\#\label{eq:trash2}  
A_{L }^* (x,-K'x,-L'x) &= 2x^\top(L'-L)^\top F_{L}^*x+x^\top (L'-L)^\top (-R^v+C^\top P_{L}^* C)(L'-L)x \notag \\ 
&\qquad + x^\top \bigl [ K' - K(L) +   ( R^u+B^\top P_{L }^*B)^{-1}B^\top P_{L }^*C(L'-L) \bigr ]^\top  ( R^u+ B^\top P_{L }^*B) \notag \\
&\qquad \qquad \qquad \bigl [ K' - K(L) +   ( R^u+B^\top P_{L }^*B)^{-1}B^\top P_{L }^*C(L'-L)  \bigr ] x \notag \\
&\qquad -x ^\top ( L'-L) ^\top C^\top  P_{L }^* B ( R^u+ B^\top P_{L }^*B)^{-1}B^\top P_{L }^*C(L'-L)x \notag \\
&\geq  2x^\top(L' - L)^\top F_{L}^*x  -  x ^\top (L' - L ) ^\top W_{L} (L'- L) x,
\# 
where  $W_{L }$ is as defined in \eqref{eq:define_wll}, and the last inequality follows from the fact that $R^u+ B^\top P_{L }^*B\geq 0$ (since $P_{L }^*>0$).  
Thus, replacing  $K'$ in   \eqref{eq:trash2} with $K(L')$ yields  
\$
& V_{L'}^*(x)- V_{L}^*(x)  \geq  2 \tr \biggl [ \sum_{t\geq 0}x_t'^* (x_t'^*)^\top(L'-L)^\top F_{L}^* \biggr ] - \tr \biggl [ \sum_{t\geq 0}x_t'^* (x_t'^*)^\top (L' - L ) ^\top W_{L } (L'- L)  \biggr ], 
%&\cC\bigl ( K(L') , L' \bigr )-  \cC \bigl (K(L), L\bigr )   = \EE_{x_0\sim\cD} \biggl [ \sum_{t\geq 0} A_L^* \bigl (x_t' , -  K(L') x_t', - L' x_t' \bigr ) \biggr ]\notag \\
%&\qquad \geq   2 \tr \bigl [ \Sigma_{L' } ^*(L' - L)^\top F_{K(L) , L}  \bigr ] -  \tr \bigl [ \Sigma_{L' } ^*  (L' - L)^\top   W_{L} (L'- L)  \bigr ], 
\$
where $x_0'^*=x$ and $x_{t+1}'^*=[A-BK(L')-CL']\cdot x_t'^*$ follows the trajectory generated by the control $(K(L'),L')$. This completes the proof of the lower bound.

On the other hand,  by defining  $\tilde K(L,L')=   K(L)  -    (R^u+B^\top P_{L}^* B)^{-1} B^\top P_{L}^*C(L'-L)$, and letting $K'=\tilde K(L,L')$  in \eqref{eq:trash2}, we obtain that   
\#\label{equ:trash3}
&V_{\tilde K(L,L'), L'}(x)- V_{L}^*(x)=2 \tr \biggl [ \sum_{t\geq 0}\tilde{x}'_t \tilde{x}'^\top_t (L'-L)^\top F_{L}^* \biggr ] - \tr \biggl [ \sum_{t\geq 0}\tilde{x}'_t \tilde{x}'^\top_t  (L' - L ) ^\top W_{L } (L'- L)  \biggr ]
\#
where $\tilde{x}'_0=x$, $\tilde{x}'_{t+1}=[A-B \tilde K(L,L')-CL']\cdot\tilde{x}'_t$ follows the trajectory generated by the control $(\tilde K(L,L'),L')$. Moreover, since $P_{L'}^*> 0$ and the optimality of $K(L')$ from  Lemma \ref{lemma:optim_of_KL}, we have  $V_{ K(L'), L'}(x)\leq V_{\tilde K(L,L'), L'}(x)$. Therefore, \eqref{equ:trash3} further gives   
\$
V_{L'}^*(x)- V_{L}^*(x)  \leq 
 2 \tr \biggl [ \sum_{t\geq 0}\tilde{x}'_t \tilde{x}'^\top_t (L'-L)^\top F_{L}^* \biggr ] - \tr \biggl [ \sum_{t\geq 0}\tilde{x}'_t \tilde{x}'^\top_t  (L' - L ) ^\top W_{L } (L'- L)  \biggr ],
%\cC\bigl (  \tilde K(L,L'), L' \bigr )- \cC \bigl (K(L), L\bigr )  =  2 \tr \bigl [ \Sigma_{ \tilde K, L' } (L' - L)^\top F_{K(L) , L}  \bigr ] -  \tr \bigl [  \Sigma_{ \tilde K, L' } (L' - L)^\top W_{L} (L'- L)  \bigr ], 
\$
%and the second inequality is due to the optimality of $K(L')$ given in Lemma \ref{lemma:optim_of_LK_and_KL}.  
which proves the upper bound in the lemma, and thus completes the proof. 
\end{proof}

Moreover, we establish  the following important  lemma on the  perturbation of the covariance matrix $\Sigma_L^*$, whose proof is a little involved and deferred to \S\ref{subsec:lemma_proof_perturb_Sigma_L}.

\begin{lemma}[Perturbation of $\Sigma_{L}^*$]\label{lemma:perturb_Sigma_L} 
Under Assumption \ref{assum:invertibility}, 
for any $L,L'\in\Omega$, there exist  some constants $\cB^L_{\Omega},\cB^P_{\Omega},\cB^K_{\Omega}>0$, such that  if 
\footnotesize
\#\label{equ:pert_Sigma_cond_bnd}
\|L'-L\|\leq \min\Bigg\{\cB^L_{\Omega},\frac{\|B\|\big[\cB^P_{\Omega}\|\tilde A_{L}-BK(L)\|+\|P_{L}^*\|\|C\|\big]}{\cB^P_{\Omega}\|B\|\|C\|},\frac{2\big(\|\tilde A_L-BK(L)\|+1\big)\big(\cB_{\Omega}^K\|B\|+\|C\|\big)}{\big(\cB_{\Omega}^K\big)^2\|B\|^2+\|C\|^2+2\cB_{\Omega}^K\|B\|\|C\|}\Bigg\},
\#
\normalsize
if follows  that
\$
\| \Sigma_{L'}^* - \Sigma_{L}^* \| \leq 4\big(\|\tilde A_L-BK(L)\|+1\big)\big(\cB_{\Omega}^K\|B\|+\|C\|\big) \cdot \| L' - L \|. 
\$
\end{lemma}

In addition,   we can also bound the norm of the nested-gradient $\|\nabla_L \tilde{\cC}(L)\|$, and the norms of the gradient-mappings, as follows. 

\begin{lemma}\label{lemma:bnd_gradient_mapping_norm}
	For any $L\in\Omega$, recall the gradient mappings $\hat{G}_{L}^*,\tilde{G}_{L}^*,\check{G}_{L}^*$ defined in   \eqref{equ:def_check_G}, then
	\$
	&\frac{2}{\sqrt{q}}\cdot\max\Big\{\mu\nu\big\|\hat{G}_{L}^*\big\|,\mu\big\|\tilde G^*_L\big\|,\big\|\check{G}^*_L\big\|\Big\}\leq \big\|\nabla_L \tilde{\cC}(L)\big\|\\
	&\quad\leq \frac{2\cC(K(L),L)}{\zeta}\sqrt{\frac{\|W_L\|[\cC(K^*,L^*) - \cC(K(L),L)]}{\mu}},
	\$
	where $q=\min\{m_2,d\}$.
\end{lemma}
\begin{proof}
	Recall that by definition $\nabla_L \tilde{\cC}(L)=2F_L^*\Sigma_L^*$. Hence, by Lemma \ref{lemma:bound_P_L_Sigma_L}, 
	\#\label{equ:bnd_grad_norm_trash_1}
	&\|\nabla_L \tilde{\cC}(L)\|^2\leq  4\tr\Big(\Sigma_L^*F_L^{*^{\top}}F_L^*\Sigma_L^*\Big)\leq \|\Sigma_L^*\|^2\tr\big(F_L^{*^{\top}}F_L^*\big)\notag\\
	&\quad\leq \frac{[\cC(K(L),L)]^2}{\zeta^2}\tr\big(F_L^{*^{\top}}F_L^*\big).
	\#
	On the other hand, by plugging-in $L'=L+W_L^{-1}F_L^*$, we have 
	\#\label{equ:bnd_grad_norm_trash_2}
	&\cC(K^*,L^*) - \cC(K(L),L)\geq 
	 \cC(K(L'),L') - \cC(K(L),L)\notag\\  &\quad\geq 2 \tr \bigl [ \Sigma_{L'}^* (L' - L)^\top F_{L}^*  \bigr ] -  \tr \bigl [ \Sigma_{L'}^* (L' - L)^\top   W_{L} (L'- L)  \bigr ]=\tr \bigl ( \Sigma_{L'}^* F_{L}^{*^\top} W_L^{-1} F_{L}^*  \bigr )\notag\\
	 &\quad \geq \frac{\mu}{\|W_L\|}\tr \bigl ( F_{L}^{*^\top} F_{L}^*  \bigr ), 
	\#
	where the first inequality is due to $\cC(K^*,L^*)\geq \cC(K(L'),L')$ for any $L'$, the second inequality follows by  taking expectation on both sides of the lower bound in Lemma \ref{lemma:Cost_Diff3},  
%	the third  inequality follows by completing the squares, 
	and the last inequality is due to $\Sigma_{L'}^*\geq \mu\cdot \Ib$ and $\sigma_{\min}(W_L^{-1})=1/\|W_L\|$.
	Combining \eqref{equ:bnd_grad_norm_trash_1} and \eqref{equ:bnd_grad_norm_trash_2} yields the upper bound on $\|\nabla_L \tilde{\cC}(L)\|$.  
	
	Moreover, by definitions of $\hat{G}_{L}^*,\tilde{G}_{L}^*,\check{G}_{L}^*$, we have
	\small
\#
		 &\tr\big(W_L^{*^{1/2}}\hat{G}_{L}^*\Sigma_L^* \hat{G}_{L}^{*^\top}W_L^{*^{1/2}}\big)\leq \tr\big(W_L^{*^{1/2}}W_L^{*^{-1}}F_L^*\Sigma_L^* \hat{G}_{L}^{*^\top}W_L^{*^{1/2}}\big)\leq \big\|F_L^*\Sigma_L^*\big\|_F\cdot\big\|\hat{G}_{L}^*\big\|_F,\label{equ:bnd_grad_norm_trash_3}\\
		  &\tr\big(\tilde{G}_{L}^*\Sigma_L^* \tilde{G}_{L}^{*^\top}\big)\leq \tr\big(F_L^*\Sigma_L^* \tilde{G}_{L}^{*^\top}\big)\leq \big\|F_L^*\Sigma_L^*\big\|_F\cdot\big\|\tilde{G}_{L}^*\big\|_F,\label{equ:bnd_grad_norm_trash_4}\\
		   &\tr\big(\check{G}_{L}^*\check{G}_{L}^{*^\top}\big)\leq \tr\big(F_L^*\Sigma_L^* \check{G}_{L}^{*^\top}\big)\leq \big\|F_L^*\Sigma_L^*\big\|_F\cdot\big\|\check{G}_{L}^*\big\|_F,\label{equ:bnd_grad_norm_trash_5}
	\#
	\normalsize
	where for all \eqref{equ:bnd_grad_norm_trash_3}-\eqref{equ:bnd_grad_norm_trash_5}, the first inequality is due to 
	Lemma \ref{lemma:proj_prop},  and the second one follows from Cauchy-Schwartz inequality. Note that 
	\$
	&\tr\big(W_L^{*^{1/2}}\hat{G}_{L}^*\Sigma_L^* \hat{G}_{L}^{*^\top}W_L^{*^{1/2}}\big)\geq \mu\sigma_{\min}(W_L) \big\|\hat{G}_{L}^*\big\|_F^2\geq \mu\nu \big\|\hat{G}_{L}^*\big\|_F^2,
%	\geq \mu\sigma_{\min}(W_L)\cdot \big\|\hat{G}_{L}^*\big\|^2\\
	&\tr\big(\tilde{G}_{L}^*\Sigma_L^* \tilde{G}_{L}^{*^\top}\big)\geq \mu \|\tilde{G}_{L}^*\|^2_F,
%	\geq \mu \|\tilde{G}_{L}^*\|^2, 
	\$
	which uses the fact that $\sigma_{\min}(W_L)\geq \sigma_{\min}(W_{L^*})=\nu$ from Lemma \ref{lemma:bound_P_L_Sigma_L}. 
	This together with \eqref{equ:bnd_grad_norm_trash_3}-\eqref{equ:bnd_grad_norm_trash_5} gives that
	\#
		\max\Big\{\mu\nu\big\|\hat{G}_{L}^*\big\|_F,\mu\big\|\tilde G^*_L\big\|_F,\big\|\check{G}^*_L\big\|_F\Big\}\leq {\big\|F_L^*\Sigma_L^*\big\|_F}\leq \frac{\sqrt{q}}{2}\cdot\big\|\nabla_L \tilde{\cC}(L)\big\|\label{equ:bnd_grad_norm_trash_6},
	\#
	where the second inequality uses the fact that $\|F_L^*\Sigma_L^*\|_F^2=\|\nabla_L \tilde{\cC}(L)\|_F^2/4$ and $\|X\|_F\leq \sqrt{r}\|X\|\leq \sqrt{\min\{m,n\}}\cdot\|X\|$ for matrix $X\in\RR^{m\times n}$ of rank $r$. 
	Dividing both sides by $\sqrt{q}/2$, and using the fact that $\|X\|_F\geq \|X\|$, we obtain the first inequality in the lemma.
\end{proof} 
 
Now we are ready to establish the global  convergence  of the three proposed  algorithms.

\vspace{10pt}
\noindent{\textbf{Projected Gauss-Newton Nested-Gradient:}}
\vspace{4pt}

First note that  the projected Gauss-Newton nested-gradient update in \eqref{equ:algorithm_nested_precond_L} can be written as 
\#\label{eq:somehow_new_gn}
L_{t+1}=\PP^{GN}_{\Omega}\big[L_t+2\eta \cdot W^{-1}_{L_t}F_{L_t}^*\big]=L_t+2\eta\cdot\hat{G}_L^*,
\#
where we recall that $\PP_{\Omega}^{GN}$ is the projection operator defined in \eqref{equ:def_proj_NG} and the gradient mapping $\hat{G}_L^*$  is defined in \eqref{equ:def_check_G}. 
 Since both $L_t$ and $L_{t+1}$ lie in $\Omega$, by the lower bound in Lemma \ref{lemma:Cost_Diff3} and \eqref{eq:somehow_new_gn}, we can bound the difference between $V_{L_{t+1} }^*$  and $V_{L_t}^*$ as
 \$
 V_{L_{t+1}}^*(x) - V_{L_t}^*(x) \geq 2\eta\tr \biggl [ \sum_{t\geq 0}x_t'^* x_t'^{*^\top} \big(\hat{G}_{L_t}^{*^\top} F_{L_t}^*+ F_{L_t}^{*^\top}\hat{G}_{L_t}^* \big) \biggr ] -  4\eta^2\tr \biggl ( \sum_{t\geq 0}x_t'^* x_t'^{*^\top} \hat{G}_{L_t}^{*^\top} W_{L_t}\hat{G}_{L_t}^*  \biggr ), 
 \$ 
 where $\{{x}^*_{\tau}\}_{\tau\geq 0}$ is the state sequence generated by the control $(K(L_{t+1}),L_{t+1})$ with ${x}^*_{0}=x$.
Taking expectation over $x_0\sim\cD$, we have 
\#\label{eq:GN_bound_C}
&\cC\bigl (K(L_{t+1}), L_{t+1} \bigr ) - \cC\bigl(K(L_{t}), L_{t}\bigr ) \notag \\
& \qquad  \geq 2 \eta \cdot \tr \Bigl [ \Sigma_{L_{t+1} }^*  \big(\hat{G}_{L_t}^{*^\top} F_{L_t}^*+ F_{L_t}^{*^\top}\hat{G}_{L_t}^* \big) \Bigr ] - 4 \eta ^2 \cdot \tr \Bigl ( \Sigma_{L_{t+1}}^* \hat{G}_{L_t}^{*^\top} W_{L_t}\hat{G}_{L_t}^*  \Bigr ).
\# 

 In the following, we bound the two terms on the right-hand side of \eqref{eq:GN_bound_C} separately. For the first term, since $L_t \in\Omega$, applying  the property of $\PP_{\Omega}^{GN}$ in Lemma \ref{lemma:proj_prop} with $L_1 = L_t + 2 \eta \cdot W_{L_t}^{-1} F_{L_t}^* $ and $L_2 = L_t$ yields  
 \$ 
\tr \Bigl ( W_{L_t}^{-1}F_{L_t}^* \Sigma_{L_t}^*  \hat{G}_{L_t}^{*^\top} W_{L_t}\Bigr )=\tr \Bigl ( F_{L_t}^* \Sigma_{L_t}^*  \hat{G}_{L_t}^{*^\top} \Bigr ) \geq  \tr \Bigl (\hat{G}_{L_t}^{*} \Sigma_{L_t}^*  \hat{G}_{L_t}^{*^\top}  W_{L_t}  \Bigr )  , 
\$ 
which implies that 
 \#\label{eq:GN_bound_C1}
& \tr \Bigl [ \Sigma_{L_{t+1} }^*  \big(\hat{G}_{L_t}^{*^\top} F_{L_t}^*+ F_{L_t}^{*^\top}\hat{G}_{L_t}^* \big) \Bigr ] \notag \\
& \qquad = \tr \Bigl [ \Sigma_{L_{t} }^*  \big(\hat{G}_{L_t}^{*^\top} F_{L_t}^*+ F_{L_t}^{*^\top}\hat{G}_{L_t}^* \big) \Bigr ] + \tr \Bigl [ \bigl( \Sigma_{L_{t+1} }^* -  \Sigma_{L_{t} }^* \bigr )  \big(\hat{G}_{L_t}^{*^\top} F_{L_t}^*+ F_{L_t}^{*^\top}\hat{G}_{L_t}^* \big) \Bigr ] \notag \\
 & \qquad \geq 2 \tr \Bigl ( \Sigma_{L_t}^*  \hat{G}_{L_t}^{*^\top}W_{L_t} \hat{G}_{L_t}^{*}   \Bigr ) - \bigl \|\Sigma_{L_{t+1} }^* -  \Sigma_{L_{t} }^*  \bigr \| \cdot \tr \Bigl [  \big(\hat{G}_{L_t}^{*^\top} F_{L_t}^*+ F_{L_t}^{*^\top}\hat{G}_{L_t}^* \big) \Bigr ] \notag\\
 & \qquad \geq 2\mu\nu \big\|\hat{G}_{L_t}^*\big\|^2_F - 16\eta\big(\|\tilde A_{L_t}-BK(L_t)\|+1\big)\big(\cB_{\Omega}^K\|B\|+\|C\|\big)  \big\|\hat{G}_{L_t}^*\big\|^2_F\big\|F_{L_t}^*\big\|_F.
  \#
  The first inequality uses triangle inequality. 
The last inequality uses the following facts: i) since $\sigma_{\min}(\Sigma_{L_t}^*)\geq \sigma_{\min}(\EE_{x_0\sim\cD} x_0x_0^\top)=\mu$ and $\sigma_{\min}(W_{L_t})\geq \sigma_{\min}(W_{L^*})=\nu$ (see Lemma \ref{lemma:bound_P_L_Sigma_L}), it follows that
\$
\tr \Bigl ( \Sigma_{L_t}^*  \hat{G}_{L_t}^{*^\top}W_{L_t} \hat{G}_{L_t}^{*}   \Bigr ) \geq\nu\tr \Bigl ( \Sigma_{L_t}^*  \hat{G}_{L_t}^{*^\top}\hat{G}_{L_t}^{*}   \Bigr )\geq  \mu\nu \big\|\hat{G}_{L_t}^*\big\|^2_F;
\$
ii) from Lemma \ref{lemma:perturb_Sigma_L}, if 
\$
\|L_{t+1}-L_t\|=2\eta \|\hat{G}_{L_t}^*\|\leq \cK_{\Omega}^{L},
\$ 
where 
\small
\#\label{equ:def_K_Omega^L}
\cK_{\Omega}^{L}=\inf_{L\in\Omega}\min\Bigg\{\cB^L_{\Omega},\frac{\|B\|\big[\cB^P_{\Omega}\|\tilde A_{L}-BK(L)\|+\|P_{L}^*\|\|C\|\big]}{\cB^P_{\Omega}\|B\|\|C\|},\frac{2\big(\|\tilde A_L-BK(L)\|+1\big)\big(\cB_{\Omega}^K\|B\|+\|C\|\big)}{\big(\cB_{\Omega}^K\big)^2\|B\|^2+\|C\|^2+2\cB_{\Omega}^K\|B\|\|C\|}\Bigg\},
\#
\normalsize 
is the infimum for the required upper-bound on $\|L'-L\|$ in Lemma \ref{lemma:perturb_Sigma_L}, i.e., \eqref{equ:pert_Sigma_cond_bnd}, then the perturbation $ \|\Sigma_{L_{t+1} }^* -  \Sigma_{L_{t} }^*  \|$ can be bounded as
\$
\|\Sigma_{L_{t+1} }^* -  \Sigma_{L_{t} }^*  \|&\leq 4\big(\|\tilde A_{L_t}-BK(L_t)\|+1\big)\big(\cB_{\Omega}^K\|B\|+\|C\|\big)\cdot \|L_{t+1}-L_t\|\\
&\leq 8\eta\big(\|\tilde A_{L_t}-BK(L_t)\|+1\big)\big(\cB_{\Omega}^K\|B\|+\|C\|\big)\cdot\|\hat{G}_{L_t}^*\|;
\$
iii) Cauchy-Schwartz inequality yields 
\$
\tr [  (\hat{G}_{L_t}^{*^\top} F_{L_t}^*+ F_{L_t}^{*^\top}\hat{G}_{L_t}^* ) ]\leq 2\big\|\hat{G}_{L_t}^*\big\|_F\big\|F_{L_t}^*\big\|_F.
\$

% the bound on $ \|\Sigma_{L_{t+1} }^* -  \Sigma_{L_{t} }^*  \|$ from Lemma \ref{lemma:perturb_Sigma_L}, provided that 
%\$
%\|L_{t+1}-L_t\|=2\eta \|\hat{G}_{L_t}^*\|\leq \cK_{\Omega}^{L},
%\$ 
%where 
%\small
%\#\label{equ:def_K_Omega^L}
%\cK_{\Omega}^{L}=\inf_{L\in\Omega}\min\Bigg\{\cB^L_{\Omega},\frac{\|B\|\big[\cB^P_{\Omega}\|\tilde A_{L}-BK(L)\|+\|P_{L}^*\|\|C\|\big]}{\cB^P_{\Omega}\|B\|\|C\|},\frac{2\big(\|\tilde A_L-BK(L)\|+1\big)\big(\cB_{\Omega}^K\|B\|+\|C\|\big)}{\big(\cB_{\Omega}^K\big)^2\|B\|^2+\|C\|^2+2\cB_{\Omega}^K\|B\|\|C\|}\Bigg\},
%\#
%\normalsize 
%takes the infimum for the upper-bound requirement on $\|L'-L\|$ in Lemma \ref{lemma:perturb_Sigma_L}, i.e., \eqref{equ:pert_Sigma_cond_bnd}. }
Note that by definition \eqref{equ:def_K_Omega^L}, $\cK_{\Omega}^{L}>0$ since it is the infimum of a strictly positive function of $L$ that is continuous over a compact set $\Omega$.  
Combined with the bound on $\|\hat G_{L_t}^*\|$ from  Lemma \ref{lemma:bnd_gradient_mapping_norm},    
we further obtain the  requirement for the stepsize $\eta$:
\#\label{equ:stepsz_cond_GN_1}
\eta \leq\frac{\cK_{\Omega}^{L}\zeta\mu\nu}{2\sqrt{q}\cdot\cC(K(L_t),L_t)}\sqrt{\frac{\mu}{\|W_{L_t}\|[\cC(K^*,L^*) - \cC(K(L_t),L_t)]}}.
\#
\normalsize 
{Moreover, notice that 
\#\label{eq:GN_bound_C2}
\tr \Bigl ( \Sigma_{L_{t+1}}^* \hat{G}_{L_t}^{*^\top} W_{L_t}\hat{G}_{L_t}^*  \Bigr )\leq \big\|\Sigma_{L_{t+1}}^*\big\|_F\big\|W_{L_t}\big\|_F\big\|\hat{G}_{L_t}^*\big\|_F^2\leq \frac{\sqrt{m}\cdot\cC(K(L_t),L_t)\|R^v\|_F}{\mu} \big\|\hat{G}_{L_t}^*\big\|_F^2,
\#
where the first inequality is due to Cauchy-Schwartz inequality, and the second one follows from Lemma \ref{lemma:bound_P_L_Sigma_L} and the fact that $\|X\|_F\leq \sqrt{r}\|X\|$ for any matrix $X$ with rank $r$. 
Substituting \eqref{eq:GN_bound_C1} and  \eqref{eq:GN_bound_C2} into \eqref{eq:GN_bound_C} yields
\#\label{equ:cost_diff_bnd_trash_1}
\cC\bigl (K(L_{t+1}), L_{t+1} \bigr ) - \cC\bigl(K(L_{t}), L_{t}\bigr ) \geq& {4}{\mu}\nu\eta \big\|\hat{G}_{L_t}^*\big\|_F^2\bigg[1 -\eta \frac{\sqrt{m}\cdot\cC(K(L_t),L_t)\|R^v\|_F}{\mu^2\nu} \\
&  - \frac{8\eta}{\mu\nu}\big(\|\tilde A_{L_t}-BK(L_t)\|+1\big)\big(\cB_{\Omega}^K\|B\|+\|C\|\big) \big\|F_{L_t}^*\big\|_F\bigg],\notag
\#
which gives us another requirement for the stepsize $\eta$: 
\small
\#\label{equ:stepsz_cond_GN_2}
\eta\leq \frac{1}{2}\cdot\bigg[ \frac{\sqrt{m}\cdot\cC(K(L_t),L_t)\|R^v\|_F}{\mu^2\nu}  + \frac{8}{\mu\nu}\big(\|\tilde A_{L_t}-BK(L_t)\|+1\big)\big(\cB_{\Omega}^K\|B\|+\|C\|\big) \big\|F_{L_t}^*\big\|_F\bigg]^{-1}.
\# }
\normalsize 

By requiring both   \eqref{equ:stepsz_cond_GN_1} and \eqref{equ:stepsz_cond_GN_2}, we can     further bound \eqref{equ:cost_diff_bnd_trash_1} as
\#\label{equ:cost_diff_bnd_trash_2}
\cC\bigl (K(L_{t+1}), L_{t+1} \bigr ) - \cC\bigl(K(L_{t}), L_{t}\bigr ) \geq{2}{\mu}\nu\eta \big\|\hat{G}_{L_t}^*\big\|_F^2. 
\#
Note that both the upper bounds   in \eqref{equ:stepsz_cond_GN_1} and \eqref{equ:stepsz_cond_GN_2} are lower bounded above from zero, since  the numerators  of both bounds are constants, and 
the denominators are upper bounded for $L\in\Omega$, due to the boundedness of $P_L^*$, $\cC(K(L),L)$, and $L$. 
Summing up both sides of   \eqref{equ:cost_diff_bnd_trash_2} from $0$ to $t\geq 1$ yields
\$
 {\frac{1}{t}\sum_{\tau=0}^{t-1}\big\|\hat{G}_{L_\tau}^*\big\|_F^2}
%&\leq \sqrt{\frac{\cC\big(K_0,L(K_0)\big)-\cC\big(K_T,L(K_T)\big)}{2T\eta\mu \sigma_{\min}(W_{K_0}^{-1})}}\\
\leq {\frac{\cC\big(K^*,L^*\big)-\cC\big(K(L_0),L_0\big)}{2\mu\nu\eta t}},
\$
 which shows  that $(K(L_t),L_t)$ converges to the NE with sublinear rate, namely, the sequence of the  average of the gradient mapping norm square $\big\{t^{-1}\sum_{\tau=0}^{t-1}\big\|\hat{G}_{L_\tau}^*\big\|_F^2\big\}_{t\geq 1}$ converges to zero with $\cO(1/{t})$ rate, so does the sequence $\big\{t^{-1}\sum_{\tau=0}^{t-1}\big\|\hat{G}_{L_\tau}^*\big\|^2\big\}_{t\geq 1}$.

\vspace{10pt}
\noindent{\textbf{Projected Natural  Nested-Gradient:}}
\vspace{4pt} 

The proof for the   projected natural NG update   \eqref{equ:algorithm_natural_nested_L} is similar. We will only cover the argument that is different from above. 
Note that \eqref{equ:algorithm_natural_nested_L} 
 can be written as 
\#\label{eq:somehow_new_ng}
L_{t+1} = \PP_{\Omega}^{NG}
\big[L_t + 2\eta \cdot F_{L_t}^*\big]=
%L_t+2\eta\cdot\frac{\PP_{\Omega}\big[L_t + 2\eta \cdot F_{L_t}^*\big]-L_t}{2\eta}=
L_t+2\eta \cdot \tilde{G}_{L_t}^*,
\#
where $\PP_{\Omega}^{NG}$ is   defined in \eqref{equ:def_proj_NG} with weight matrix $\Sigma_{L_t}^*$  and   $\tilde{G}_{L}^*$ is defined in \eqref{equ:def_check_G}. Then by Lemma \ref{lemma:Cost_Diff3} and taking expectation $x_0\sim \cD$, we also have \eqref{eq:GN_bound_C} but with $\hat{G}_{L}^*$ replaced by $\tilde{G}_{L}^*$. Then, by the property of $\PP_{\Omega}^{NG}$ and letting $L_1 = L_t + 2 \eta \cdot  F_{L_t}^* $ and $L_2 = L_t$ in Lemma \ref{lemma:proj_prop}  gives
 \$ 
\tr \Bigl [ \Sigma_{L_t}^*  \big(\tilde{G}_{L_t}^{*^\top} F_{L_t}^*+ F_{L_t}^{*^\top}\tilde{G}_{L_t}^* \big) \Bigr ] \geq 2  \tr \Bigl ( \Sigma_{L_t}^*  \tilde{G}_{L_t}^{*^\top} \tilde{G}_{L_t}^{*}   \Bigr ). 
\$ 
Hence, we have 
%\#\label{eq:natural_bound_C1}
\$
& \tr \Bigl [ \Sigma_{L_{t+1} }^*  \big(\tilde{G}_{L_t}^{*^\top} F_{L_t}^*+ F_{L_t}^{*^\top}\tilde{G}_{L_t}^* \big) \Bigr ] \notag \\
& \qquad = \tr \Bigl [ \Sigma_{L_{t} }^*  \big(\tilde{G}_{L_t}^{*^\top} F_{L_t}^*+ F_{L_t}^{*^\top}\tilde{G}_{L_t}^* \big) \Bigr ] + \tr \Bigl [ \bigl( \Sigma_{L_{t+1} }^* -  \Sigma_{L_{t} }^* \bigr )  \big(\tilde{G}_{L_t}^{*^\top} F_{L_t}^*+ F_{L_t}^{*^\top}\tilde{G}_{L_t}^* \big) \Bigr ] \notag \\
% & \qquad \geq 2 \tr \Bigl ( \Sigma_{L_t}^*  \tilde{G}_{L_t}^{*^\top} \tilde{G}_{L_t}^{*}   \Bigr ) - \bigl \|\Sigma_{L_{t+1} }^* -  \Sigma_{L_{t} }^*  \bigr \| \cdot \tr \Bigl [  \big(\tilde{G}_{L_t}^{*^\top} F_{L_t}^*+ F_{L_t}^{*^\top}\tilde{G}_{L_t}^* \big) \Bigr ] \notag\\
 & \qquad \geq {2}{\mu} \big\|\tilde{G}_{L_t}^*\big\|^2_F - 16\eta\big(\|\tilde A_L-BK(L)\|+1\big)\big(\cB_{\Omega}^K\|B\|+\|C\|\big)  \big\|\tilde{G}_{L_t}^*\big\|^2_F\big\|F_{L_t}^*\big\|_F,
%  \#
\$
  where the last inequality uses  Lemma \ref{lemma:perturb_Sigma_L}, which requires that 
 $\|L_{t+1}-L_t\|=2\eta \|\tilde{G}_{L_t}^*\|\leq \cK_{\Omega}^{L}$  (see $\cK_{\Omega}^{L}$ as defined in \eqref{equ:def_K_Omega^L}).  
  This further results in the following bound on the stepsize $\eta$, due to the bound on $\|\tilde{G}_{L_t}^*\|$ from Lemma \ref{lemma:bnd_gradient_mapping_norm}:
  \#\label{equ:stepsz_cond_NG_1}
\eta \leq\frac{\cK_{\Omega}^{L}\zeta\mu}{2\sqrt{q}\cdot\cC(K(L_t),L_t)}\sqrt{\frac{\mu}{\|W_{L_t}\|[\cC(K^*,L^*) - \cC(K(L_t),L_t)]}}.
\#
Moreover, we can have another requirement for $\eta$, similar to  \eqref{equ:stepsz_cond_GN_2}, as 
\small
\#\label{equ:stepsz_cond_NG_2}
\eta\leq \frac{1}{2}\cdot\bigg[ \frac{\sqrt{m}\cdot\cC(K(L_t),L_t)\|R^v\|_F}{\mu^2}  + \frac{8}{\mu}\big(\|\tilde A_{L_t}-BK(L_t)\|+1\big)\big(\cB_{\Omega}^K\|B\|+\|C\|\big) \big\|F_{L_t}^*\big\|_F\bigg]^{-1}.
\# 
\normalsize
Thus, if $\eta$ satisfies  \eqref{equ:stepsz_cond_NG_1} and \eqref{equ:stepsz_cond_NG_2}, we have 
\#\label{equ:cost_diff_bnd_trash_3}
\cC\bigl (K(L_{t+1}), L_{t+1} \bigr ) - \cC\bigl(K(L_{t}), L_{t}\bigr ) \geq{2}{\mu}\eta \big\|\tilde{G}_{L_t}^*\big\|_F^2. 
\#
Summing up both sides of   \eqref{equ:cost_diff_bnd_trash_3} from $0$ to $t\geq 1$ yields
\$
 {\frac{1}{t}\sum_{\tau=0}^{t-1}\big\|\tilde{G}_{L_\tau}^*\big\|_F^2}
%&\leq \sqrt{\frac{\cC\big(K_0,L(K_0)\big)-\cC\big(K_T,L(K_T)\big)}{2T\eta\mu \sigma_{\min}(W_{K_0}^{-1})}}\\
\leq {\frac{\cC\big(K^*,L^*\big)-\cC\big(K(L_0),L_0\big)}{2\mu\eta t}},
\$
 which completes the proof of $\cO(1/{t})$ convergence rate for the  sequence  $\big\{t^{-1}\sum_{\tau=0}^{t-1}\big\|\tilde{G}_{L_\tau}^*\big\|^2\big\}_{t\geq 1}$.

\vspace{10pt}
\noindent{\textbf{Projected Nested-Gradient:}}
\vspace{4pt} 

The projected nested-gradient update \eqref{equ:algorithm_nested_L} can be written as 
\$
L_{t+1} = \PP_{\Omega}^{GD}
\big[L_t + 2\eta \cdot F_{L_t}^*\Sigma_{L_t}^*\big]=
%L_t+2\eta\cdot\frac{\PP_{\Omega}\big[L_t + 2\eta \cdot F_{L_t}^*\big]-L_t}{2\eta}=
L_t+2\eta \cdot \check{G}_{L_t}^*,
\$
where $\PP_{\Omega}^{GD}$ is   defined in \eqref{equ:def_proj_GD}  and   $\check{G}_{L}^*$ is defined in \eqref{equ:def_check_G}. By the property of $\PP_{\Omega}^{GD}$ and Lemma \ref{lemma:proj_prop}, we have 
 \$ 
\tr   \big(\check{G}_{L_t}^{*^\top} F_{L_t}^*\Sigma_{L_t}^*\big)=\tr   \big(\Sigma_{L_t}^*\check{G}_{L_t}^{*^\top} F_{L_t}^*\big) \geq   \tr \Bigl (  \check{G}_{L_t}^{*^\top} \check{G}_{L_t}^{*}   \Bigr ),
\$
which implies that
\$
& \tr \Bigl [ \Sigma_{L_{t+1} }^*  \big(\check{G}_{L_t}^{*^\top} F_{L_t}^*+ F_{L_t}^{*^\top}\check{G}_{L_t}^* \big) \Bigr ] \notag \\
& \qquad = \tr \Bigl [ \Sigma_{L_{t} }^*  \big(\check{G}_{L_t}^{*^\top} F_{L_t}^*+ F_{L_t}^{*^\top}\check{G}_{L_t}^* \big) \Bigr ] + \tr \Bigl [ \bigl( \Sigma_{L_{t+1} }^* -  \Sigma_{L_{t} }^* \bigr )  \big(\check{G}_{L_t}^{*^\top} F_{L_t}^*+ F_{L_t}^{*^\top}\check{G}_{L_t}^* \big) \Bigr ] \notag \\
% & \qquad \geq 2 \tr \Bigl ( \Sigma_{L_t}^*  \tilde{G}_{L_t}^{*^\top} \tilde{G}_{L_t}^{*}   \Bigr ) - \bigl \|\Sigma_{L_{t+1} }^* -  \Sigma_{L_{t} }^*  \bigr \| \cdot \tr \Bigl [  \big(\tilde{G}_{L_t}^{*^\top} F_{L_t}^*+ F_{L_t}^{*^\top}\tilde{G}_{L_t}^* \big) \Bigr ] \notag\\
 & \qquad \geq {2} \big\|\check{G}_{L_t}^*\big\|^2_F - 16\eta\big(\|\check A_L-BK(L)\|+1\big)\big(\cB_{\Omega}^K\|B\|+\|C\|\big)  \big\|\check{G}_{L_t}^*\big\|^2_F\big\|F_{L_t}^*\big\|_F,
  \$
  if, by Lemma \ref{lemma:perturb_Sigma_L}, $\|L_{t+1}-L_t\|=2\eta \|\check{G}_{L_t}^*\|\leq \cK_{\Omega}^{L}$ holds. By the bound on $\|\check{G}_{L_t}^*\|$ from Lemma \ref{lemma:bnd_gradient_mapping_norm}, we further require
  \#\label{equ:stepsz_cond_GD_1}
\eta \leq\frac{\cK_{\Omega}^{L}\zeta}{2\sqrt{q}\cdot\cC(K(L_t),L_t)}\sqrt{\frac{\mu}{\|W_{L_t}\|[\cC(K^*,L^*) - \cC(K(L_t),L_t)]}}.
\#
Also, similar to \eqref{equ:stepsz_cond_GN_2}, we also require
\small
\#\label{equ:stepsz_cond_GD_2}
\eta\leq \frac{1}{2}\cdot\bigg[ \frac{\sqrt{m}\cdot\cC(K(L_t),L_t)\|R^v\|_F}{\mu}  + {8}\big(\|\tilde A_{L_t}-BK(L_t)\|+1\big)\big(\cB_{\Omega}^K\|B\|+\|C\|\big) \big\|F_{L_t}^*\big\|_F\bigg]^{-1}.
\#
\normalsize
Thus, if $\eta$ satisfies  \eqref{equ:stepsz_cond_GD_1} and \eqref{equ:stepsz_cond_GD_2}, we have
\#\label{equ:cost_diff_bnd_trash_4}
\cC\bigl (K(L_{t+1}), L_{t+1} \bigr ) - \cC\bigl(K(L_{t}), L_{t}\bigr ) \geq{2}\eta \big\|\check{G}_{L_t}^*\big\|_F^2. 
\#
Summing up both sides of   \eqref{equ:cost_diff_bnd_trash_4} from $0$ to $t\geq 1$ yields the desired $\cO(1/{t})$ convergence rate for the sequence $\big\{t^{-1}\sum_{\tau=0}^{t-1}\big\|\check{G}_{L_\tau}^*\big\|^2\big\}_{t\geq 1}$, which thus completes the proof.  
\hfil\qed

\subsection{Proof of Theorem \ref{thm:full_grad_conv_local}}\label{subsec:proof_full_grad_conv_local}
Now we   analyze the \emph{locally linear}  convergence rates of the proposed  algorithms.

\vspace{10pt}
\noindent{\textbf{Projected Gauss-Newton Nested-Gradient:}}
\vspace{4pt}

First, by Assumption \ref{assum:invertibility} and  the definition of $\Omega$ in \eqref{equ:def_Omega},   $L^*$ is an  interior point of $\Omega$. 
Letting $L'=L^*$ and $L=L_t$ in  the upper bound of  Lemma \ref{lemma:Cost_Diff3}, we have
\small
\#\label{equ:local_rate_trash_1}
&\cC(K^*, L^* ) - \cC  \bigl (K(L_t),  L_t\bigr )  
\leq 2   \tr \bigl [ \Sigma_{\tilde K_t, L^*} (L^*-L_t)^\top F_{L_t}^* \bigr ] - \tr \bigl [ \Sigma_{\tilde K_t, L^*}  (L^* - L_t ) ^\top W_{L_t} (L^*- L_t)  \bigr ] \notag \\
& \quad
\leq  \tr \Bigl ( \Sigma_{\tilde K_t, L^*}  F_{L_t}^{*^\top} W_{L_t}^{-1}   F_{L_t}^*\Bigr )\leq \big\|\Sigma_{\tilde K_t, L^*}  \big\|\cdot\tr \Bigl ( F_{L_t}^{*^\top} W_{L_t}^{-1}   F_{L_t}^*\Bigr ),
\#
where $\tilde K_t$ is defined as follows  
\$
\tilde K_t= K(L_t)  -  (R^u+B^\top P_{L_t}^* B)^{-1} B^\top P_{L_t}^*C(L^*-L_t)=(R^u+B^\top P_{L_t}^* B)^{-1} B^\top P_{L_t}^*(A-CL^*),
\$
and the second inequality follows by completing squares. 
%Different from the proof for LQR in \cite{fazel2018global}, 
Note that the correlation matrix $\Sigma_{\tilde K_t, L^*}$ may be unbounded, since the control pair $(\tilde K_t,L^*)$, where $\tilde K_t$ is generated by $L_t$, may not be  stabilizing, unless $L_t$ is close to $L^*$, since we know by Assumption \ref{assum:invertibility} that $(K^*,L^*)$ is stabilizing.  
In fact, by the continuity of $P_L^*$ w.r.t. $L$ from  Lemma \ref{lemma:conti_PLs}, and the continuity of $\rho(A-B K-CL^*)$ w.r.t. $K$ \citep{tyrtyshnikov2012brief},  
there exists a ball centered at $L^*$ with radius $\omega_1>0$, denoted by $\cB(L^*,\omega_1)$, such that   $\cB(L^*,\omega_1)\subseteq \Omega$,  and for any $L_t\in \cB(L^*,\omega_1)$, $\rho(A-B\tilde K_t-CL^*)<1$, i.e., $(\tilde K_t,L^*)$ is stabilizing.  Thus by Lemma \ref{lemma:bound_P_L_Sigma_L}, \eqref{equ:local_rate_trash_1} can be bounded as
\footnotesize
\#\label{equ:local_rate_trash_2}
&\cC(K^*, L^* ) - \cC  \bigl (K(L_t),  L_t\bigr )  
\leq \frac{\cC(\tilde K_t,  L^*)}{\zeta}\cdot\tr \Bigl ( F_{L_t}^{*^\top} W_{L_t}^{-1}   F_{L_t}^*\Bigr )\leq \frac{\cC(K^*,  L^*)+\vartheta}{\zeta}\cdot\tr \Bigl ( F_{L_t}^{*^\top} W_{L_t}^{-1}   F_{L_t}^*\Bigr ),
\#
\normalsize
for some constant $\vartheta\geq 0$, where the last inequality is due to the continuity of $P_{K,L}$, and thus $\cC(K,L)=\tr(\Sigma_0 P_{K,L})$  where $\Sigma_0=\EE x_0x_0^\top$,  w.r.t. $K$, for given $L$,  from Lemma \ref{lemma:conti_Sigma}.

On the other hand,  
due to the continuity of $P_L^*$ from Lemma \ref{lemma:conti_PLs},   $\cC(K(L),L)=\tr(\Sigma_0 P_L^*)$  is continuous w.r.t. $L$ for any $L\in\Omega$.   Let $\bar{\cC}_{\Omega}=\sup_{L\in\partial \Omega}\cC(K(L),L)$, where $\partial \Omega$ denotes the boundary of the set $\Omega$. Then by continuity and the uniqueness of the maximizer $L^*$, there exists some $L_t\in\cB(L^*,\omega_1)$ around $L^*$ such that $\bar{\cC}_{\Omega}<\cC(K(L_t),L_t)<\cC(K^*,L^*)$, and the upper-level set $\cA_{{\Omega}}^{L_t}:=\{L\given \cC(K(L),L)\geq \cC(K(L_t),L_t)\}$ lies in $\cB(L^*,\omega_1)$ (thus also lies in $\Omega$).  Since $\cC(K^*,L^*)$ is the upper bound of $\cC(K(L),L)$, the upper-level set $\cA_{{\Omega}}^{L_t}$ is compact. Also, letting $\Omega^c:=\RR^{m_2\times d}/\{\Omega/\partial \Omega\}$, then we know that $\Omega^c=\{L\given \lambda_{\max}(L^\top R^v L-Q+\zeta\cdot \Ib)\geq 0\}$, which is closed since $\lambda_{\max}(\cdot)$ is a continuous function. Thus, by Lemma \ref{lemma:dist_compact}, there exists a distance $\omega_2>0$ between the disjoint sets $\cA_{{\Omega}}^{L_t}$ and $\Omega^c$. Thus,   for any $L_{t+1}$ such that $\|L_{t+1}-L_t\|\leq \omega_2$, $L_{t+1}$ belongs to $\Omega$, namely, the projection is ineffective, i.e.,  $\PP^{GN}(L_{t+1})=L_{t+1}$. Letting   
$
L_{t+1}=L_t+2\eta W_{L_t}^{-1}F_{L_t}^*
$.
In addition,  we have 
\$
\big\|F_L^*\big\|\leq \big\|F_L^*\Sigma_L^*\big\|\big\| \Sigma_L^{*^{-1}}\big\|\leq \frac{\sqrt{q}}{2}\cdot\big\|\nabla_L \tilde{\cC}(L)\big\|\cdot\frac{1}{\sigma_{\min}\big(\Sigma_L^{*}\big)}\leq  \frac{\sqrt{q}}{2\mu}\cdot\big\|\nabla_L \tilde{\cC}(L)\big\|,
%\cdot\frac{1}{\mu}
\$
where the second inequality follows from \eqref{equ:bnd_grad_norm_trash_6} in the proof of Lemma \ref{lemma:bnd_gradient_mapping_norm}, and the fact that $\| \Sigma_L^{*^{-1}}\|=\sigma_{\min}^{-1}\big(\Sigma_L^{*}\big)$.
By Lemma \ref{lemma:bnd_gradient_mapping_norm}, we further have
\#\label{equ:eta_bnd_local_GN_trash_1}
\big\|F_L^*\big\|\leq  \frac{\sqrt{q}}{2\mu}\cdot\big\|\nabla_L \tilde{\cC}(L)\big\|\leq \frac{\sqrt{q}\cC(K(L),L)}{\mu\zeta}\sqrt{\frac{\|W_L\|[\cC(K^*,L^*) - \cC(K(L),L)]}{\mu}}.
\#
Also, notice that  
\#\label{equ:eta_bnd_local_GN_trash_2}
\big\|W_{L}^{-1}F_L^*\big\|\leq\big\|W_{L}^{-1}\big\|\big\|F_L^*\big\|=\frac{\big\|F_L^*\big\|}{\sigma_{\min}(W_L)}\leq \frac{\big\|F_L^*\big\|}{\nu}. 
\#
Thus,  by \eqref{equ:eta_bnd_local_GN_trash_1} and \eqref{equ:eta_bnd_local_GN_trash_2}, to ensure $\|L_{t+1}-L_t\|\leq \omega_2$ we require
\$
\eta\leq \frac{\omega_2\mu\nu\zeta}{2\sqrt{q}\cC(K(L),L)}\sqrt{\frac{\mu}{\|W_L\|[\cC(K^*,L^*) - \cC(K(L),L)]}},
\$
which can be satisfied by the following sufficient condition
\#\label{equ:eta_bnd_local_GN_1}
\eta\leq \frac{\omega_2\mu\nu\zeta}{2\sqrt{q}\cC(K^*,L^*)}\sqrt{\frac{\mu}{\|R^v\|\cC(K^*,L^*)}},
\#
where   $\|W_{L}\|\leq \|R^v\|$ by  Lemma \ref{lemma:bound_P_L_Sigma_L}. 
Note that the bound in \eqref{equ:eta_bnd_local_GN_1} is independent of $L$.

In sum, as long as $\eta$ satisfies   \eqref{equ:eta_bnd_local_GN_1}, we know that $
L_{t+1}=L_t+2\eta W_{L}^{-1}F_L^*
$ still lies in $\Omega$. 
Hence, by the lower bound in Lemma \ref{lemma:Cost_Diff3}, we have 
\#\label{eq:recall_gauss_newton_2_bound} 
\cC\bigl (K(L_{t+1}), L_{t+1} \bigr ) - \cC\bigl(K(L_{t}), L_{t}\bigr ) &\geq 4\eta  \tr\Big(\Sigma_{L_{t+1}}^*F_{L_t}^{*^\top}W_{L_t}^{-1}F_{L_t}^*\Big)-4\eta^2\tr\Big(\Sigma_{L_{t+1}}^*F_{L_t}^{*^\top}W_{L_t}^{-1}F_{L_t}^*\Big)\notag\\
&\geq 2\eta  \mu \cdot \tr \bigl (  F_{L_t}^{*^\top} W_{L_t}^{-1}F_{L_t}^* \bigr ),
%\notag\\
%& \quad \leq  -2\eta \cdot \frac{\sigma_{\min} (\Sigma_{K_{t+1}}^* ) }{ \|\Sigma_{K^* , \tilde L_t}  \|} \cdot \Bigr [\cC  \bigl (K_t ,  L( K_t ) \bigr )- \cC(K^*, L^* )  \Bigr ] ,  
\#
provided that the stepsize $\eta\leq 1/2$.

Combining    \eqref{equ:local_rate_trash_2} and \eqref{eq:recall_gauss_newton_2_bound} yields
\$
\cC\bigl (K(L_{t+1}), L_{t+1} \bigr ) - \cC\bigl(K(L_{t}), L_{t}\bigr ) \geq \frac{2\eta  \mu\zeta}{\cC(K^*,  L^*)+\vartheta} \cdot\big[\cC(K^*, L^* ) - \cC  \bigl (K(L_t),  L_t\bigr )\big],
\$
which further leads to 
\small
\#\label{equ:GN_local_rate_final_C}
\cC(K^*, L^* ) - \cC\bigl(K(L_{t+1}), L_{t+1}\bigr ) \leq \bigg(1-\frac{2\eta  \mu\zeta}{\cC(K^*,  L^*)+\vartheta}\bigg) \cdot \big[\cC(K^*, L^* ) - \cC  \bigl (K(L_t),  L_t\bigr )\big].
\#
\normalsize
That is, the sequence $\{\cC  \bigl (K(L_t),  L_t\bigr )\}_{t\geq 0}$ converges to $\cC(K^*, L^* )$ with linear rate, provided that
\$
\eta\leq \min\Bigg\{\frac{1}{2},\frac{\omega_2\mu\nu\zeta}{2\sqrt{q}\cC(K^*,L^*)}\sqrt{\frac{\mu}{\|R^v\|\cC(K^*,L^*)}},\frac{\cC(K^*,  L^*)+\vartheta}{4\mu\zeta}\Bigg\}.
\$
In addition,  by Lemma \ref{lemma:bnd_gradient_mapping_norm}
\$
\big\|\nabla_L \tilde{\cC}(L_t)\big\|^2 \leq \frac{4\cC(K^*,L^*)^2\|R^v\|}{\mu\zeta^2}\cdot{{[\cC(K^*,L^*) - \cC(K(L_t),L_t)]}},
\$
where we use that $\cC(K(L_t),L_t)\leq \cC(K^*,L^*)$ and $W_{L_t}\leq R^v$. 
Thus,  \eqref{equ:GN_local_rate_final_C} also implies the locally linear convergence rate of $\{\|\nabla_L \tilde{\cC}(L_t)\|^2\}_{t\geq 0}$, which 
 completes the proof.

%From \eqref{eq:recall_gauss_newton_2_bound} we know that
%\$
%\tr \bigl [  (F_{L_t}^*)^\top F_{L_t}^* \bigr ]&\leq \frac{\sigma_{\max}(W_{L_0})}{2\eta  \mu}\cdot \big\{[C(K^*,L^*)-\cC\bigl (K(L_{t}), L_{t} \bigr )]-[C(K^*,L^*)-\cC\bigl (K(L_{t+1}), L_{t+1} \bigr )]\big\}\\
%&\leq \frac{\sigma_{\max}(W_{L_0})}{2\eta  \mu}\cdot [C(K^*,L^*)-\cC\bigl (K(L_{t}), L_{t} \bigr )],
%\$
%where we use the fact that $W^{-1}_{L_t}\geq W^{-1}_{L_0}$ and $\sigma_{\min}(W^{-1}_{L_0})=1/\sigma_{\max}(W_{L_0})$.
%This implies that the sequence $\{\|F_{L_t}^*\|_F\}$ also converges to zero with linear rate. 
%
% 
%On the other hand, by the proof of Theorem \ref{thm:full_grad_conv},  the sequence $\{\|F_{L_t}^*\|_F\}$ converges to zero, namely, $\{L_t\}$      converges to the ball $\cB(L^*,\delta)$, with sublinear rate, 
%which completes the proof. 

\vspace{10pt}
\noindent{\textbf{Projected  Natural  Nested-Gradient:}}
\vspace{4pt} 

The proof for projected natural nested-gradient is similar to the one above. \eqref{equ:local_rate_trash_2} and \eqref{equ:eta_bnd_local_GN_trash_1} still hold. Now since the update becomes $L_{t+1}=L_t+2\eta F_{L_t}^*$, to ensure $\|L_{t+1}-L_t\|\leq \omega_2$ we require
\$
\eta\leq \frac{\omega_2\mu\zeta}{2\sqrt{q}\cC(K(L),L)}\sqrt{\frac{\mu}{\|W_L\|[\cC(K^*,L^*) - \cC(K(L),L)]}},
\$
which can be satisfied by
\#\label{equ:eta_bnd_local_NG_1}
\eta\leq \frac{\omega_2\mu\zeta}{2\sqrt{q}\cC(K^*,L^*)}\sqrt{\frac{\mu}{\|R^v\|\cC(K^*,L^*)}}. 
\#
Then,  by the lower bound in Lemma \ref{lemma:Cost_Diff3}, it follows that
\small
\#\label{eq:recall_natural_gradient_2_bound} 
&\cC\bigl (K(L_{t+1}), L_{t+1} \bigr ) - \cC\bigl(K(L_{t}), L_{t}\bigr ) \geq 4\eta  \tr\Big(\Sigma_{L_{t+1}}^*F_{L_t}^{*^\top}F_{L_t}^*\Big)-4\eta^2\tr\Big(\Sigma_{L_{t+1}}^*F_{L_t}^{*^\top}W_{L_t}F_{L_t}^*\Big)\notag\\
&\quad\geq 4\eta  \tr\Big(\Sigma_{L_{t+1}}^*F_{L_t}^{*^\top}F_{L_t}^*\Big)-4\eta^2\|R^v\|\tr\Big(\Sigma_{L_{t+1}}^*F_{L_t}^{*^\top}F_{L_t}^*\Big) \geq 2\eta  \mu \cdot \tr \bigl (  F_{L_t}^{*^\top} F_{L_t}^* \bigr ),
%\notag\\
%& \quad \leq  -2\eta \cdot \frac{\sigma_{\min} (\Sigma_{K_{t+1}}^* ) }{ \|\Sigma_{K^* , \tilde L_t}  \|} \cdot \Bigr [\cC  \bigl (K_t ,  L( K_t ) \bigr )- \cC(K^*, L^* )  \Bigr ] ,  
\#
\normalsize
where the second inequality is due to $\|W_{L_t}\|\leq \|R^v\|$ from Lemma \ref{lemma:bound_P_L_Sigma_L}, and the last inequality holds if 
$
\eta\leq 1/(2\|R^v\|)
$. 
Note that \eqref{equ:local_rate_trash_2} further gives
\small 
\#\label{equ:local_rate_trash_22}
\cC(K^*, L^* ) - \cC  \bigl (K(L_t),  L_t\bigr )\leq \frac{\cC(K^*,  L^*)+\vartheta}{\zeta\sigma_{\min}(W_{L_t})}\tr \Bigl ( F_{L_t}^{*^\top}   F_{L_t}^*\Bigr )\leq \frac{\cC(K^*,  L^*)+\vartheta}{\zeta\nu}\tr \Bigl ( F_{L_t}^{*^\top}   F_{L_t}^*\Bigr ),
\#
\normalsize
which combined with   \eqref{eq:recall_natural_gradient_2_bound} yields
\$
\cC\bigl (K(L_{t+1}), L_{t+1} \bigr ) - \cC\bigl(K(L_{t}), L_{t}\bigr ) \geq \frac{2\eta  \mu\zeta\nu}{\cC(K^*,  L^*)+\vartheta} \cdot\big[\cC(K^*, L^* ) - \cC  \bigl (K(L_t),  L_t\bigr )\big].
\$
Therefore, the linear convergence rate follows as
\small
\#\label{equ:NG_local_rate_final_C}
\cC\bigl (K^*, L^*  \bigr ) - \cC\bigl(K(L_{t+1}), L_{t+1}\bigr ) \leq \bigg(1-\frac{2\eta  \mu\zeta\nu}{\cC(K^*,  L^*)+\vartheta} \bigg)\cdot\big[\cC(K^*, L^* ) - \cC  \bigl (K(L_t),  L_t\bigr )\big],
\#
\normalsize
provided that the stepsize $\eta$ satisfies 
\$
\eta\leq \min\Bigg\{\frac{1}{2\|R^v\|},\frac{\omega_2\mu\zeta}{2\sqrt{q}\cC(K^*,L^*)}\sqrt{\frac{\mu}{\|R^v\|\cC(K^*,L^*)}},\frac{\cC(K^*,  L^*)+\vartheta}{4\mu\zeta\nu}\Bigg\}.
\$
Note that  \eqref{equ:NG_local_rate_final_C} also implies the locally linear  rate of $\{\|\nabla_L \tilde{\cC}(L_t)\|^2\}_{t\geq 0}$, completing the proof.

\vspace{10pt}
\noindent{\textbf{Projected   Nested-Gradient:}}
\vspace{4pt}   

By \eqref{equ:bnd_grad_norm_trash_5} and Lemma \ref{lemma:bnd_gradient_mapping_norm}, we have
\#\label{equ:GD_trash_1}
&\big\|F_L^*\Sigma_L^*\big\|\leq \big\|F_L^*\Sigma_L^*\big\|_F\leq \frac{\sqrt{q}}{2}\cdot\big\|\nabla_L \tilde{\cC}(L)\big\|\notag\\
&\quad\leq \frac{\sqrt{q}\cC(K(L),L)}{\zeta}\sqrt{\frac{\|W_L\|[\cC(K^*,L^*) - \cC(K(L),L)]}{\mu}}. 
\#
Since the update becomes $L_{t+1}=L_t+2\eta F_{L_t}^*\Sigma_{L_t}^*$, to ensure $\|L_{t+1}-L_t\|\leq \omega_2$, we require
%\$
%\eta\leq \frac{\omega_2\zeta}{2\sqrt{q}\cC(K(L),L)}\sqrt{\frac{\mu}{\|W_L\|[\cC(K^*,L^*) - \cC(K(L),L)]}},
%\$
%which can be satisfied by
\#\label{equ:eta_bnd_local_GD_1}
\eta\leq \frac{\omega_2\zeta}{2\sqrt{q}\cC(K^*,L^*)}\sqrt{\frac{\mu}{\|R^v\|\cC(K^*,L^*)}}. 
\#
Then, applying Lemma \ref{lemma:Cost_Diff3} we have
\small
\#\label{eq:recall_gradient_descent_2_bound} 
&\cC\bigl (K(L_{t+1}), L_{t+1} \bigr ) - \cC\bigl(K(L_{t}), L_{t}\bigr ) \geq 4\eta  \tr\Big(\Sigma_{L_{t+1}}^*\Sigma_{L_t}^*F_{L_t}^{*^\top}F_{L_t}^*\Big)-4\eta^2\tr\Big(\Sigma_{L_{t+1}}^*\Sigma_{L_t}^*F_{L_t}^{*^\top}W_{L_t}F_{L_t}^*\Sigma_{L_t}^*\Big)\notag\\
&\quad\geq \big(4\eta- 4\eta^2\|R^v\|\|\Sigma_{L_{t+1}}^*\|\big) \tr\Big(\Sigma_{L_{t}}^*\Sigma_{L_{t}}^*F_{L_t}^{*^\top}F_{L_t}^*\Big)-4\eta\big\|\Sigma_{L_{t+1}}^*-\Sigma_{L_{t}}^*\big\| \tr\Big(\Sigma_{L_{t}}^*F_{L_t}^{*^\top}F_{L_t}^*\Big)\notag\\
&\quad\geq \big(4\eta- 4\eta^2\|R^v\|\|\Sigma_{L_{t+1}}^*\|\big) \tr\Big(\Sigma_{L_{t}}^*\Sigma_{L_{t}}^*F_{L_t}^{*^\top}F_{L_t}^*\Big)-4\eta\frac{\big\|\Sigma_{L_{t+1}}^*-\Sigma_{L_{t}}^*\big\|}{\mu} \tr\Big(\Sigma_{L_{t}}^*F_{L_t}^{*^\top}F_{L_t}^*\Sigma_{L_{t}}^*\Big)\notag\\
&\quad= 4\eta\Bigg(1- \eta\|R^v\|\|\Sigma_{L_{t+1}}^*\|-\frac{\big\|\Sigma_{L_{t+1}}^*-\Sigma_{L_{t}}^*\big\|}{\mu}\Bigg) \big\|F_L^*\Sigma_L^*\big\|_F^2.
%\tr\Big(\Sigma_{L_{t}}^*\Sigma_{L_{t}}^*F_{L_t}^{*^\top}F_{L_t}^*\Big),
%\notag\\
%& \quad \leq  -2\eta \cdot \frac{\sigma_{\min} (\Sigma_{K_{t+1}}^* ) }{ \|\Sigma_{K^* , \tilde L_t}  \|} \cdot \Bigr [\cC  \bigl (K_t ,  L( K_t ) \bigr )- \cC(K^*, L^* )  \Bigr ] ,  
\#
\normalsize

By recalling Lemma \ref{lemma:perturb_Sigma_L} and the definition of $\cK_{\Omega}^L$ in \eqref{equ:def_K_Omega^L}, 
if $\eta$  makes  $\|L_{t+1}-L_t\|=2\eta\|F_{L_t}^*\Sigma_{L_t}^*\|\leq \cK_{\Omega}^L$, 
i.e., 
\#\label{equ:GD_trash_2.5}
\eta\leq \frac{\cK_{\Omega}^L\zeta}{2\sqrt{q}\cC(K^*,L^*)}\sqrt{\frac{\mu}{\|R^v\|\cC(K^*,L^*)}},
\#
then it follows  that
\$
&\frac{\big\|\Sigma_{L_{t+1}}^*-\Sigma_{L_{t}}^*\big\|}{\mu}\leq \frac{4\eta}{\mu}\big(\|\tilde A_{L_t}-BK(L_t)\|+1\big)\big(\cB_{\Omega}^K\|B\|+\|C\|\big)\cdot \|L_{t+1}-L_t\|\\
&\quad \leq \frac{4\eta\cK_{\Omega}^L}{\mu}\big(\|\tilde A_{L_t}-BK(L_t)\|+1\big)\big(\cB_{\Omega}^K\|B\|+\|C\|\big). 
%\leq  \frac{8\eta\cK_{\Omega}^L}{\mu}\big(\|\tilde A_L-BK(L)\|+1\big)\big(\cB_{\Omega}^K\|B\|+\|C\|\big)
\$
If we further require 
\#\label{equ:GD_trash_2}
\eta\leq \frac{\mu}{16\eta \cK_{\Omega}^L\big(\|\tilde A_{L_t}-BK(L_t)\|+1\big)\big(\cB_{\Omega}^K\|B\|+\|C\|\big)},
\#
then ${\big\|\Sigma_{L_{t+1}}^*-\Sigma_{L_{t}}^*\big\|}/{\mu}\leq 1/4$, which also implies that
\$
\big\|\Sigma_{L_{t+1}}^*\big\|\leq \big\|\Sigma_{L_{t}}^*\big\|+\big\|\Sigma_{L_{t+1}}^*-\Sigma_{L_{t}}^*\big\|\leq \frac{\cC(K(L_t),L_t)}{\zeta}+\frac{\mu}{4}\leq \frac{\cC(K(L_t),L_t)}{\zeta}+\frac{\big\|\Sigma_{L_{t+1}}^*\big\|}{4}. 
\$
Thus, we can bound $\big\|\Sigma_{L_{t+1}}^*\big\|\leq 4\cC(K(L_t),L_t)/(3\zeta)$.
Then if 
$\eta$ further satisfies
\#\label{equ:GD_trash_3}
\eta\leq \frac{3\zeta}{16\cC\bigl(K(L_{t}), L_{t}\bigr )\|R^v\|},
\#
we have 
$
1- \eta\|R^v\| \cdot \|\Sigma_{L_{t+1}}^*\|- \|\Sigma_{L_{t+1}}^*-\Sigma_{L_{t}}^*\| /\mu\geq 1-1/4-1/4= 1/ 2,
$
which establishes the bound in  \eqref{eq:recall_gradient_descent_2_bound} as
\#\label{equ:GD_trash_4}
\cC\bigl (K(L_{t+1}), L_{t+1} \bigr ) - \cC\bigl(K(L_{t}), L_{t}\bigr ) \geq 2\eta \big\|F_L^*\Sigma_L^*\big\|_F^2. 
\#

On the other hand, by \eqref{equ:local_rate_trash_22}, we also have 
\#\label{equ:GD_trash_5}
\cC(K^*, L^* ) - \cC  \bigl (K(L_t),  L_t\bigr )\leq \frac{\cC(K^*,  L^*)+\vartheta}{\zeta\nu\mu^2}\tr \Bigl ( \Sigma_{L_t}^*F_{L_t}^{*^\top}   F_{L_t}^*\Sigma_{L_t}^*\Bigr ). 
\#
Combining \eqref{equ:GD_trash_4} and \eqref{equ:GD_trash_5} yields 
\$
\cC\bigl (K^*, L^*  \bigr ) - \cC\bigl(K(L_{t+1}), L_{t+1}\bigr ) \leq \bigg(1-\frac{2\eta  \mu^2\zeta\nu}{\cC(K^*,  L^*)+\vartheta} \bigg)\cdot\big[\cC(K^*, L^* ) - \cC  \bigl (K(L_t),  L_t\bigr )\big],
\$
which gives the locally linear convergence rate if
\#\label{equ:GD_trash_6}
\eta\leq \frac{\cC(K^*,  L^*)+\vartheta}{4\mu\zeta\nu}.
\#
In sum, there exists some $\eta$ that satisfies \eqref{equ:eta_bnd_local_GD_1}, \eqref{equ:GD_trash_2.5}, \eqref{equ:GD_trash_2}, \eqref{equ:GD_trash_3}, and \eqref{equ:GD_trash_6}, to guarantee the locally linear convergence rates of both $\{\cC  \bigl (K(L_t),  L_t\bigr )\}_{t\geq 0}$ and $\{\|\nabla_L \tilde{\cC}(L_t)\|^2\}_{t\geq 0}$,  which concludes the proof.
 \qed

%% file: simulations.tex
%!TEX root =Minimax_LQR.tex

\section{Simulation Results}\label{sec:simulations}

In this section, we provide some  numerical results to show the superior convergence property of several PO    methods.    
We consider two settings referred to as  {\bf Case $1$} and {\bf Case $2$}, which are created based on the simulations in 
  \cite{al2007model}, with 
%  $
%  A=\left[\begin{matrix}
%  	0.956488 & 0.0816012 & -0.0005\\
%  	0.0741349 & 0.94121 & -0.000708383
%  	\\
%  	0 & 0 & 0.132655
%  \end{matrix}
%  \right]$, 
%  $ 
%  B=\left[-0.00550808,~-0.096,~0.867345\right]^\top
%  $, 
\$
  &A=\left[\begin{matrix}
  	0.956488 & 0.0816012 & -0.0005\\
  	0.0741349 & 0.94121 & -0.000708383
  	\\
  	0 & 0 & 0.132655
  \end{matrix}
  \right],\quad 
  B=\left[\begin{matrix}
-0.00550808 & -0.096 & 0.867345
\end{matrix}
\right]^\top,
  \$
  and $R^u=R^v=\Ib$, $
  \Sigma_0 =0.03\cdot \Ib. 
  $
  We choose $Q=\Ib$ and  $
C=\left[0.00951892,~0.0038373 ,~0.001\right]^\top$ for {\bf Case $1$}; while $Q=0.01\cdot\Ib$ and $
C=\left[0.00951892,~ 0.0038373,~ 0.2\right]^\top$ for {\bf Case $2$}.  
By direct calculation, we have  that 
\$
 {\textbf{Case $1$:}~} P^*=\left[\begin{matrix}
  	23.7658  & 16.8959   & 0.0937\\
  	16.8959  & 18.4645   & 0.1014\\
  	0.0937   & 0.1014   & 1.0107
  \end{matrix}
  \right],\quad {\textbf{Case $2$:}~} P^*=\left[\begin{matrix}
  	6.0173   & 5.6702 &  -0.0071\\
  	5.6702   & 5.4213 &  -0.0067\\
  	-0.0071  & -0.0067 &   0.0102
  \end{matrix}
  \right].
  \$
 Thus, one can easily check that $R^v-C^\top P^*C>0$ is satisfied for both {\bf Case $1$} and {\bf Case $2$}, i.e., Assumption \ref{assum:invertibility} i) holds.  
 However, for {\bf Case $1$}, $\lambda_{\min}(Q-(L^*)^\top R^v L^*)=0.8739>0$ satisfies Assumption \ref{assum:invertibility} ii); for {\bf Case $2$}, $\lambda_{\min}(Q-(L^*)^\top R^v L^*)=-0.0011<0$ fails to satisfy it.   
% 
%    For {\bf Case $2$}, we choose $Q=0.01\cdot\Ib$ and $
%C=\left[0.00951892,~ 0.0038373,~ 0.2\right]^\top$, 
%  which gives   
%  $
%  P^*=\left[\begin{matrix}
%  	6.0173   & 5.6702 &  -0.0071\\
%  	5.6702   & 5.4213 &  -0.0067\\
%  	-0.0071  & -0.0067 &   0.0102
%  \end{matrix}
%  \right]
%  $. 
%  Thus, it satisfies that $R^v-C^\top P^*C>0$,  but $\lambda_{\min}(Q-L^\top R^v L)=-0.0011>0$  fails to satisfy Assumption \ref{assum:invertibility} ii). 
  
% whose detailed descriptions can be found in \S\ref{sec:append_simu}. In {\bf Case $1$}, both conditions in Assumption \ref{assum:invertibility} hold, i.e., $R^v-C^\top P^*C>0$ and  $Q-(L^*)^\top R^v L^* >0$; while in {\bf Case $2$}, $R^v-C^\top P^*C>0$  holds but   $Q-(L^*)^\top R^v L^*> 0$ does not hold.  
In both settings, 
we evaluate the convergence performance of not only our nested-gradient methods, but also two types of their variants, alternating-gradient (AG) and gradient-descent-ascent (GDA) methods. AG  methods are based on the nested-gradient methods, but at each outer-loop iteration, the  inner-loop gradient-based updates only perform a finite number of iterations, instead of converging to the exact solution $K(L_t)$ as nested-gradient methods, which follows the idea in \citep{nouiehed2019solving}. 
The GDA methods perform policy gradient descent for the minimizer and ascent  for the maximizer simultaneously. 
Detailed updates of these two types of methods   are deferred to \S\ref{sec:append_simu}.

\begin{figure*}[!thbp]
	\centering
	\begin{tabular}{ccc}
		\hskip-6pt\includegraphics[width=0.323\textwidth]{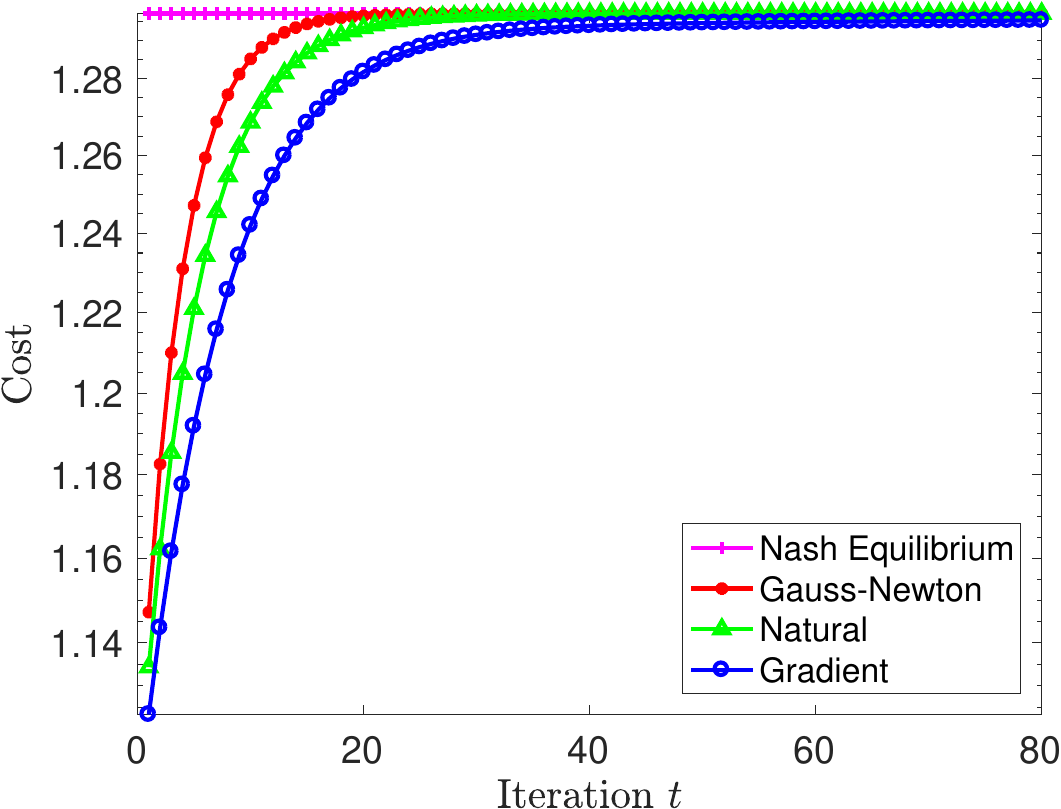}
		&
		\hskip-6pt\includegraphics[width=0.323\textwidth]{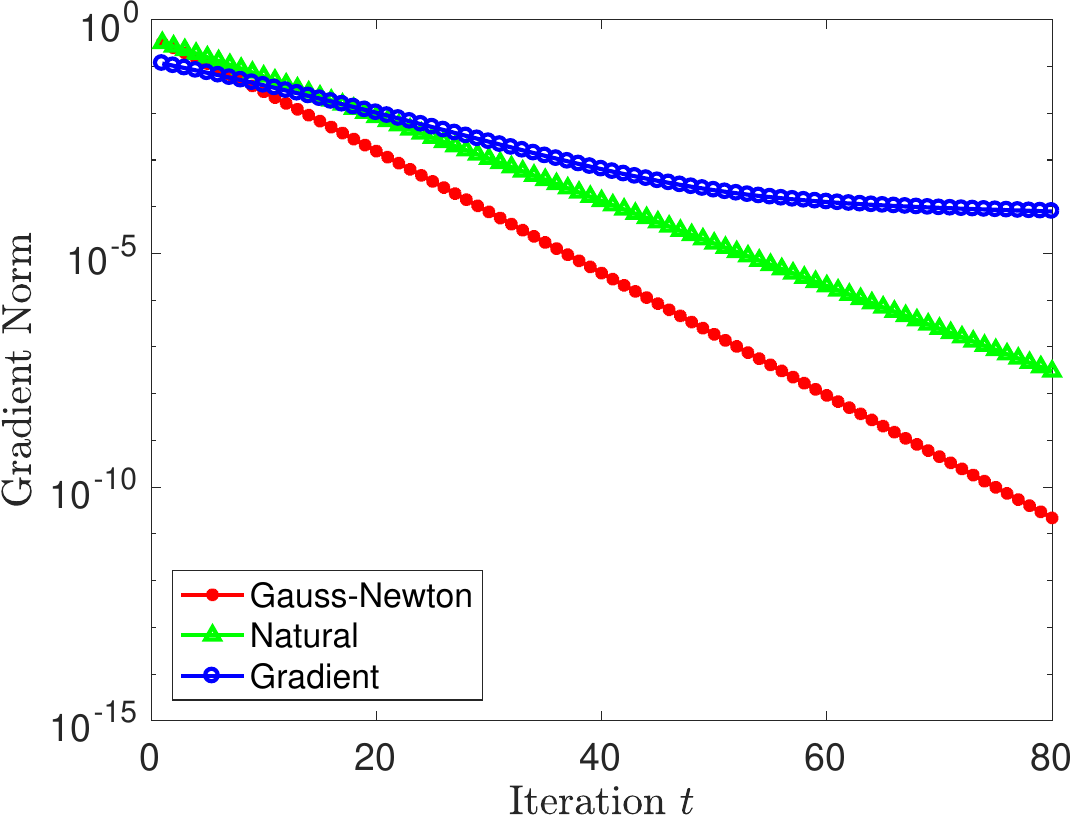}
		& 
		\hskip-6pt\includegraphics[width=0.323\textwidth]{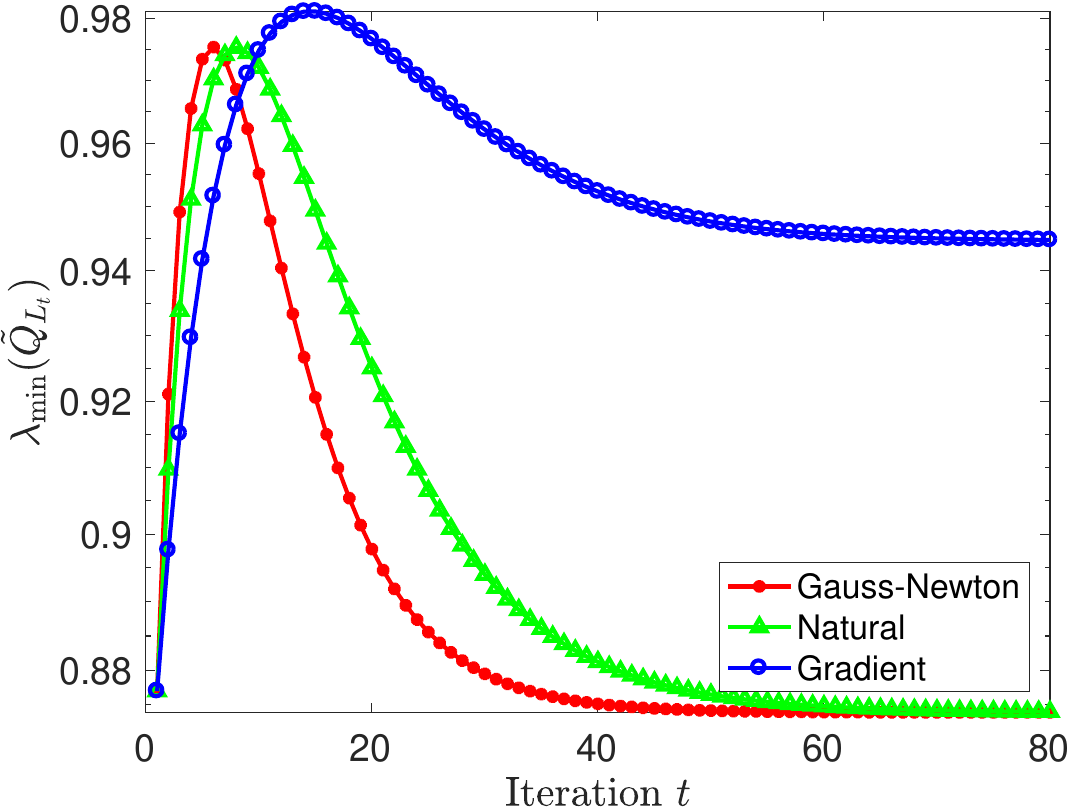}\\
		\hskip-8pt(a) $\cC(K(L),L)$ & \hskip 0pt(b) Grad. Mapp.   Norm Square  & \hskip2pt (c) $\lambda_{\min}(\tilde Q_L)$ 
	\end{tabular}
	\caption{Performance of the three projected NG methods  for {\bf Case $1$} where Assumption \ref{assum:invertibility} ii) is satisfied. (a) shows the monotone convergence of the expected cost $\cC(K(L),L)$ to the NE cost $\cC(K^*,L^*)$; (b) shows the convergence of the gradient mapping norm square; (c) shows the change of the smallest eigenvalue of $\tilde Q_L=Q-L^\top R^vL$.
%	 The algorithms are initialized randomly with stabilizing control pairs, and the stepsizes 
}
	\label{fig:case_1_nested}
\end{figure*}

\begin{figure*}[!thbp]
	\centering
	\begin{tabular}{ccc}
		\hskip-6pt\includegraphics[width=0.323\textwidth]{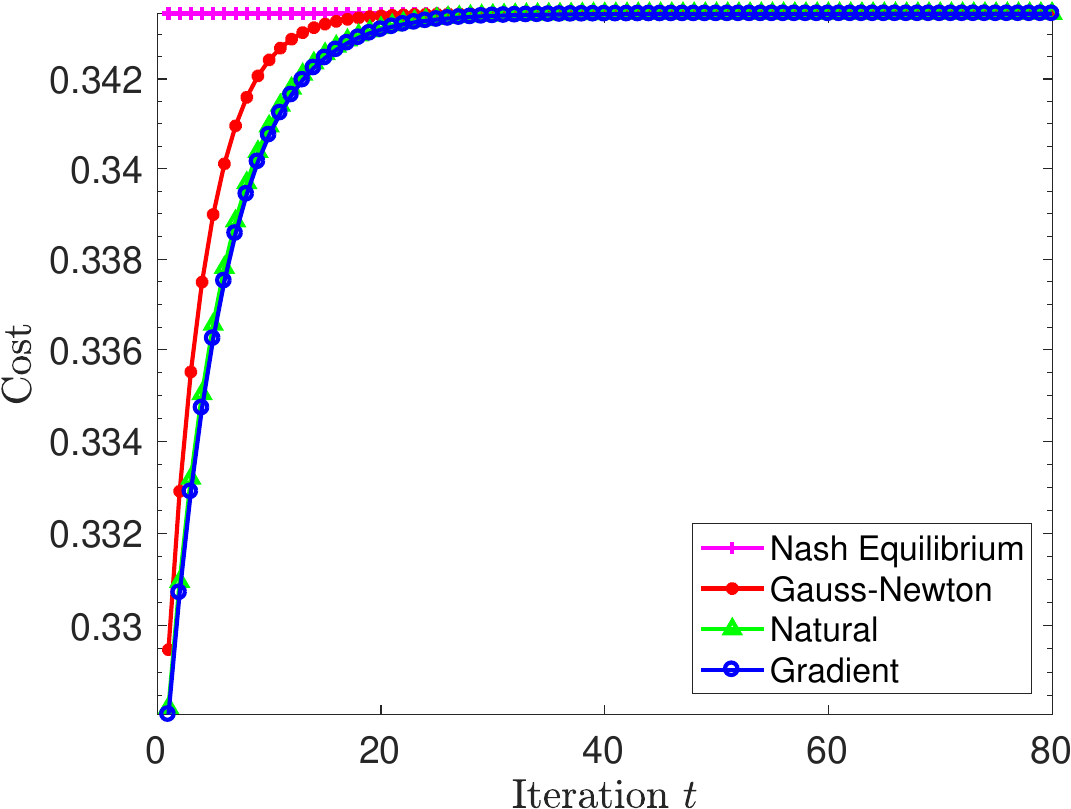}
		&
		\hskip-6pt\includegraphics[width=0.323\textwidth]{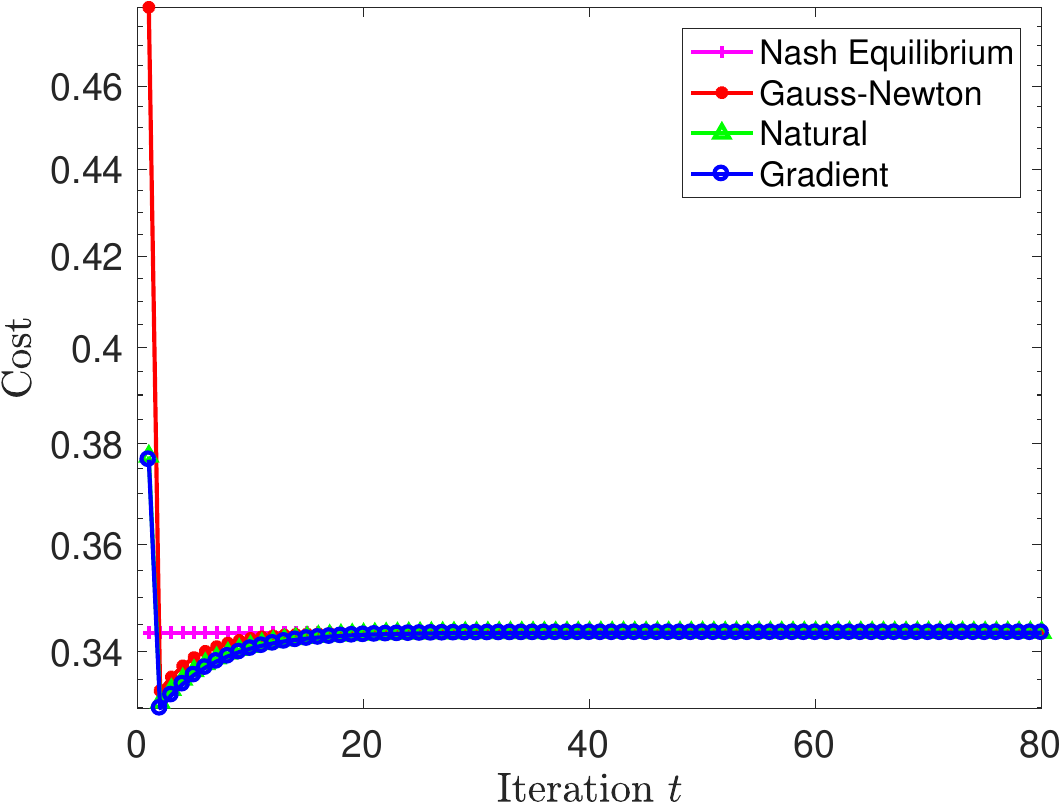}
		& 
		\hskip-6pt\includegraphics[width=0.323\textwidth]{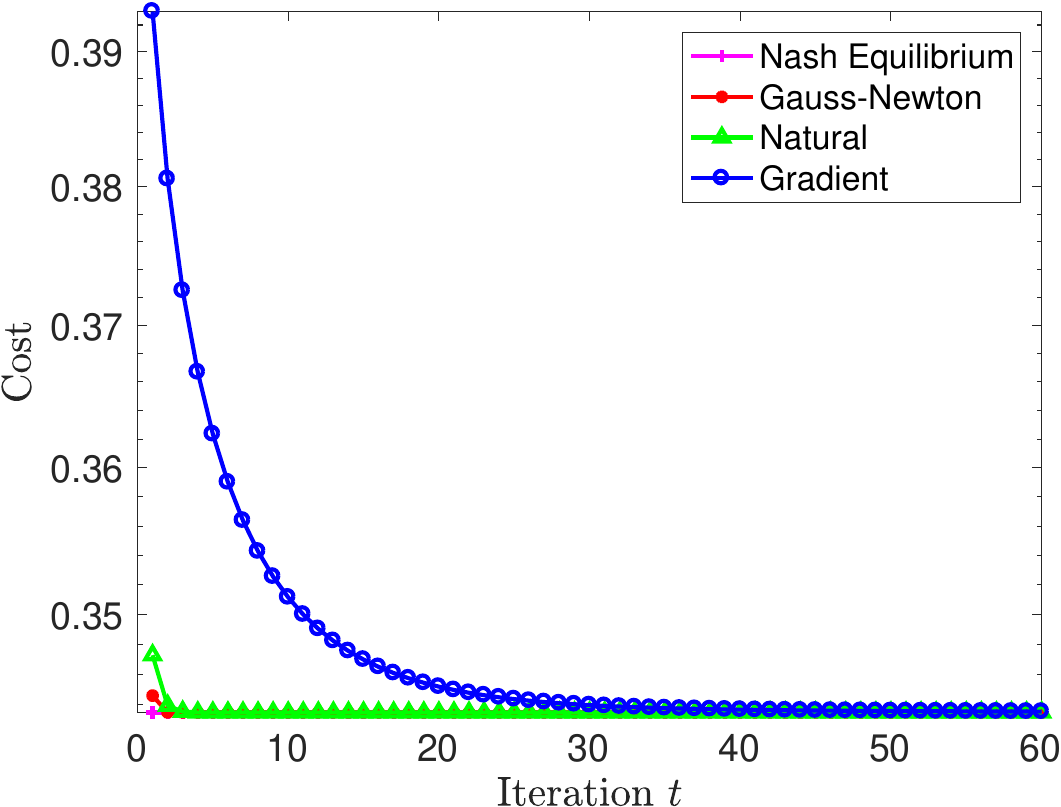}\\
		\hskip-8pt(a) Nested-Gradient  & \hskip 4pt(b) Alternating-Gradient  & \hskip4pt (c) Gradient-Descent-Ascent
	\end{tabular}
	\caption{Convergence of the cost for {\bf Case $2$} where Assumption \ref{assum:invertibility} ii) is not satisfied. (a), (b), and (c) show convergence of the  NG, AG,  and GDA methods, respectively. 
%	 (b) shows the convergence of the gradient  norm square; (c) shows the change of the smallest eigenvalue of $\tilde Q_L=Q-L^\top R^vL$.
}
	\label{fig:case_2_cost_conv}
\end{figure*}
%\vspace{-10pt}

%\begin{figure*}[htbp]
%	\centering
%	\begin{tabular}{ccc}
%		\hskip-6pt\includegraphics[width=0.323\textwidth]{../../simulation/double_loop/figs/nested_grad_cost.pdf}
%		&
%		\hskip-6pt\includegraphics[width=0.323\textwidth]{../../simulation/double_loop/figs/alterna_grad_norm.pdf}
%		& 
%		\hskip-6pt\includegraphics[width=0.323\textwidth]{../../simulation/double_loop/figs/alterna_grad_lambda.pdf}\\
%		\hskip-8pt(a) $\cC(K(L),L)$ v.s. $t$ & \hskip 4pt(b) Grad. Mapping v.s. $t$ & \hskip4pt (c) $\sigma_{\min}(\tilde Q_L)$ v.s. $t$.
%	\end{tabular} 
%	\caption{Performance of the three gradient-descent-ascent  methods. (a) shows the monotone convergence of the expected cost $\cC(K(L),L)$ to the NE cost $\cC(K^*,L^*)$; (b) shows the convergence of the gradient norm square; (c) shows the change of the smallest eigenvalue of $\tilde Q_L=Q-L^\top R^vL$.}
%	\label{fig:case_1_simul}
%\end{figure*} 

Figure \ref{fig:case_1_nested} shows that for {\bf Case $1$}, our nested-gradient methods  indeed enjoy the global convergence to the NE. 
The cost $\cC(K(L),L)$ monotonically increases to that at the NE, and the convergence rate of natural NG sits  between that of the other two NG methods. Also, we note that the convergence rates of gradient mapping square in (b) are linear, which are due to   (c) that $\lambda(\tilde Q_L)$ is always positive along the iteration, i.e., the projection is not effective. This way, our convergence results follow from the local convergence rates  in Theorem \ref{thm:full_grad_conv_local}, although the initialization is random (global). We have also shown in Figure \ref{fig:case_2_cost_conv} that even without Assumption \ref{assum:invertibility} ii), i.e., in {\bf Case $2$}, all the PO methods mentioned successfully converge to the NE, although the cost sequences do not converge monotonically. This motivates us to provide theory for other policy optimization methods, also for more general settings of LQ games. 
We note that no projection was imposed when implementing these algorithms in all our experiments, which  justifies that the projection here is just for the purpose of theoretical  analysis. In fact,  we have not found an instance of LQ games that makes the projections effective as the algorithms proceed. This motivates the  theoretical study of \emph{projection-free} algorithms in our future work.   
More simulation results can be found in \S\ref{sec:append_simu}.

%{\color{red} TO DO: TO BE ADDED.}

%% file: conclusion.tex
\section{Concluding Remarks}
This paper has developed policy optimization methods, specifically,  projected nested-gradient methods, to solve for the Nash equilibria of zero-sum LQ games. In spite of the nonconvexity-nonconcavity of the problem, the gradient-based algorithms have been shown to converge to the NE with globally sublinear and locally linear rates. This work appears  to be  the first one showing that policy optimization methods can converge to the NE of a class of zero-sum Markov games, with finite-iteration  analyses.  Interesting simulation results have demonstrated the superior convergence property of our algorithms, even without the projection operator, and that of the gradient-descent-ascent algorithms with simultaneous updates of both  players, even when  Assumption \ref{assum:invertibility} ii) is relaxed. Based on both the theory and simulation, future directions include convergence analysis for the setting  under a relaxed version of 
 Assumption \ref{assum:invertibility}, and that for the projection-free versions of the algorithms, which we believe can be done by the techniques in our recent work \cite{zhang2018policymixed}. 
% and their model-free versions  
 Besides, developing policy optimization methods for general-sum LQ  games is  another interesting yet challenging future direction.

%We believe  the results  set theoretical foundations  for  developing model-free policy-based reinforcement learning algorithms for zero-sum LQ games.

%% file: appendix_psedo.tex
\newpage

%\section{Deferred Proofs}\label{sec:append_proof_defer}

%~\\
%\centerline{{\fontsize{13.5}{13.5}\selectfont \textbf{Supplementary Materials for ``Policy Optimization Provably}}}
%
%\vspace{6pt}
% \centerline{\fontsize{13.5}{13.5}\selectfont \textbf{
% Converges to Nash Equilibria in Zero-Sum Linear Quadratic Games''}}

\section{Pseudocode for Model-Free Nested-Gradient Algorithms} \label{sec:append_alg} 

In this section, we provide the pseudocode of the  model-free nested-gradient algorithms, 
 which are built upon the nested-gradient updates proposed in \S\ref{sec:algorithms}.

First, as essential elements in the nested-gradient, the gradient ${\nabla_K
  \cC(K,L)}$ and the correlation matrix $\Sigma_{K,L}$ for given $K,L$ can be estimated via samples. The estimates are obtained from the function \textbf{Est($K$;$L$)}, which is tabulated in  Algorithm \ref{alg:est_grad_corre}. The estimate of ${\nabla_K
  \cC(K,L)}$, denoted by $\widehat{\nabla_K
  \cC(K,L)}$,  is obtained via zeroth-order optimization algorithms, where the perturbation $U_i$ is drawn from a ball with fixed radius.

\begin{algorithm}[htpb]
	\caption{\textbf{Est($K$;$L$)}: Estimating  $\widehat{\nabla_K
  \cC(K,L)}$ and $\widehat \Sigma_{K,L}$ at $K$ for given   $L$} 
	\label{alg:est_grad_corre}
	\begin{algorithmic}[1]
		\STATE Input: $K,L$, number of trajectories $m$, rollout length $\cR$,
                smooth parameter $r$, dimension $\tilde d=m_1 d$
%		\INPUT stuff
		\FOR{$i = 1, \cdots m$}
		\STATE Sample a policy $\widehat K_i = K+U_i$, where $U_i$ is
                drawn uniformly at random over matrices with $\|U_i\|_F=r$
		\STATE Simulate $(\widehat K_i,L)$ for $\cR$ steps starting
                from $x_0\sim \cD$, and collect the empirical estimates $\widehat \cC_i$ and $\widehat \Sigma_i$ as:
\[
\widehat \cC_i = \sum_{t=1}^\cR c_t \, , \quad \widehat \Sigma_i = \sum_{t=1}^\cR x_t x_t^\top
\]
where $c_t$ and $x_t$ are the costs and states following  this trajectory
%		\STATE Update:
		\ENDFOR
		\STATE Return the estimates:
\[
\widehat{\nabla_K
  \cC(K,L)} = \frac{1}{m} \sum_{i=1}^m \frac{\tilde d}{r^2} \widehat \cC_i U_i
\, , \quad
\widehat \Sigma_{K,L} = \frac{1}{m} \sum_{i=1}^m \widehat \Sigma_i .
\]
	\end{algorithmic}
\end{algorithm}

\begin{algorithm}[!h] 
	\caption{\textbf{Inner-NG($L$): Model-Free Updates  For Finding $K(L)$}} 
	\label{alg:model_free_inner_NPG}
	\begin{algorithmic}[1]
		\STATE Input: $L$, number of iterations $\cT$,  initialization $K_0$ such that $(K_0,L)$ is stable
%		\INPUT stuff
		\FOR{$\tau = 0, \cdots, \cT-1$}
		\STATE Call \textbf{Est($K_\tau$;$L$)} to obtain the gradient  and the correlation matrix estimates:
		\$
		[\widehat{\nabla_K \cC(K_\tau,L)},\widehat \Sigma_{K_\tau,L}]=\textbf{Est}(K_\tau;L)
		\$
		\STATE Either PG update:
		$~~\quad\quad\quad\quad
		K_{\tau+1} = K_\tau-\alpha \widehat{\nabla_K \cC(K_\tau,L)}, 
		$\\
		or 
		natural PG update:
		$\quad\qquad
		K_{\tau+1} = K_\tau-\alpha \widehat{\nabla_K \cC(K_\tau,L)}\cdot \widehat \Sigma_{K_\tau,L}^{-1}. 
		$
		\ENDFOR
		\STATE Return the iterate  $K_{\cT}$
%\[ 
%\widehat{\nabla_K
%  C(K,L)} = \frac{1}{m} \sum_{i=1}^m \frac{d}{r^2} \widehat C_i U_i
%\, , \quad
%\widehat \Sigma_{K,L} = \frac{1}{m} \sum_{i=1}^m \widehat \Sigma_i .
%\]
	\end{algorithmic}
\end{algorithm}

Given Algorithm \ref{alg:est_grad_corre}, we then summarize the model-free updates for solving the inner-loop minimization problem, i.e., finding $K(L)$ as a subroutine \textbf{Inner-NG($L$)} in Algorithm \ref{alg:model_free_inner_NPG}. Note that among updates \eqref{equ:inner_gd}-\eqref{equ:inner_gauss_newton},  only the policy gradient and the natural PG updates can be converted to model-free versions.

After a finite number $\cT$ of inner-loop updates in Algorithm \ref{alg:model_free_inner_NPG}, the approximate stationary point solution $K_{\cT}$ is then substituted into the outer-loop nested-gradient update, as shown in Algorithm \ref{alg:model_free_outer_NPG}. Note that the example uses projected NG update only, since the corresponding projection operator $ \PP^{GD}_{\Omega}[\cdot]$, see definition in \eqref{equ:def_proj_GD}, does not rely on the iterate $L_t$ at each iteration $t$. Then after a finite number $T$ of projected NG iterates, the algorithm  outputs the solution pair $\big(\widehat{K(L_T)},L_T\big)$.

\begin{algorithm}[!t]
	\caption{\textbf{Outer-NG: Model-Free  Nested-Gradient Algorithms}} 
	\label{alg:model_free_outer_NPG}
	\begin{algorithmic}[1]
		\STATE Input: $L_0$, number of trajectories $m$, number of iterations $T$, rollout length $\cR$, 
                 parameter $r$, dimension $\tilde d=m_2d$
%		\INPUT stuff
		\FOR{$t = 0, \cdots, T-1$}
		\FOR{$i = 1, \cdots m$}
		\STATE Sample a policy $\widehat L_i = L_t+V_i$, where $V_i$ is
                drawn uniformly at random over matrices with $\|V_i\|_F=r$
        \STATE Call \textbf{Inner-NG}($\widehat L_i$) to obtain the estimate of $K(\widehat L_i)$:
        \[
        \widehat{K(\widehat L_i)} = \textbf{Inner-NG}(\widehat L_i)
        \]
		\STATE Simulate $(\widehat{K(\widehat L_i)},\widehat L_i)$ for $\cR$ steps starting
                from $x_0\sim \cD$, and collect the empirical estimates $\widehat C_i$ and $\widehat \Sigma_i$ as:
\[
\widehat \cC_i = \sum_{t=1}^\cR c_t \, , \quad \widehat \Sigma_i = \sum_{t=1}^\cR x_t x_t^\top
\] 
where $c_t$ and $x_t$ are the costs and states following  this trajectory 
%		\STATE Update:
		\ENDFOR
		\STATE Obtain the estimates of the gradient and the correlation  matrix:
\[
\widehat{\nabla_L
  \tilde \cC(L_t)} = \frac{1}{m} \sum_{i=1}^m \frac{\tilde d}{r^2} \widehat \cC_i V_i
\, , \quad
\widehat \Sigma_{\widehat{K(L_t)},L_t} = \frac{1}{m} \sum_{i=1}^m \widehat \Sigma_i 
\]
\STATE Either projected NG update:
		$\quad\quad\quad\quad
		L_{t+1} = \PP^{GD}_{\Omega}\Big[L_t + \eta \widehat{\nabla_L
  \tilde \cC(L_t)}\Big], 
		$\\
		or 
		projected natural NG  update:
		$~~\quad\quad
		L_{t+1} = \PP^{NG}_{\Omega}\Big[L_t + \eta \widehat{\nabla_L
  \tilde \cC(L_t)}\widehat \Sigma_{\widehat{K(L_t)},L_t}^{-1}\Big]. 
		$
%\STATE Projected NG update
%$$
%  $$
%  or natural NG 
%update $L_t$ by $$L_{t+1} = \PP_{\Omega}\bigg[L_t + \eta \widehat{\nabla_L
%  \tilde \cC(L_t)}\cdot \widehat \Sigma_{\widehat{K(L_t)},L_t}^{-1}\bigg]$$
		\ENDFOR
		\STATE Return the iterate $L_T$.	
		\end{algorithmic}
\end{algorithm}

%% file: appendix_aux_proof.tex
\newpage
\clearpage

%\section{Deferred Proofs}\label{sec:append_proof_defer}

\section{Supplementary Proofs}\label{sec:append_aux_proof}
In this section, we provide  supplementary proofs for some results that are either claimed in the paper or used in the proofs before. 
%several auxiliary results and important helper lemmas, together with their proofs. 

\subsection{Proof of Lemma \ref{lemma:saddle_point_formal}}\label{subsec:proof_saddle_point_formal}
\begin{proof}
	Since $Q-(L^*)^\top R^v L^*>0$, we know that $Q>0$, which implies that $(A,Q^{1/2})$ is observable. Then by Theorem $3.7$ in \cite{bacsar2008h}, the existence of $P^*$ in Assumption \ref{assum:invertibility} shows that the value of the game \eqref{equ:def_state_game}   exists. Moreover, by Lemma $3.1$ in \cite{stoorvogel1994discrete}, such a stabilizing solution $P^*$, if exists, is unique. Hence,  by \cite[Theorem $3.7$]{bacsar2008h}, the value of the game \eqref{equ:def_state_game} is represented as $x_0^\top P^* x_0$, and given $\{u_t^*\}_{t\geq 0}$, $\{v_t^*\}_{t\geq 0}$ achieves the upper-value among any control sequence $\{v_t\}_{t\geq 0}$, i.e., for any $x_0\in\RR^d$,
	\#\label{equ:saddle_immed_1}
	\sum_{t=0}^\infty c_t(x_t,u^*_t,v_t)\leq \sum_{t=0}^\infty c_t(x_t,u^*_t,v^*_t). 
	\#
	Also, the closed-loop system $A-BK^*-CL^*$ is stable, i.e., the control pair $(K^*,L^*)$ is stabilizing. 
	
	On the other hand, by \cite{jacobson1977values}, given $\{v_t^*\}_{t\geq 0}$, $\{u_t^*\}_{t\geq 0}$ achieves the lower-value among any stabilizing control sequence $\{u_t\}_{t\geq 0}$; for the control sequence $\{u_t\}_{t\geq 0}$ that is not stabilizing, since $Q-(L^*)^\top R^v L^*>0$, the cost  goes to infinity. Hence,  
		\#\label{equ:saddle_immed_2}
	\sum_{t=0}^\infty c_t(x_t,u^*_t,v^*_t)\leq \sum_{t=0}^\infty c_t(x_t,u_t,v^*_t), 
	\#
	for any control sequence $\{u_t\}_{t\geq 0}$. Combining \eqref{equ:saddle_immed_1} and \eqref{equ:saddle_immed_2} yields that $(\{u_t^*\}_{t\geq 0},\{v_t^*\}_{t\geq 0})$  is a saddle-point of the game, i.e., the NE of the game \eqref{equ:def_state_game}, which completes the proof. 
\end{proof}

\subsection{Proof of Lemma \ref{lemma:nonconvex_nonconcave}}\label{subsec:proof_nonconvex_nonconcave}
\begin{proof}
Since by Assumption \ref{assum:invertibility}, $Q-(L^*)^\top R^v L^*>0$ and $\rho(A-BK^*-CL^*)<1$, it suffices to only consider those $L\in\underline{\Omega}$. For those $L$, $Q+K^\top R^u K-L^\top R^v L>0$, implying that the necessary and sufficient condition for the cost $\cC(K,L)$ to be finite is that the control pair $(K,L)$ is stabilizing.  Thus, we can use the counter-example used in the proof of Lemma $2$ in \cite{fazel2018global}, by  making $B=C=\Ib$, and letting $A-CL$ here equal to the $A$ matrix there, in order to show the nonconvexity of the feasible set of $K$ for these given $L$. Hence, $\min_{K}\cC(K,L)$ is a nonconvex minimization problem.  	Similarly, by letting    $A-BK$ and $C$ here equal to  $A$ and $B$ there, respectively, we know  that the set of stabilizing $L$ for these given $K$ is not convex. Therefore, $\max_{L\in\underline{\Omega}}\cC(K,L)$ is a nonconcave maximization problem, which completes the proof. 
\end{proof}

\subsection{Proof of Lemma \ref{lemma:policy_grad}}\label{subsec:proof_policy_grad}
\begin{proof}
	Let $\cC_{K,L}(x)=x^\top P_{K,L} x$. Then 
	\#\label{equ:lemma_1_1}
	\cC_{K,L}(x_0)=x_0^\top (Q+K^\top R^u K-L^\top R^v L )x_0+\cC_{K,L}((A-BK-CL)x_0).
	\#
	Note that $\cC_{K,L}((A-BK-CL)x_0)$ on the right-hand side  of \eqref{equ:lemma_1_1} has both its subscript and the argument related to $K$. Thus, we have
	\$
	\nabla_K \cC(K,L)=&~2R^uKx_0x_0^\top-2B^\top P_{K,L}(A-BK-CL)x_0x_0^\top+\nabla_K \cC_{K,L}(x_1)\biggiven_{x_1=(A-BK-CL)x_0}\\
	=&~2[(R^u+B^\top P_{K,L}B)K-B^\top P_{K,L}(A-CL)]\cdot \sum_{t=0}^\infty x_t x_t^\top,
	\$
	where the second equation follows from induction. Similarly, we can obtain the gradient w.r.t. $L$ as \eqref{equ:policy_grad_L_form}, which completes the proof.
\end{proof}

\subsection{Proof of Lemma \ref{lemma:stationary_point}}\label{subsec:proof_stationary_point}
\begin{proof}
	Since $\Sigma_{K,L}$ is full-rank, then  
%	$\nabla_K \cC(K,L)=\nabla_L \cC(K,L)=0$ if $2[(R^u+B^\top P_{K,L}B)K-B^\top P_{K,L}(A-CL)]$ and $2[(-R^v+C^\top P_{K,L}C)L-C^\top P_{K,L}(A-BK)]$ lie in the nulls space of $\Sigma_{K,L}$. Otherwise, 
	if $\nabla_K \cC(K,L)=\nabla_L \cC(K,L)=0$, we have
	\#
	K=&~(R^u+B^\top P_{K,L}B)^{-1}B^\top P_{K,L}(A-CL)\label{equ:K_Stationary_0}\\
	L=&~(-R^v+C^\top P_{K,L}C)^{-1}C^\top P_{K,L}(A-BK),\label{equ:L_Stationary_0}
	\#
	provided that the matrix inversion  $(-R^v+C^\top P_{K,L}C)^{-1}$ exists. By solving \eqref{equ:K_Stationary_0} and \eqref{equ:L_Stationary_0}, we obtain that 
	\#
	K=&~[R^u+B^\top  P_{K,L}B-B^\top  P_{K,L}C(-R^v+C^\top  P_{K,L}C)^{-1}C^\top  P_{K,L}B]^{-1}\notag\\
	&\quad\times [B^\top  P_{K,L}A-B^\top  P_{K,L}C(-R^v+C^\top  P_{K,L}C)^{-1}C^\top  P_{K,L}A],\label{equ:K_Stationary}\\
	L=&~[-R^v+C^\top P_{K,L}C-C^\top  P_{K,L}B(R^u+B^\top  P_{K,L}B)^{-1}B^\top  P_{K,L}C]^{-1}\notag\\
	&\quad\times [C^\top  P_{K,L}A-C^\top  P_{K,L}B(R^u+B^\top  P_{K,L}B)^{-1}B^\top  P_{K,L}A].\label{equ:L_Stationary}
	\#

	Now it suffices to compare $P_{K,L}$ and $P^*$. In fact, at the NE, $P^*$ should also satisfy the Lyapunov  equation, i.e., 
	\#\label{equ:P_star_Sol}
	P^*=Q+(K^*)^\top R^u K^*-(L^*)^\top R^v L^*+(A-BK^*-CL^*)^\top P^*(A-BK^*-CL^*),
	\#
	where $K^*$ and $L^*$ satisfy \eqref{equ:K_GARE} and \eqref{equ:L_GARE}. 
	Note that  the set of equations \eqref{equ:K_Stationary}, \eqref{equ:L_Stationary}, and \eqref{equ:P_KL_Sol} is essentially the same as the set of equations  \eqref{equ:K_GARE}, \eqref{equ:L_GARE}, and \eqref{equ:P_star_Sol}.
	Thus, the two sets of equations have identical solutions, which are all solutions to the GARE \eqref{equ:P_GARE} since the latter can be obtained by substituting \eqref{equ:K_GARE} and \eqref{equ:L_GARE} into  \eqref{equ:P_star_Sol}. 
	
	On the other hand, 
	under Assumption \ref{assum:invertibility}, the solution $P^*$ to the GARE \eqref{equ:P_GARE} is unique in the regime of positive definite matrices that generate a stabilizing  control pair $(K^*,L^*)$ following \eqref{equ:K_Stationary}-\eqref{equ:L_Stationary}  \citep{stoorvogel1994discrete,bacsar2008h}. Hence, such a stable control pair $(K,L)$ coincides with the NE pair $(K^*,L^*)$, which completes  the proof.  
\end{proof}

\subsection{Proof of Lemma \ref{lemma:convex_Omega}}\label{subsec:proof_convex_Omega}
\begin{proof}
	Recall the definition of ${\Omega}$ in \eqref{equ:def_Omega} for any  $0<\zeta<\sigma_{\min}(\tilde Q_{L^*})$. Then for any $L_1,L_2\in \Omega$ and $\lambda\in[0,1]$, we have 
	\$
	&[\lambda L_1+(1-\lambda)L_2]^\top R^v [\lambda L_1+(1-\lambda)L_2]\\
	&\quad=\lambda^2 L_1^\top R^v L_1+(1-\lambda)^2 L_2^\top R^v L_2+\lambda(1-\lambda)\big(L_1^\top R^v L_2+L_2^\top R^v L_1\big)\\
	&\quad \leq \lambda^2 L_1^\top R^v L_1+(1-\lambda)^2 L_2^\top R^v L_2+\lambda(1-\lambda)\big(L_1^\top R^v L_1+L_2^\top R^v L_2\big)\\
	&\quad \leq [\lambda^2+(1-\lambda)^2+2\lambda(1-\lambda)]\cdot(Q-\zeta \cdot\Ib)=Q-\zeta \cdot\Ib,
	\$
	where the first inequality follows  from
	 $
	(L_1-L_2)^\top R^v (L_1-L_2)\geq 0
	$ 
	for $R^v>0$, 
	and the second inequality is by  definition of $L_1$ and $L_2$. 
	This shows that
	 $\lambda L_1+(1-\lambda)L_2$ also lies in $\Omega$, which shows that the set $\Omega$ is convex. 
	 
	 Moreover, since the largest eigenvalue of $L^\top R^v L-Q+\zeta\cdot \Ib$, i.e., $\lambda_{\max}(L^\top R^v L-Q+\zeta\cdot \Ib)$ is a continuous function of $L$, and is lower bounded by $-\lambda_{\max}(Q)+\zeta$, 
	 the lower-level set  $\{L\given\lambda_{\max}(L^\top R^v L-Q+\zeta\cdot \Ib)\leq 0\}$ is closed and bounded, i.e., compact, which proves  that $\Omega$ is  compact, thus completing the proof.
\end{proof}

%\subsection{Proof of Lemma \ref{lemma:optim_of_KL}}\label{subsec:proof_optim_of_KL}

\subsection{Proof of Lemma \ref{lemma:proj_prop}}\label{subsec:proof_lemma:proj_prop}
\begin{proof}
	 We choose the proof for the projected natural NG operator $\PP_{\Omega}^{NG}$ as an example. The proofs  for the other two operators  are similar, and follow  directly. Recall that the  following definition of $\PP_{\Omega}^{NG}$  at iterate $L$ is
	 \$
\PP^{NG}_{\Omega}[\tilde L]=\argmin_{\check{L}\in\Omega} ~\tr\Big[\big(\check{L}-\tilde L\big)\Sigma_{L}^* \big(\check{L}-\tilde L\big)^\top\Big],
\$
whose optimality condition can be written as
\$
\tr\Big[\big(\PP_{\Omega}^{NG}[\tilde L]-\tilde L\big)\Sigma_L^*\big(\check{L}-\PP_{\Omega}^{NG}[\tilde L]\big)^\top\Big]\geq 0,\qquad\forall \check{L}\in\Omega.
\$
Letting $\tilde L=L_1$ and $\check{L}=\PP_{\Omega}^{NG}[L_2]$, we have
\#\label{equ:proj_proof_trash_1}
\tr\Big[\big(\PP_{\Omega}^{NG}[L_1]-L_1\big)\Sigma_L^*\big(\PP_{\Omega}^{NG}[L_2]-\PP_{\Omega}^{NG}[L_1]\big)^\top\Big]\geq 0.
\#
Also, letting $\tilde L=L_2$ and $\check{L}=\PP_{\Omega}^{NG}[L_1]$ yields
\#\label{equ:proj_proof_trash_2}
\tr\Big[\big(\PP_{\Omega}^{NG}[L_2]-L_2\big)\Sigma_L^*\big(\PP_{\Omega}^{NG}[L_1]-\PP_{\Omega}^{NG}[L_2]\big)^\top\Big]\geq 0.
\#
Combining \eqref{equ:proj_proof_trash_1} and \eqref{equ:proj_proof_trash_2} leads to 
\$
\tr\Big[\big(L_1-L_2-\PP_{\Omega}^{NG}[L_1]+\PP_{\Omega}^{NG}[L_2]\big)\Sigma_L^*\big(\PP_{\Omega}^{NG}[L_1]-\PP_{\Omega}^{NG}[L_2]\big)^\top\Big]\geq 0,
\$
namely,
\small
\$
\tr\Big[\big(L_1-L_2\big)\Sigma_L^*\big(\PP_{\Omega}^{NG}[L_1]-\PP_{\Omega}^{NG}[L_2]\big)^\top\Big]\geq \tr\Big[\big(\PP_{\Omega}^{NG}[L_1]-\PP_{\Omega}^{NG}[L_2]\big)\Sigma_L^*\big(\PP_{\Omega}^{NG}[L_1]-\PP_{\Omega}^{NG}[L_2]\big)^\top\Big],
\$
\normalsize
which completes the proof.
\end{proof}

\subsection{Proof of Lemma \ref{lemma:conti_PLs}}\label{subsec:lemma_conti_PLs}
\begin{proof}
	Note that the Riccati equation for the inner  problem (see \eqref{equ:Inner_Riccati})  can be rewritten as 
	\#\label{equ:Inner_Riccati_Bellman_way}
	P_L^*=\tilde Q_L+\tilde{A}_L^\top \big[I+P_L^*B(R^u)^{-1}B^\top\big]^{-1}P_L^*\tilde{A}_L. 
	\#
	We now use the implicit function theorem \citep{krantz2012implicit} to show that $P_L^*$ is a continuous function of $L$. 
%	In fact, using the theorem can even show that $P_L^*$ is continuously differentiable w.r.t. $L$.  	
	To this end, it suffices to show that $\vect(P_L^*)$ is continuous w.r.t. $\vect(L)$.

	By vectorizing both sides of \eqref{equ:Inner_Riccati_Bellman_way}, 
	we have
	\$
	\Psi\big(\vect(P_L^*),\vect(L)\big):&=\vect(\tilde Q_L)+\vect\big\{\tilde A_L^\top \big[I+P_L^*B(R^u)^{-1}B^\top\big]^{-1}P_L^*\tilde A_L\big\}\\
	&=\vect(\tilde Q_L)+\big(\tilde A_L^\top \otimes \tilde A_L^\top\big)\vect\big\{\big[I+P_L^*B(R^u)^{-1}B^\top\big]^{-1}P_L^*\big\}=\vect(P_L^*), 
	\$
 where we define a mapping $\Psi:\RR^{d^2}\times \RR^{m_2d}\to \RR^{d^2}$ as above, and also use the  relationship between  Kronecker product and matrix vectorization that for any matrices $A$, $B$, and $X$ with proper dimensions
	\$
	\vect(AXB) =\big(B^\top \otimes A\big) \vect(X). 
	\$ 
Then by the chain rule of  matrix differentials (see Theorem $9$ in \cite{magnus1985matrix}), we  know that 
	\#\label{equ:matrix_diff_Pt}
	&\frac{\partial\vect\big\{\big[I+P_L^*B(R^u)^{-1}B^\top\big]^{-1}P_L^*\big\}}{\partial \vect^\top (P_L^*)}\notag\\
	&\quad=(P_L^*\otimes I)\cdot\frac{\partial \vect\big\{\big[I+P_L^*B(R^u)^{-1}B^\top\big]^{-1}\big\}}{\partial \vect^\top (P_L^*)}+I\otimes \big[I+P_L^*B(R^u)^{-1}B^\top\big]^{-1},
	\#
	where $I$ denotes the identity matrices of compatible dimensions.

	Now we show  that 
   \#\label{equ:matrix_diff_trash_1}
	&\frac{\partial \vect\big\{\big[I+P_L^*B(R^u)^{-1}B^\top\big]^{-1}\big\}}{\partial \vect^\top (P_L^*)}\notag\\
	&\quad=\big\{-B(R^u)^{-1}B^\top \cdot\big[I+P_L^*B(R^u)^{-1}B^\top\big]^{-1}\big\}\otimes \big[I+P_L^*B(R^u)^{-1}B^\top\big]^{-1}.
	\#
	To this end, since both sides of \eqref{equ:matrix_diff_trash_1} are matrices with dimension $d^2\times d^2$, we can compare  the element at the $[(j-1)d+i]$-th row and the $[(l-1)d+k]$-th column on both sides  with $i,j,k,l\in[d]$. On the left-hand side, we first notice  that    
	\$
	&\frac{\partial \vect\big\{\big[I+P_L^*B(R^u)^{-1}B^\top\big]^{-1}\big\}}{\partial [P_L^*]_{k,l}}\\
	&\quad=-\big[I+P_L^*B(R^u)^{-1}B^\top\big]^{-1}\cdot \frac{\partial [P_L^*B(R^u)^{-1}B^\top]}{\partial [P_L^*]_{k,l}}\cdot \big[I+P_L^*B(R^u)^{-1}B^\top\big]^{-1},
	\$
	since  for some matrix function $F$, $(F^{-1})'=-F^{-1}F'F^{-1}$. Also, due to the fact that
	 \$
	 \frac{\partial [P_L^*B(R^u)^{-1}B^\top]}{\partial [P_L^*]_{k,l}}= \left[\begin{matrix}
		\rule[5pt]{60pt}{0.4pt}&0&\rule[5pt]{60pt}{0.4pt}\\
%		 & \vdots &  \\
		\big[B(R^u)^{-1}B^\top\big]_{l,1}&\cdots & \big[B(R^u)^{-1}B^\top\big]_{l,m}  \\  
%		& \vdots & \\
		\rule[2pt]{60pt}{0.4pt}&0&\rule[2pt]{60pt}{0.4pt}
	\end{matrix}\right]\leftarrow k\text{-th row},
	 \$
	 we have
	 \small
	\#\label{equ:matrix_diff_trash_2}
	&\Bigg[\frac{\partial \vect\big\{\big[I+P_L^*B(R^u)^{-1}B^\top\big]^{-1}\big\}}{\partial \vect^\top (P_L^*)}\Bigg]_{(j-1)d+i,(l-1)d+k}=\frac{\partial \Big[\big[I+P_L^*B(R^u)^{-1}B^\top\big]^{-1}\Big]_{i,j}}{\partial [P_L^*]_{k,l}}\notag\\
	&\quad = -\Big[\big[I+P_L^*B(R^u)^{-1}B^\top\big]^{-1}\Big]_{i,k}\cdot\sum_{q=1}^d \big[B(R^u)^{-1}B^\top\big]_{l,q} \cdot\Big[\big[I+P_L^*B(R^u)^{-1}B^\top\big]^{-1}\Big]_{q,j}.
	\#
	\normalsize
	On the right-hand side of \eqref{equ:matrix_diff_trash_1},  we have 
\#\label{equ:matrix_diff_trash_3}
	&\bigg[-\Big\{B(R^u)^{-1}B^\top \cdot\big[I+P_L^*B(R^u)^{-1}B^\top\big]^{-1}\Big\}\otimes \big[I+P_L^*B(R^u)^{-1}B^\top\big]^{-1}\bigg]_{(j-1)d+i,(l-1)d+k}\notag\\
	&\quad=\Big[-B(R^u)^{-1}B^\top \cdot\big[I+P_L^*B(R^u)^{-1}B^\top\big]^{-1}\Big]_{j,l}\cdot \Big[\big[I+P_L^*B(R^u)^{-1}B^\top\big]^{-1}\Big]_{i,k}\notag\\
	&\quad =\Big[-B(R^u)^{-1}B^\top \cdot\big[I+P_L^*B(R^u)^{-1}B^\top\big]^{-1}\Big]_{l,j}\cdot \Big[\big[I+P_L^*B(R^u)^{-1}B^\top\big]^{-1}\Big]_{i,k}, 
	\#
	where the first equation is due to the definition of Kronecker product, and the second one follows from  that  the matrix  
	\$
	&-B(R^u)^{-1}B^\top \cdot\big[I+P_L^*B(R^u)^{-1}B^\top\big]^{-1}\\
	&\quad=-B(R^u)^{-1}B^\top+B(R^u)^{-1}B^\top\big[(P^{*}_L)^{-1}+B(R^u)^{-1}B^\top\big]^{-1}B(R^u)^{-1}B^\top
	\$ is symmetric. 
	Therefore,   for any $(i,j,k,l)\in[d]$, \eqref{equ:matrix_diff_trash_2} and \eqref{equ:matrix_diff_trash_3} are identical, which verifies    \eqref{equ:matrix_diff_trash_1}. 
	
	Combining  \eqref{equ:matrix_diff_trash_1} with   \eqref{equ:matrix_diff_Pt}, we have
\#\label{equ:matrix_diff_trash_4}
		&\frac{\partial\vect\big\{\big[I+P_L^*B(R^u)^{-1}B^\top\big]^{-1}P_L^*\big\}}{\partial \vect^\top (P_L^*)}=I\otimes \big[I+P_L^*B(R^u)^{-1}B^\top\big]^{-1}\notag\\
		&\quad\qquad+(P_L^*\otimes I)\cdot\big\{-B(R^u)^{-1}B^\top \cdot\big[I+P_L^*B(R^u)^{-1}B^\top\big]^{-1}\big\}\otimes \big[I+P_L^*B(R^u)^{-1}B^\top\big]^{-1}\notag\\
		&\quad=\big\{I-P_L^*B(R^u)^{-1}B^\top \cdot\big[I+P_L^*B(R^u)^{-1}B^\top\big]^{-1}\big\}\otimes \big[I+P_L^*B(R^u)^{-1}B^\top\big]^{-1}\notag\\
		&\quad= \big[I+P_L^*B(R^u)^{-1}B^\top\big]^{-1}\otimes \big[I+P_L^*B(R^u)^{-1}B^\top\big]^{-1}
	\#
	where the second equation uses the fact that $(A\otimes B)(C\otimes D)=(AC)\otimes(BD)$ and $(A\otimes B)+(C\otimes B)=(A+C)\otimes B$, and the last one uses matrix inversion lemma. Hence,  we can write the partial derivative of $  \Psi\big(\vect(P_L^*),\vect(L)\big)-\vect(P_L^*)$ as
	\#\label{equ:partial_der}
	&\frac{\partial \big[\Psi\big(\vect(P_L^*),\vect(L)\big)-\vect(P_L^*)\big]}{\partial \vect^\top (P_L^*)}\notag\\
	&\quad=\Big\{\tilde A_L^\top\big[I+P_L^*B(R^u)^{-1}B^\top\big]^{-1}\Big\} \otimes \Big\{\tilde A_L^\top\big[I+P_L^*B(R^u)^{-1}B^\top\big]^{-1}\Big\}-I.  
	\# 
	By definition of $K(L)$ in \eqref{equ:KL_def}, we have  
			\$
		\tilde A_L^\top\big[I+P_L^*B(R^u)^{-1}B^\top\big]^{-1}=\Big\{\tilde A_L\big[I+B(R^u)^{-1}B^\top P_L^*\big]^{-1}\Big\}^{\top}=\big[\tilde A_L-BK(L)\big]^\top.
		\$
		By Lemma \ref{lemma:optim_of_KL}, we know that $L\in\underline{\Omega}$ implies that $(K(L),L)$ is stabilizing, i.e., $\tilde A_L^\top\big[I+P_L^*B(R^u)^{-1}B^\top\big]^{-1}$ has spectral radius less than $1$.
		Therefore, the partial derivative in \eqref{equ:partial_der}  is invertible, since  the eigenvalues of the first matrix  on the right-hand side of \eqref{equ:partial_der} are the products of any two eigenvalues of $\tilde A_L^\top\big[I+P_L^*B(R^u)^{-1}B^\top\big]^{-1}$, which have   absolute values smaller than $1$. In addition,  $\Psi\big(\vect(P_L^*),\vect(L)\big)-\vect(P_L^*)$ is continuous w.r.t. both $\vect(P_L^*)$ and $\vect(L)$. 
		Hence,    we obtain from  the implicit function theorem that $\vect(P_L^*)$ is a continuously differentiable function w.r.t. $\vect(L)$, at some open neighborhood around $L$, 
		so is $P_L^*$ w.r.t. $L$. Note that such an argument holds for any  $L\in\underline{\Omega}$,  
%		 all stabilizing  $(K(L),L)$, 
		 which  completes the proof.  
\end{proof}

\subsection{Proof of Lemma \ref{lemma:perturb_Sigma_L}}\label{subsec:lemma_proof_perturb_Sigma_L}
\begin{proof}
	The proof is composed of several important  lemmas,  following the same vein as the proof of Lemma $16$ in \cite{fazel2018global}. 	
	Note that Assumption \ref{assum:invertibility} is assumed to hold throughout the proof, and  will not be repeated  at each intermediate result. 
	
	We first provide the perturbation result for $P_L^*$ in the following proposition. The results are based on the perturbation theory of algebraic Riccati equations in \cite{konstantinov1993perturbation,sun1998perturbation}, since for given $L$, $P_{L}^*$ is the solution to the inner-loop Riccati equation \eqref{equ:Inner_Riccati} with cost matrix $\tilde Q_{L}=Q-L^\top R^v L$ and transition matrix $\tilde A_{L}=A-CL$.

\begin{proposition}[Perturbation of $P_L^*$]\label{prop:P_L_Perturb}
For any $L,L'\in \Omega$, where 
$\Omega$ is defined in \eqref{equ:def_Omega},  
	there exists some constant $\cB^L_{\Omega}>0$ such that 	if
	$
	\|L'-L\|\leq \cB^L_{\Omega},
	$
	it follows that 
	\$
	\|P_{L'}^*-P_L^*\|\leq \cB^P_{\Omega}\cdot\|L'-L\|,
	\$
	for some constant $\cB^P_{\Omega}>0$.
\end{proposition}
\begin{proof}
	The proof is built upon the result of Theorem $4.1$ in \cite{sun1998perturbation}. 
First, since both $L,L'\in\Omega$, we have $\tilde Q_{L'},\tilde Q_{L}\geq 0$, and also $B(R^u)^{-1}B^\top\geq 0$ for both $L$ and $L'$. This validates the applicability  of  \cite[Theorem $4.1$]{sun1998perturbation}. 
Also note that by Lemma \ref{lemma:optim_of_KL}, both $P_L^*$ and $P_{L'}^*$ exist and are positive definite.  
First recalling the definition of $K(L)$ in \eqref{equ:KL_def}, we have the following relationship:
\$
\tilde A_L-BK(L)=\tilde A_L-B(R^u+B^\top P_{L}^*B)^{-1}B^\top P_{L}^*\tilde A_L=[\Ib+B(R^u)^{-1}B^\top P_{L}^*]^{-1}\tilde A_L,
\$
where the second equation uses the matrix inversion lemma. 

To simplify the notation, we let $\Delta=L'-L$ and also define the following quantities\footnote{Note that we change some of the notations used in \cite[Theorem $4.1$]{sun1998perturbation} in order to:  i) avoid the conflict with our notations; ii) simplify the bound for better readability.}:
\$
&\delta=\big\|\tilde A_{L'}-\tilde A_{L}\big\|=\|C\Delta\|,\quad f=\big\|[\Ib+B(R^u)^{-1}B^\top P_{L}^*]^{-1}\big\|,\quad g=\|B(R^u)^{-1}B^\top\|\\
&\phi=\big\|[\Ib+B(R^u)^{-1}B^\top P_{L}^*]^{-1}\tilde A_L\big\|,\quad\gamma =f\delta(2\phi+f\delta),\quad \psi=\big\|P_{L}^*\cdot[\Ib+B(R^u)^{-1}B^\top P_{L}^*]^{-1}\big\|\\
&T_L=\Ib-\big[\tilde A_L-BK(L)\big]^\top\otimes \big[\tilde A_L-BK(L)\big]^\top,~ \ell=\|T_L^{-1}\|^{-1},~H=P_L^*[\Ib+B(R^u)^{-1}B^\top P_{L}^*]^{-1}\tilde A_L\\
&p=\big\|T_L^{-1}\big[\Ib\otimes H^\top +(H^\top \otimes\Ib)\Pi\big]\big\|,~~\varepsilon =\frac{1}{\ell}\big\|\Delta R^v L+L^\top R^v \Delta+\Delta^\top R^v \Delta\big\|+\Big(p+\frac{\psi \delta}{\ell}\Big)\|C\Delta\|\\
&\qquad\qquad\qquad\qquad\qquad\alpha=f(\|\tilde A_L\|+\|C\Delta\|),\qquad \theta=\frac{\ell}{\phi+\sqrt{\phi^2+\ell}},
\$
where $\Pi$ is the vec-permutation matrix \citep[pp. 32-34]{graham2018kronecker}. 
Also, from Lemma \ref{lemma:optim_of_KL},  we know that $\tilde A_L-BK(L)$ is stabilizing, which thus implies that $\ell$ is finite and 
\#\label{equ:l_greater_0}
\ell=1/\|T_L^{-1}\|=\sigma_{\min}(T_L)>0.
\#  
We note that since $\Omega$ is a compact set of $L$, and $\sigma_{\min}(T_L)$ is a continuous function of $L$, $\ell$ is uniformly lower bounded above zero for any $L\in\Omega$. 

Since the term $\|\Delta G\|$ in \cite[Theorem $4.1$]{sun1998perturbation} is zero here, the first condition in (4.40) of \cite{sun1998perturbation} is trivially  satisfied. 
For the other  two conditions in (4.40) there, we require the following sufficient conditions  to hold 
\#\label{equ:veri_440_2}
1-fg\xi_*\geq0, \quad \frac{f\delta+\phi fg\xi_*}{1-fg\xi_*}\leq \theta,
\#
where $\xi_*$ is d efined as 
$
\xi_*={(2\ell\varepsilon)}\cdot{(\ell/2+\ell fg\varepsilon)^{-1}}.
$
Note that if we additionally require 
\small
\#\label{equ:prop_pert_P_cond_1}
\gamma=f\delta(2\phi+f\delta)\leq f\|C\Delta\|(2\phi+2\ell+f\|C\Delta\|)\leq 2f(\phi+\ell)\|C\|\|\Delta\|+f^2\|C\|^2\|\Delta\|^2\leq \ell/2,
\#
\normalsize
then the definition of $\xi_*$ here is strictly larger than that in \cite{sun1998perturbation}. 
Thus, if such an $\xi_*$ satisfies \eqref{equ:veri_440_2},  
then the  other two conditions in \eqref{equ:veri_440_2} can be satisfied, too. 
Moreover, if we also let  
\#\label{equ:veri_441}
\varepsilon &=\frac{1}{\ell}\big\|\Delta R^v L+L^\top R^v \Delta+\Delta^\top R^v \Delta\big\|+\Big(p+\frac{\psi \delta}{\ell}\Big)\|C\Delta\|\notag\\
&\leq \frac{1}{\ell}\big(2\|R^v L\|\|\Delta\|+\|R^v\|\|\Delta\|^2\big)
+\Big(p+\frac{\psi \delta}{\ell}\Big)\|C\Delta\|\leq \frac{(\ell/2)^2}{2\ell fg(\ell+2{\alpha})}=\frac{\ell}{8fg(\ell+2{\alpha})}
\#
hold, then since $\gamma\leq \ell/2$ from \eqref{equ:prop_pert_P_cond_1},  the right-hand side of \eqref{equ:veri_441} satisfies
\$
\frac{(\ell/2)^2}{2\ell fg(\ell+2{\alpha})}\leq \frac{(\ell-\gamma)^2}{\ell fg(\ell-\gamma+2\alpha+\sqrt{(\ell-\gamma+2\alpha)^2-(\ell-\gamma)^2}}.
\$
This implies that the condition in (4.41) in  \cite{sun1998perturbation} holds. Then, we obtain from Theorem $4.1$ in \cite{sun1998perturbation} that 
\#\label{equ:P_pert_trash_1}
&\big\|P_{L'}^*-P_{ L}^*\big\|\leq \xi_*=\frac{2\ell\varepsilon}{\ell/2+\ell fg\varepsilon}\leq 4\varepsilon=\frac{4}{\ell}\big\|\Delta R^v L+L^\top R^v \Delta+\Delta^\top R^v \Delta\big\|+4\Big(p+\frac{\psi \delta}{\ell}\Big)\|C\Delta\|\notag\\
&\quad \leq \frac{8}{\ell}\| R^v L\|\|\Delta\|+\frac{4}{\ell}\|R^v\|\|\Delta\|^2+4p\|C\|\|\Delta\|+4\frac{\psi}{\ell}\|C\|^2\|\Delta\|^2. 
\#

Now we discuss   sufficient  conditions of \eqref{equ:veri_440_2}, \eqref{equ:prop_pert_P_cond_1}, and \eqref{equ:veri_441}, to ensure a perturbation bound on $P_L^*$ as desired from \eqref{equ:P_pert_trash_1}. 
The two conditions in \eqref{equ:veri_440_2} can be written as 
\#\label{equ:P_pert_trash_2}
fg\frac{2\varepsilon}{1/2+fg\varepsilon}\leq 1\Longrightarrow fg\varepsilon\leq 1/2,\quad f\delta +\big(\phi+\theta\big)fg\xi_*\leq \theta,
\#
where one sufficient condition for the second one to hold is 
\#\label{equ:P_pert_trash_3}
f\delta +4(\phi+\theta)fg\varepsilon\leq\theta,
\#
since $\xi_*\leq 2\varepsilon/(1/2)=4\varepsilon$.
Note that since $f\delta\geq 0$ and $fg\varepsilon\geq 0$, \eqref{equ:P_pert_trash_3} holds implies that $fg\varepsilon\leq 1/2$. Hence we only need a  sufficient condition for  \eqref{equ:P_pert_trash_3} to hold, which can be the following one
\small
\#\label{equ:P_pert_final_trash_1}
f\|C\|\|\Delta\| +4(\phi+\theta)fg\bigg(\frac{2}{\ell}\| R^v L\|\|\Delta\|+\frac{1}{\ell}\|R^v\|\|\Delta\|^2+p\|C\|\|\Delta\|+\frac{\psi}{\ell}\|C\|^2\|\Delta\|^2\bigg)\leq\theta. 
\#
\normalsize
\eqref{equ:P_pert_final_trash_1} can be satisfied if the following condition  on $\|\Delta\|$ holds:
\small
\#\label{equ:P_pert_final_trash_1_final}
\|\Delta\| \leq \min\bigg\{\frac{\|C\|+4(\phi+\theta)g(2\|R^vL\|/\ell+p\|C\|)}{4(\phi+\theta)g(\|R^v\|/\ell+\psi\|C\|^2/\ell)},\frac{\theta}{2f\|C\|+8f(\phi+\theta)g(2\|R^vL\|/\ell+p\|C\|)}\bigg\}.
\#
\normalsize
Moreover, the condition in \eqref{equ:prop_pert_P_cond_1} gives 
\#\label{equ:P_pert_final_trash_2}
2f(\phi+\ell)\|C\|\|\Delta\|+f^2\|C\|^2\|\Delta\|^2\leq \ell/2,
\#
which can be satisfied by the following condition on $\|\Delta\|$:
\#\label{equ:P_pert_final_trash_2_final}
\|\Delta\|\leq \min\bigg\{\frac{2(\phi+\ell)}{f\|C\|},\frac{\ell}{8f(\phi+\ell)\|C\|}\bigg\}. 
\#
Also, by letting 
%\small
\#\label{equ:P_pert_final_trash_3}
&2\alpha=2f(\|\tilde A_L\|+\|C\Delta\|) \leq 2f(\|\tilde A_L\|+\|C\|)\\
&\Longrightarrow \frac{\ell}{8fg(\ell+2{\alpha})}\geq \frac{\ell}{8fg[\ell+2f(\|\tilde A_L\|+\|C\|)]},\notag
\#
\normalsize 
the condition in \eqref{equ:veri_441} can thus be satisfied if  we let
\#\label{equ:P_pert_final_trash_4}
\frac{1}{\ell}\big(2\|R^v L\|\|\Delta\|+\|R^v\|\|\Delta\|^2\big)
+(p+1)\|C\|\|\Delta\|+\frac{\psi}{\ell}\|C\|^2\|\Delta\|^2\leq \frac{\ell}{8fg[\ell+2f(\|\tilde A_L\|+\|C\|)]}.
\#
Note that conditions \eqref{equ:P_pert_final_trash_3}-\eqref{equ:P_pert_final_trash_4} can be satisfied if
\#\label{equ:P_pert_final_trash_3_and_4_final}
\|\Delta\|\leq \min\bigg\{1,\frac{2\|R^v L\|+(p+1)\ell\|C\|}{\|R^v\|+\psi\|C\|^2},\frac{\ell}{16fg[\ell+2f(\|\tilde A_L\|+\|C\|)] (\|R^v L\|/\ell+2(p+1)\|C\|)}\bigg\}.
\#
Thus, under  \eqref{equ:P_pert_final_trash_1_final}, \eqref{equ:P_pert_final_trash_2_final}, and \eqref{equ:P_pert_final_trash_3_and_4_final}, the  bound  \eqref{equ:P_pert_trash_1} can be further written as 
\#\label{equ:P_L_pert_final_trash}
&\big\|P_{L'}^*-P_{ L}^*\big\|\leq \bigg[\frac{\|C\|}{(\phi+\theta)g}
+\frac{16\|R^vL\|}{\ell}+8p\|C\|\bigg]\cdot\|\Delta\|, 
\#
where the inequality follows by using the first bound of $\Delta$ in the  $\min$ of \eqref{equ:P_pert_final_trash_1_final}.  
It is straightforward to see that  all the upper bounds on $\|\Delta\|$ from  \eqref{equ:P_pert_final_trash_1_final}, \eqref{equ:P_pert_final_trash_2_final}, and \eqref{equ:P_pert_final_trash_3_and_4_final} are lower bounded above zero, since: i) $\ell$, $\theta$ and $\|C\|$ are all strictly above zero (see \eqref{equ:l_greater_0}), so are all the numerators of the bounds  in \eqref{equ:P_pert_final_trash_1_final}, \eqref{equ:P_pert_final_trash_2_final}, and \eqref{equ:P_pert_final_trash_3_and_4_final}; ii) the denominators of the bounds are all finite  and bounded above, due to the boundedness of $L$, i.e., the boundedness of $\Omega$,  and the boundedness of $P_L^*$ from Lemma \ref{lemma:optim_of_KL}. 
In addition, note that all the quantities used in the bounds on $\|\Delta\|$ are norms of matrices composed of $L$ and $P_L^*$, which are both  continuous functions of $L$ (see Lemma \ref{lemma:conti_PLs} on the continuity of $P_L^*$), over the compact set $\Omega$. Hence, there exists some constant $\cB^L_{\Omega}>0$, which is the infimum  of the bounds on $\|\Delta\|$ over $\Omega$. Also, from \eqref{equ:P_L_pert_final_trash}, there exists some 
\$
\cB_{\Omega}^P=\frac{\|C\|}{(\phi+\theta)g}
+\frac{16\|R^vL\|}{\ell}+8p\|C\|,
\$
such that $\|P_{L'}^*-P_L^*\|\leq \cB^P_{\Omega}\cdot\|L'-L\|$, which  completes the proof. 
\end{proof}

We then need to establish the perturbation   of $K(L)$  as in the following lemma. 

\begin{lemma}\label{lemma:K_L_Perturb}
	For any $L,L'\in \Omega$, recalling the definition of $K(L)$  in \eqref{equ:KL_def},  
	there exists some constant $\cB^L_{\Omega}>0$ such that 	if
	\#\label{equ:lemma_L_bnd}
	\|L'-L\|\leq \min\bigg\{\cB^L_{\Omega},\frac{\|B\|\cdot(\cB^P_{\Omega}\|\tilde A_{L}-BK(L)\|+\|P_{L}^*\|\|C\|)}{\cB^P_{\Omega}\|B\|\|C\|}\bigg\},
	\#
	 it follows that
%	\$
%	&\|K(L')-K(L)\|\leq \frac{\|B\|\cdot(\cB^P_{\Omega}\|\tilde A_{L}-BK(L)\|+\|P_{L}^*\|\|C\|)}{\sigma_{\min}(R^u)}\|L'-L\|+\frac{\cB^P_{\Omega}\|B\|\|C\|}{\sigma_{\min}(R^u)}\|L'-L\|^2,
%	\$
		\$
	&\|K(L')-K(L)\|\leq \frac{2\|B\|\cdot(\cB^P_{\Omega}\|\tilde A_{L}-BK(L)\|+\|P_{L}^*\|\|C\|)}{\sigma_{\min}(R^u)}\cdot\|L'-L\|,
%	+\frac{\cB^P_{\Omega}\|B\|\|C\|}{\sigma_{\min}(R^u)}\|L'-L\|^2,
	\$
	where $\cB^L_{\Omega},\cB^P_{\Omega}$ are as defined in the proof of Proposition \ref{prop:P_L_Perturb}.
\end{lemma}
\begin{proof}
	By definition, it holds that
	\$
	(R^u+B^\top P_{\tilde L}^*B)K(\tilde L)=B^\top P_{\tilde L}^*\tilde A_{\tilde L}
	\$
	for both $\tilde L=L$ and $\tilde L=L'$. Subtracting both equations yields
	\small
	\$
	B^\top(P_{L'}^*-P_L^*)BK(L) +(R^u+B^\top P_{L'}B)[K(L')-K(L)]=B^\top (P_{L'}^*-P_L^*)\tilde A_{L'}+B^\top P_L^*C(L-L'),
	\$
	\normalsize
	which further gives
	\$
	&\|K(L')-K(L)\|=\|(R^u+B^\top P_{L'}B)^{-1}B^\top (P_{L'}^*-P_L^*)[\tilde A_{L}-BK(L)+C(L-L')]\\ 
	&\qquad\qquad\qquad\qquad\qquad+(R^u+B^\top P_{L'}B)^{-1}B^\top P_{L}^*C(L-L')\|\\\notag
	&\leq \|(R^u+B^\top P_{L'}B)^{-1}\|\|B\|\big[\|P_{L'}^*-P_{L}^*\|\big(\|\tilde A_{L}-BK(L)\|+\|C\|\|L'-L\|\big)+\|P_{L}^*\|\|C\|\|L'-L\|\big]\\\notag
	&\leq \frac{\|B\|}{\sigma_{\min}(R^u)}\|P_{L'}^*-P_{L}^*\|\big(\|\tilde A_{L}-BK(L)\|+\|C\|\|L'-L\|\big)+\frac{\|B\|}{\sigma_{\min}(R^u)}\|P_{L}^*\|\|C\|\|L'-L\|.
	\$
	Combined with the bound on $\|P_{L'}^*-P_{L}^*\|$ in Proposition \ref{prop:P_L_Perturb}, we obtain that
	\$
	&\|K(L')-K(L)\|\leq \frac{\|B\|\cdot(\cB^P_{\Omega}\|\tilde A_{L}-BK(L)\|+\|P_{L}^*\|\|C\|)}{\sigma_{\min}(R^u)}\|L'-L\|+\frac{\cB^P_{\Omega}\|B\|\|C\|}{\sigma_{\min}(R^u)}\|L'-L\|^2,
	\$
	which combined with the bound on \eqref{equ:lemma_L_bnd} gives the  desired result. 
\end{proof}

Now we are ready to establish the perturbation of $\Sigma_L^*$. 
We start 
by defining a linear  operator on symmetric matrices $\cT_L^*(\cdot)$:
	\$
	\cT_L^*(X):=\sum_{t=0}^\infty [A-BK(L)-CL]^t X [A-BK(L)-CL]^{t^\top},
	\$
	and its induced norm as 
	\$
	\|\cT_L^*\|:=\sup_{X}\frac{\cT_L^*(X)}{\|X\|},
	\$
	where $\sup$ is taken over all non-zero symmetric matrices.
Also, we let $\Sigma_0=\EE(x_0x_0^\top)$. 
	Then we can show that the induced norm $\|\cT_L^*\|$ is bounded as follows.

\begin{lemma}\label{lemma:T_L_norm_bnd}
For any $L\in\Omega$, 
	the induced norm pf $\|\cT_L^*\|$ is bounded as
	\$
	\|\cT_L^*\|\leq \frac{\cC(K(L),L)}{\mu\cdot\zeta}.
	\$
\end{lemma}
\begin{proof}
The proof mostly follows the proof of Lemma $17$ in \cite{fazel2018global}, except replacing $(A-BK)$ there by $A-BK(L)-CL$, and the upper bound of $\|\Sigma_L^*\|$ by $\cC(K(L),L)/\zeta$ due to Lemma \ref{lemma:bound_P_L_Sigma_L}. 
\end{proof}

We can also define another operator $\cF_L^*(X)$ as
\$
\cF_L^*(X)=[A-BK(L)-CL]X[A-BK(L)-CL]^\top, 
\$	
which, by the same argument as Lemma $18$ in \cite{fazel2018global}, gives that
\#\label{equ:rela_F_Ls_T_Ls}
\cT_L^*=(\Ib-\cF_L^*)^{-1},
\#
where $\Ib$ is the identity operator. 
Hence, the following proof is to find the bound of
\$
\|\Sigma_{L'}^*-\Sigma_{L}^*\|=\|(\cT_{L'}^*-\cT_L^*)(\Sigma_0)\|=\|[(\Ib-\cF_{L'}^*)^{-1}-(\Ib-\cF_{L}^*)^{-1}](\Sigma_0)\|.
\$
To this end, we first have the following bound on $\|\cF_L^*-\cF_{L'}^*\|$. 

\begin{lemma}\label{lemma:pert_F_L}
For any $L,L'\in\Omega$, it follows  that
	\$
&\|\cF_{L'}^*-\cF_{L}^*\|\leq  2\|A-BK(L)-CL\|\big(\|B\|\|K(L')-K(L)\|+\|C\|\|\Delta\|\big)\\
&\qquad\qquad\quad\quad+\|B\|^2\|K(L')-K(L)\|^2+\|C\|^2\|\Delta\|^2+2\|B\|\|C\|\|K(L')-K(L)\|\|\Delta\|.
	\$
\end{lemma}
\begin{proof}
Let $\Delta=L'-L$, then for any symmetric matrix $X$,
\$
(\cF_{L'}^*-\cF_{L}^*)(X)=&-[A-BK(L)-CL]X\big\{B[K(L')-K(L)]+C\Delta\big\}^\top\\
&\quad -\big\{B[K(L')-K(L)]+C\Delta\big\}X[A-BK(L)-CL]^\top\\
&\quad +\big\{B[K(L')-K(L)]+C\Delta\big\}X\big\{B[K(L')-K(L)]+C\Delta\big\}^\top,
\$
which leads to the desired norm bound by using $\|AX\|\leq \|A\|\|X\|$ for any operator  $A$.
\end{proof}

Moreover, we have the following argument similar to Lemma $20$ in \cite{fazel2018global}. 

\begin{lemma}\label{lemma:lemma_20_ge_rong}
	If $\|\cT_L^*\|\|\cF_{L'}^*-\cF_{L}^*\|\leq 1/2$, and both $(K(L'),L')$ and $(K(L),L)$ are stabilizing.  Then
	\$
	\|(\cT_{L'}^*-\cT_{L}^*)(\Sigma)\|\leq 2\|\cT_{L}^*\|\|\cF_{L'}^*-\cF_{L}^*\|\|\cT_{L}^*(\Sigma)\|\leq 2\|\cT_{L}^*\|^2\|\cF_{L'}^*-\cF_{L}^*\|\|\Sigma\|.
	\$
\end{lemma}
\begin{proof}
The proof follows directly from that of Lemma $20$ in \cite{fazel2018global}, which is omitted here for brevity.   
\end{proof}

We are now ready to prove the   perturbation of $\Sigma_L^*$. 
To simplify the notation, let
\$
\cB_{\Omega}^K=\frac{2\|B\|\cdot\big(\cB^P_{\Omega}\|\tilde A_{L}-BK(L)\|+\|P_{L}^*\|\|C\|\big)}{\sigma_{\min}(R^u)},
\$
then $\cB_{\Omega}^K>0$. By Lemmas \ref{lemma:K_L_Perturb} and  \ref{lemma:pert_F_L}, for any $L,L'\in\Omega$, letting $\Delta=L'-L$,  if
\small
\$
\|\Delta\|\leq \min\Bigg\{\cB^L_{\Omega},\frac{\|B\|\big[\cB^P_{\Omega}\|\tilde A_{L}-BK(L)\|+\|P_{L}^*\|\|C\|\big]}{\cB^P_{\Omega}\|B\|\|C\|},\frac{2\big(\|\tilde A_L-BK(L)\|+1\big)\big(\cB_{\Omega}^K\|B\|+\|C\|\big)}{\big(\cB_{\Omega}^K\big)^2\|B\|^2+\|C\|^2+2\cB_{\Omega}^K\|B\|\|C\|}\Bigg\},
\$
\normalsize
then 
\$
&\|\cF_{L'}^*-\cF_{L}^*\|\leq  2\|\tilde A_L-BK(L)\|\big(\cB_{\Omega}^K\|B\|\|\Delta\|+\|C\|\|\Delta\|\big)\\
&\qquad\qquad\qquad\qquad+\big(\cB_{\Omega}^K\big)^2\|B\|^2\|\Delta\|^2+\|C\|^2\|\Delta\|^2+2\|B\|\|C\|\cB_{\Omega}^K\|\Delta\|^2\\
&\quad\leq  2\big(\|\tilde A_L-BK(L)\|+1\big)\big(\cB_{\Omega}^K\|B\|\|\Delta\|+\|C\|\|\Delta\|\big)\\
&\qquad\qquad\qquad\qquad+\big(\cB_{\Omega}^K\big)^2\|B\|^2\|\Delta\|^2+\|C\|^2\|\Delta\|^2+2\|B\|\|C\|\cB_{\Omega}^K\|\Delta\|^2\\
&\quad\leq  4\big(\|\tilde A_L-BK(L)\|+1\big)\big(\cB_{\Omega}^K\|B\|+\|C\|\big)\cdot\|\Delta\|,
\$
where the first inequality uses Lemma \ref{lemma:K_L_Perturb}, and the second inequality is due to the  third term in the $\min$ of the upper bound on $\|\Delta\|$. This completes the proof of Lemma \ref{lemma:perturb_Sigma_L}.
\end{proof}

%
%
%\begin{lemma}\label{lemma:Sigma_Pert}
%	For the two maximizer controllers $L,L'\in\Omega$ (recall the  definition of  $\Omega$ in \eqref{equ:def_Omega}) that satisfy 
%	\$ 
%	\|L'-L\|\leq \frac{\zeta \mu}{4\cC(K(L),L)\|C\|(\|A-BK(L)-CL\|+1)},  
%	\$
%	 it follows that
%	\$
%	\|\Sigma_{L'}^*-\Sigma_{L}^*\|\leq XXXXXX.
%	\$
%\end{lemma}
%\begin{proof}
%	The proof is similar to that for Lemma $16$, with additional arguments on the perturbation of $K(L)$ and $P_L^*$. To this end, we 
%\end{proof}
 
%This Lemma \ref{lemma:Sigma_Pert} and Lemma $10$ in \cite{fazel2018global}, together with the form of \eqref{equ:inner_gd}, give that 
%\#\label{equ:contract_P_K_gd}
%&V_{K',L}(x)-V_{K,L}(x)= 4\alpha^2\tr\bigg[\sum_{t\geq0}(x'_t)(x'_t)^\top \Sigma_{K,L}E_{K,L}^\top(R^u+B^\top P_{K,L}B)E_{K,L}\Sigma_{K,L}\bigg]\notag\\
%&\qquad\qquad\qquad\qquad\qquad-4\alpha\tr\bigg[\sum_{t\geq0}(x'_t)(x'_t)^\top \Sigma_{K,L}E_{K,L}^\top E_{K,L}\bigg]\notag\\
%&\quad\leq 4\alpha^2\|R^u+B^\top P_{K,L}B\|\cdot\|\|^2\tr\bigg[\sum_{t\geq0}(x'_t)(x'_t)^\top \Sigma_{K,L}E_{K,L}^\top E_{K,L}\Sigma_{K,L}\bigg]\notag\\ 
%&\quad-4\alpha\tr\bigg[\sum_{t\geq0}(x'_t)(x'_t)^\top \Sigma_{K,L}E_{K,L}^\top E_{K,L}\bigg]\notag\\
%&\leq 0,
%\# 
%

\begin{lemma}\label{lemma:dist_compact}
	For any disjoint sets $\cA,\cB\subseteq\RR^{m\times n}$, if $\cA$ is compact, and if $\cB$ is closed, then there exists some $\omega>0$, such that for any $A\in \cA$ and $B\in\cB$, $\|A-B\|\geq \omega$.
\end{lemma}
\begin{proof}
	Assume that the conclusion does not hold. Let $A_n\in\cA$ and $B_n\in\cB$ be chosen such that $\|A_n-B_n\|\to 0$ as $n\to \infty$. Since $\cA$ is compact, there exists a convergent subsequence of  $\{A_n\}_{n\geq 0}$, denoted by $\{A_{n_m}\}_{m\geq 0}$, that converges to some $A\in\cA$. Hence, we have
	\$
	\|A-B_{n_m}\|\leq \|A-A_{n_m}\|+\|A_{n_m}-B_{n_m}\|\to 0,
	\$ 
	as  $m\to\infty$. This implies that $A$ is a limit point of $\cB$. Since $\cB$ is closed, we have $A\in\cB$, which leads to a contradiction and thus completes the proof. 
\end{proof}

%% file: appendix_simulation.tex
\newpage
\clearpage

%\section{Deferred Proofs}\label{sec:append_proof_defer}

\section{Simulation Details}\label{sec:append_simu}

\vspace{3pt}
{\noindent \bf Alternating-Gradient (AG) Methods.}

AG  methods follows the idea in \cite{nouiehed2019solving}, which are based on our nested-gradient methods, but at each outer-loop iteration, the  inner-loop updates only perform a finite number of iterations, instead of converging to the exact solution $K(L_t)$ as nested-gradient methods. The updates are given in Algorithm \ref{alg:AG}, whose performance is showcased in Figures  \ref{fig:case_1_AG} and \ref{fig:case_2_AG}, showing that AG methods converge to the NE in both settings.

\begin{algorithm}[!th]
	\caption{\textbf{Alternating-Gradient (AG) Methods}} 
	\label{alg:AG}
	\begin{algorithmic}[1]
		\STATE Input: $(K_0,L_0)$ that is  stabilizing
%		\INPUT stuff
		\FOR{$t = 0, \cdots, T-1$}
		\FOR{$\tau  = 0, \cdots \cT-1$}
		\STATE \begin{flalign*}
{\rm \textbf{Policy Gradient:}}\qquad ~~  	 K_{\tau+1}&=K_{\tau}-	\eta \nabla_K{\cC}(K_{\tau},L_t),
\end{flalign*} 
        \STATE Or
        \begin{flalign*}
{\rm \textbf{Natural Policy Gradient:}}\qquad ~~  	 K_{\tau+1}&=K_\tau-	\eta \nabla_K{\cC}(K_\tau,L_t)\Sigma^{-1}_{K_\tau,L_t},
\end{flalign*} 
        \STATE Or
        \begin{flalign*}
{\rm \textbf{Gauss-Newton:}}\qquad ~~  	 K_{\tau+1}&=K_\tau-	\eta (R^u+B^\top P_{K_\tau,L_t}B)^{-1}\nabla_K{\cC}(K_\tau,L_t)\Sigma^{-1}_{K_\tau,L_t},
\end{flalign*}
		\ENDFOR	
		\STATE \begin{flalign*}
{\rm \textbf{Policy Gradient:}}\qquad ~~  	 
L_{t+1}&=L_t+	\eta \nabla_L{\cC}(K_{\cT},L_t),
\end{flalign*} 
        \STATE Or
        \begin{flalign*}
{\rm \textbf{Natural Policy Gradient:}}\qquad ~~ 
L_{t+1}&=L_t+	\eta \nabla_L{\cC}(K_{\cT},L_t)\Sigma^{-1}_{K_{\cT},L_t},
\end{flalign*} 
        \STATE Or
        \begin{flalign*}
{\rm \textbf{Gauss-Newton:}}\qquad ~~  	 
L_{t+1}&=L_t+	\eta(R^v-C^\top P_{K_{\cT},L_t}C)^{-1} \nabla_L{\cC}(K_{\cT},L_t)\Sigma^{-1}_{K_{\cT},L_t},
\end{flalign*}	
		\ENDFOR
		\STATE Return the iterate $(K_{\cT},L_T)$.	
		\end{algorithmic}
\end{algorithm}

\begin{algorithm}[!t]
	\caption{\textbf{Gradient-Descent-Ascent (GDA) Methods}} 
	\label{alg:GDA}
	\begin{algorithmic}[1]
		\STATE Input: $(K_0,L_0)$ that is  stabilizing
%		\INPUT stuff
		\FOR{$t = 0, \cdots, T-1$}
%		\FOR{$\tau  = 0, \cdots \cT-1$}
		\STATE \begin{flalign*}
{\rm \textbf{Policy Gradient:}}\qquad ~~  	 K_{t+1}&=K_t-	\eta \nabla_K{\cC}(K_t,L_t)\\
L_{t+1}&=L_t+	\eta \nabla_L{\cC}(K_t,L_t),
\end{flalign*} 
        \STATE Or
        \begin{flalign*}
{\rm \textbf{Natural Policy Gradient:}}\qquad ~~  	 K_{t+1}&=K_t-	\eta \nabla_K{\cC}(K_t,L_t)\Sigma^{-1}_{K_t,L_t}\\
L_{t+1}&=L_t+	\eta \nabla_L{\cC}(K_t,L_t)\Sigma^{-1}_{K_t,L_t},
\end{flalign*} 
        \STATE Or
        \begin{flalign*}
{\rm \textbf{Gauss-Newton:}}\qquad ~~  	 K_{t+1}&=K_t-	\eta (R^u+B^\top P_{K_t,L_t}B)^{-1}\nabla_K{\cC}(K_t,L_t)\Sigma^{-1}_{K_t,L_t}\\
L_{t+1}&=L_t+	\eta(R^v-C^\top P_{K_t,L_t}C)^{-1} \nabla_L{\cC}(K_t,L_t)\Sigma^{-1}_{K_t,L_t},
\end{flalign*}
%		\ENDFOR		
		\ENDFOR
		\STATE Return the iterate $(K_T,L_T)$.	
		\end{algorithmic}
\end{algorithm}

\begin{figure*}[!t]
	\centering
	\begin{tabular}{ccc}
		\hskip-6pt\includegraphics[width=0.323\textwidth]{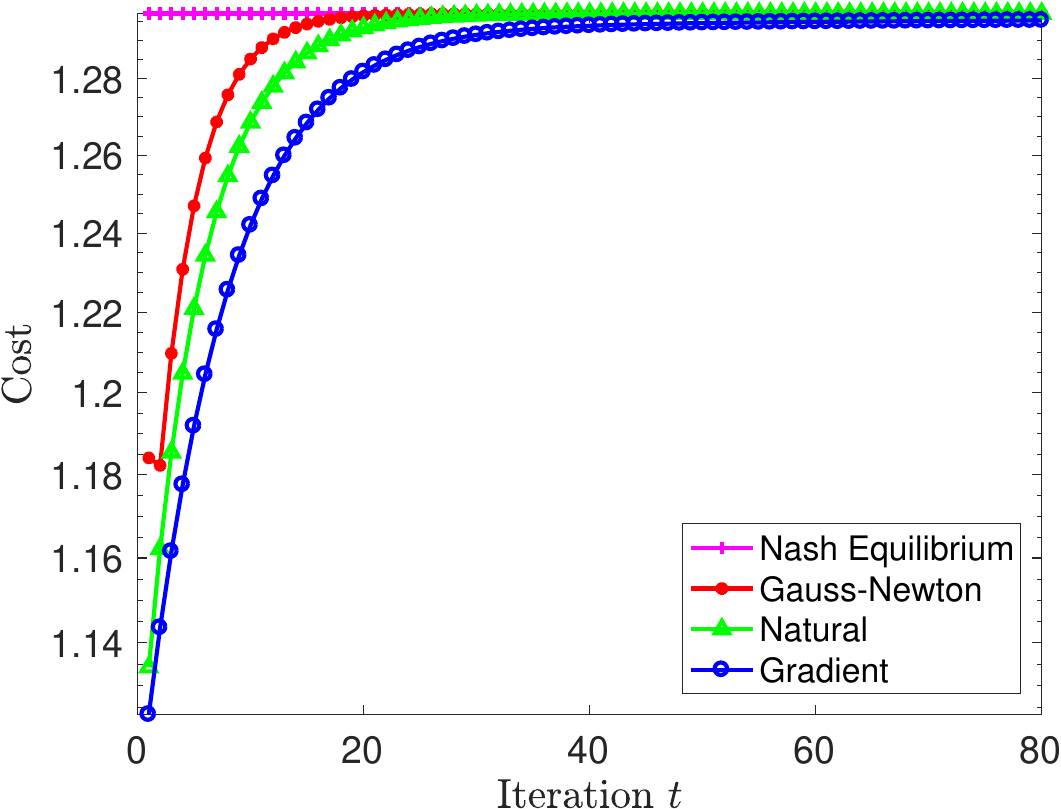}
		&
		\hskip-6pt\includegraphics[width=0.323\textwidth]{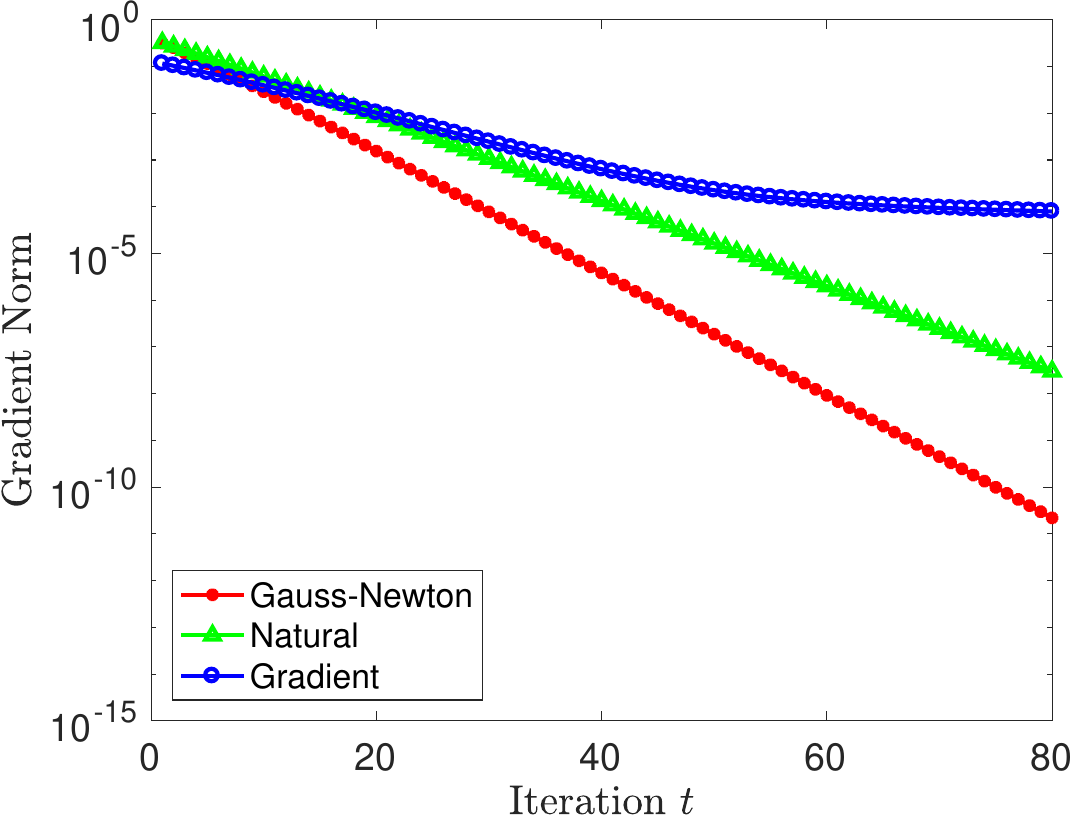}
		& 
		\hskip-6pt\includegraphics[width=0.323\textwidth]{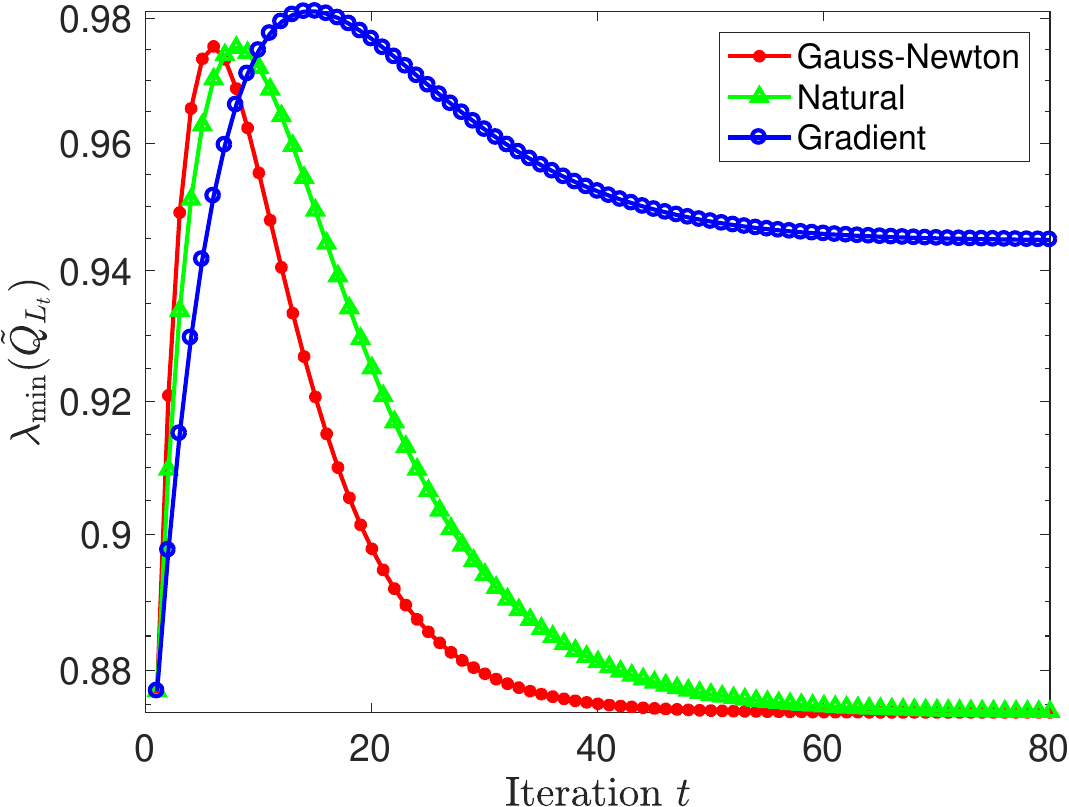}\\
		\hskip-8pt(a) $\cC(K(L),L)$ & \hskip 0pt(b) Grad. Mapp.   Norm Square  & \hskip2pt (c) $\lambda_{\min}(\tilde Q_L)$ 
	\end{tabular}
	\caption{Performance of the three AG methods  for {\bf Case $1$} where Assumption \ref{assum:invertibility} ii) is satisfied. (a) shows the monotone convergence of the expected cost $\cC(K(L),L)$ to the NE cost $\cC(K^*,L^*)$; (b) shows the convergence of the gradient mapping norm square; (c) shows the change of the smallest eigenvalue of $\tilde Q_L=Q-L^\top R^vL$.
%	 The algorithms are initialized randomly with stabilizing control pairs, and the stepsizes 
}
	\label{fig:case_1_AG}
\end{figure*}

\begin{figure*}[!t]
	\centering
	\begin{tabular}{ccc}
		\hskip-6pt\includegraphics[width=0.323\textwidth]{figs/alterna_grad_cost_2.pdf}
		&
		\hskip-6pt\includegraphics[width=0.323\textwidth]{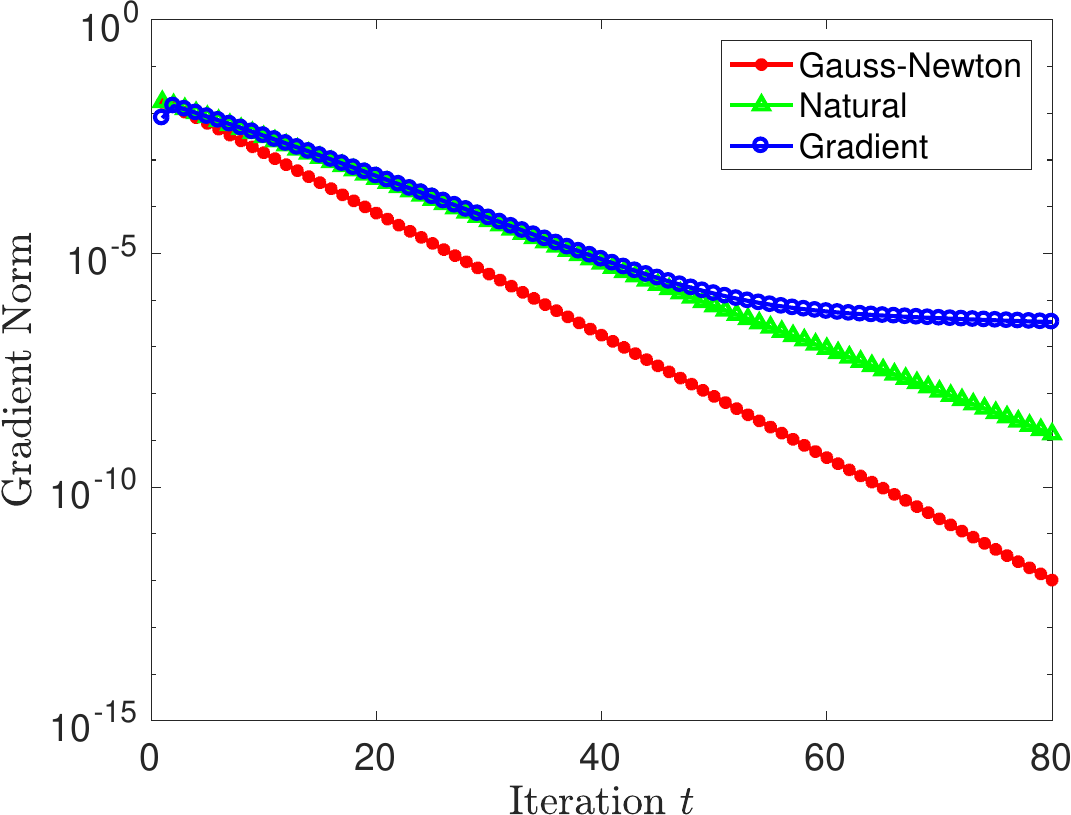}
		& 
		\hskip-6pt\includegraphics[width=0.323\textwidth]{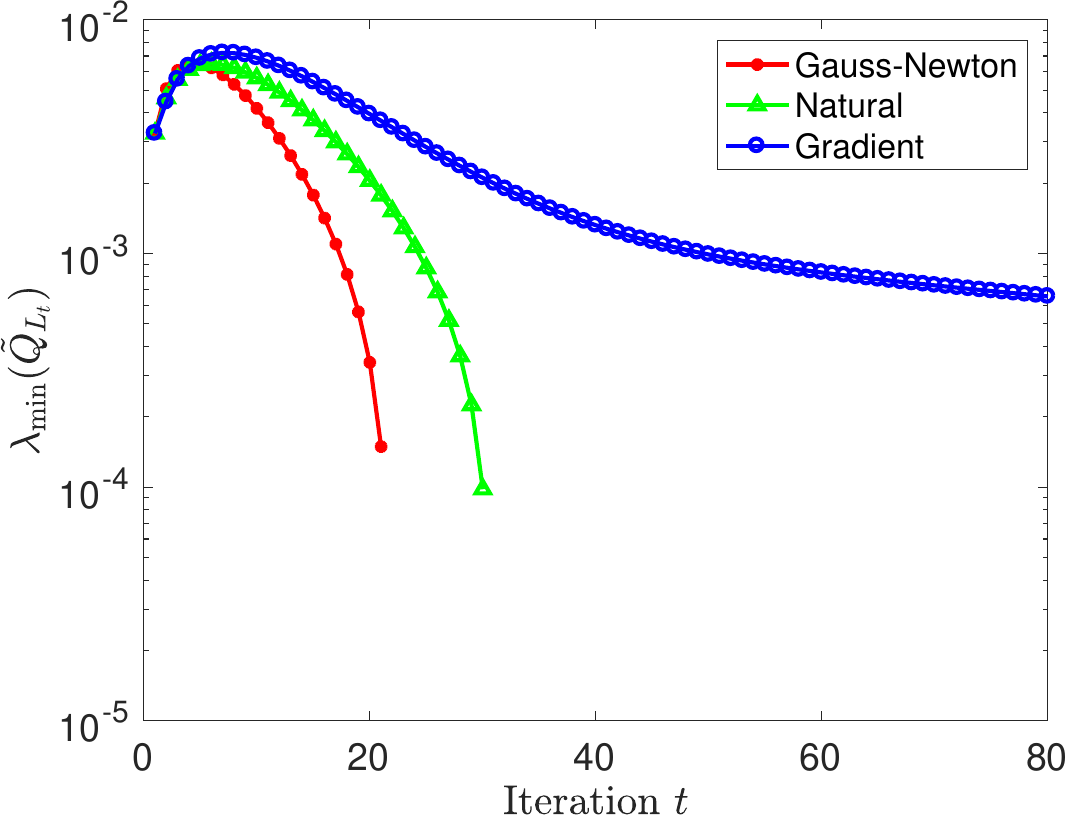}\\
		\hskip-8pt(a) $\cC(K(L),L)$ & \hskip 0pt(b) Grad. Mapp.   Norm Square  & \hskip2pt (c) $\lambda_{\min}(\tilde Q_L)$ 
	\end{tabular}
	\caption{Performance of the three AG  methods  for {\bf Case $2$} where Assumption \ref{assum:invertibility} ii) is not satisfied. (a) shows the monotone convergence of the expected cost $\cC(K(L),L)$ to the NE cost $\cC(K^*,L^*)$; (b) shows the convergence of the gradient mapping norm square; (c) shows the change of the smallest eigenvalue of $\tilde Q_L=Q-L^\top R^vL$.
%	 The algorithms are initialized randomly with stabilizing control pairs, and the stepsizes 
}
	\label{fig:case_2_AG}
\end{figure*}

\vspace{3pt}
{\noindent \bf Gradient-Descent-Ascent (GDA) Methods.} 
  
  Note that GDA and its variants with simultaneous updates have   drawn increasing attention recently for solving saddle-point problems  \citep{cherukuri2017saddle,daskalakis2018limit,mazumdar2019finding,jin2019minmax}, mainly due to their popularity in training GANs. The algorithms perform policy gradient descent for the minimizer and ascent for the maximizer. The updates are given in Algorithm \ref{alg:GDA}, whose performance is  showcased in Figures  \ref{fig:case_1_GDA} and \ref{fig:case_2_GDA}, showing their  convergence to the NE in both settings.

\begin{figure*}[!t]
	\centering
	\begin{tabular}{ccc}
		\hskip-6pt\includegraphics[width=0.323\textwidth]{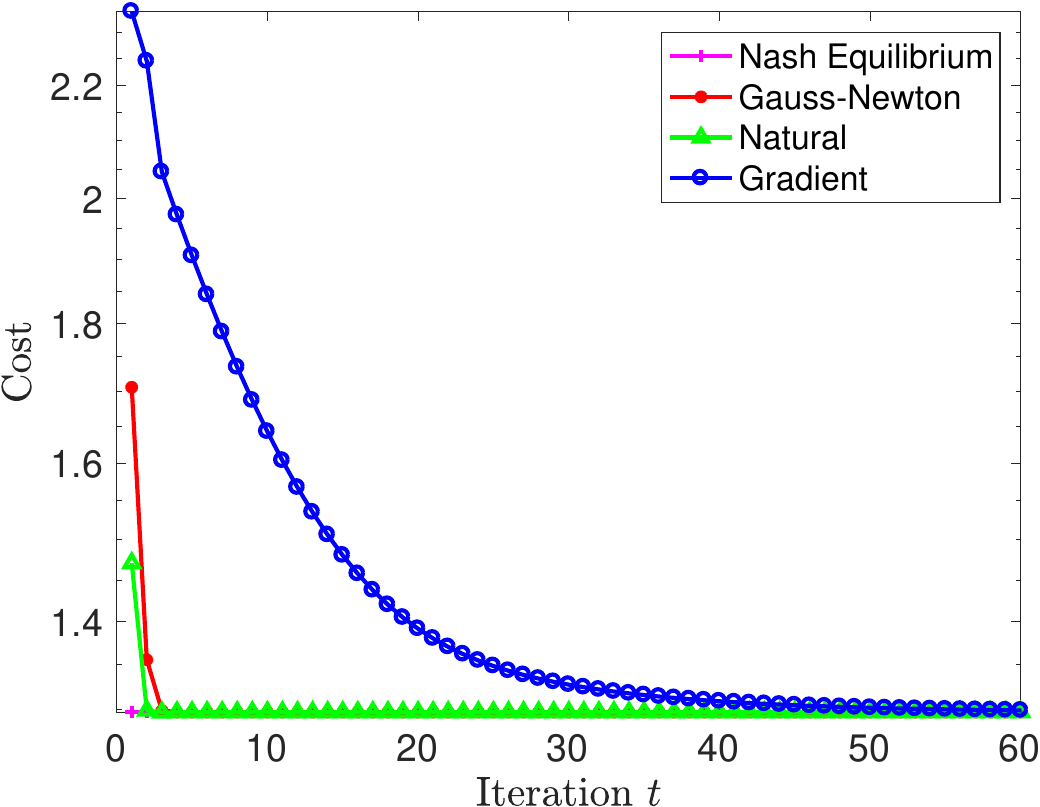}
		&
		\hskip-6pt\includegraphics[width=0.323\textwidth]{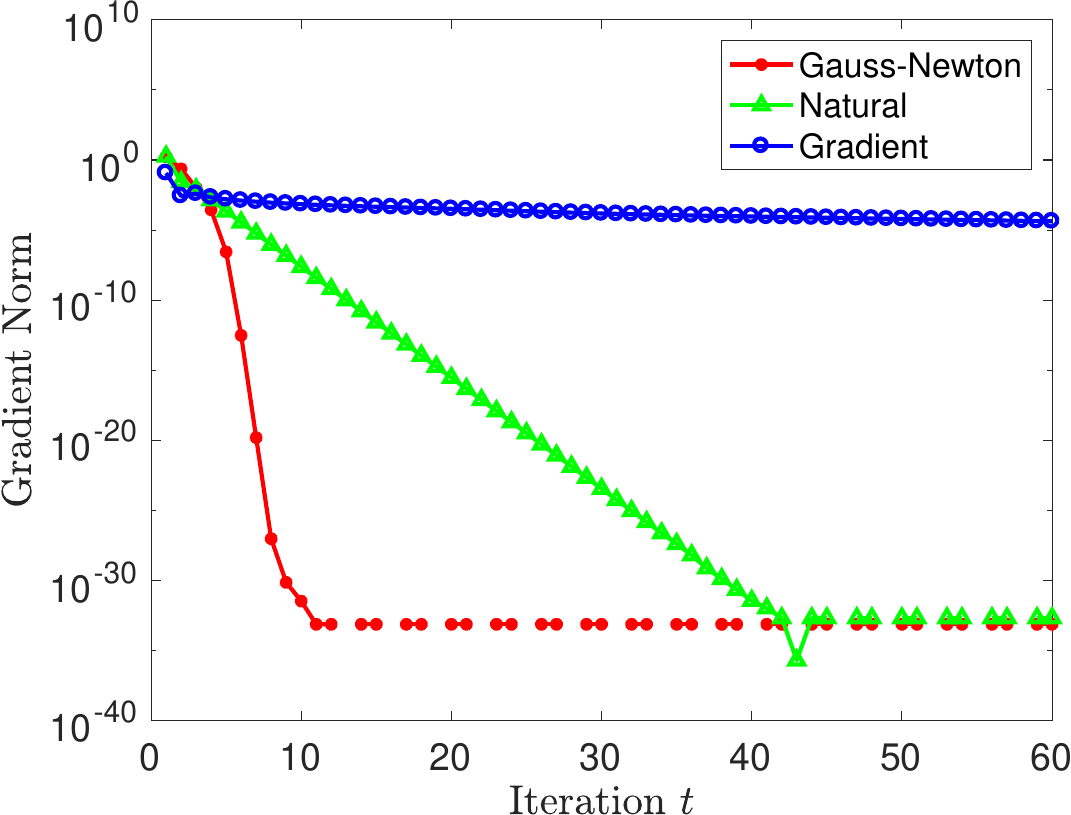}
		& 
		\hskip-6pt\includegraphics[width=0.323\textwidth]{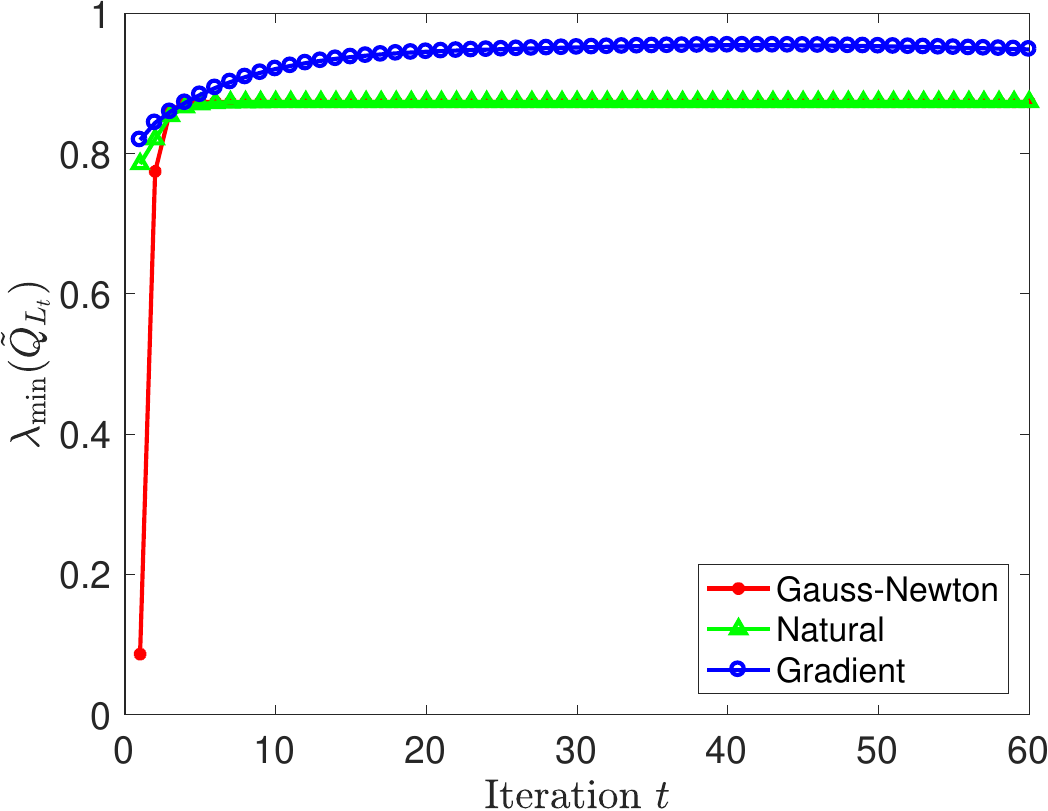}\\
		\hskip-8pt(a) $\cC(K(L),L)$ & \hskip 0pt(b) Grad. Mapp.   Norm Square  & \hskip2pt (c) $\lambda_{\min}(\tilde Q_L)$ 
	\end{tabular}
	\caption{Performance of the three GDA methods  for {\bf Case $1$} where Assumption \ref{assum:invertibility} ii) is satisfied. (a) shows the monotone convergence of the expected cost $\cC(K(L),L)$ to the NE cost $\cC(K^*,L^*)$; (b) shows the convergence of the gradient mapping norm square; (c) shows the change of the smallest eigenvalue of $\tilde Q_L=Q-L^\top R^vL$.
%	 The algorithms are initialized randomly with stabilizing control pairs, and the stepsizes 
}
	\label{fig:case_1_GDA}
\end{figure*}

\begin{figure*}[!t] 
	\centering
	\begin{tabular}{ccc}
		\hskip-6pt\includegraphics[width=0.323\textwidth]{figs/simultan_grad_cost_2.pdf}
		&
		\hskip-6pt\includegraphics[width=0.323\textwidth]{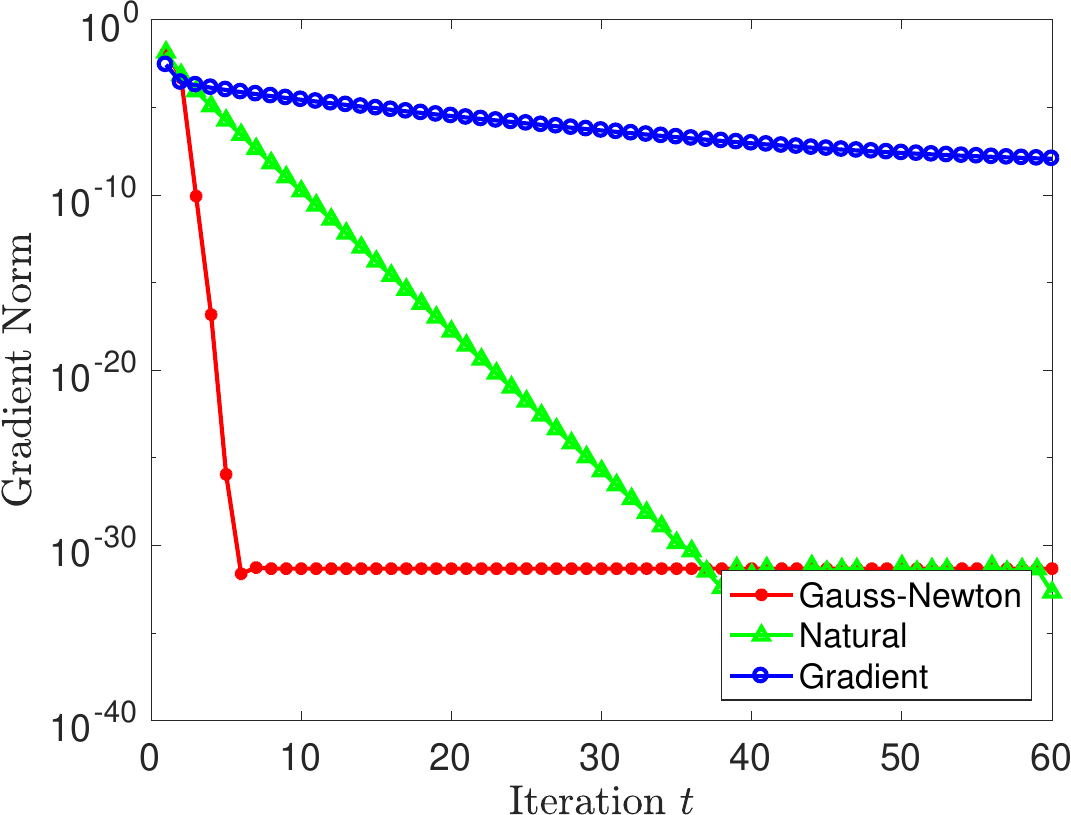}
		& 
		\hskip-6pt\includegraphics[width=0.323\textwidth]{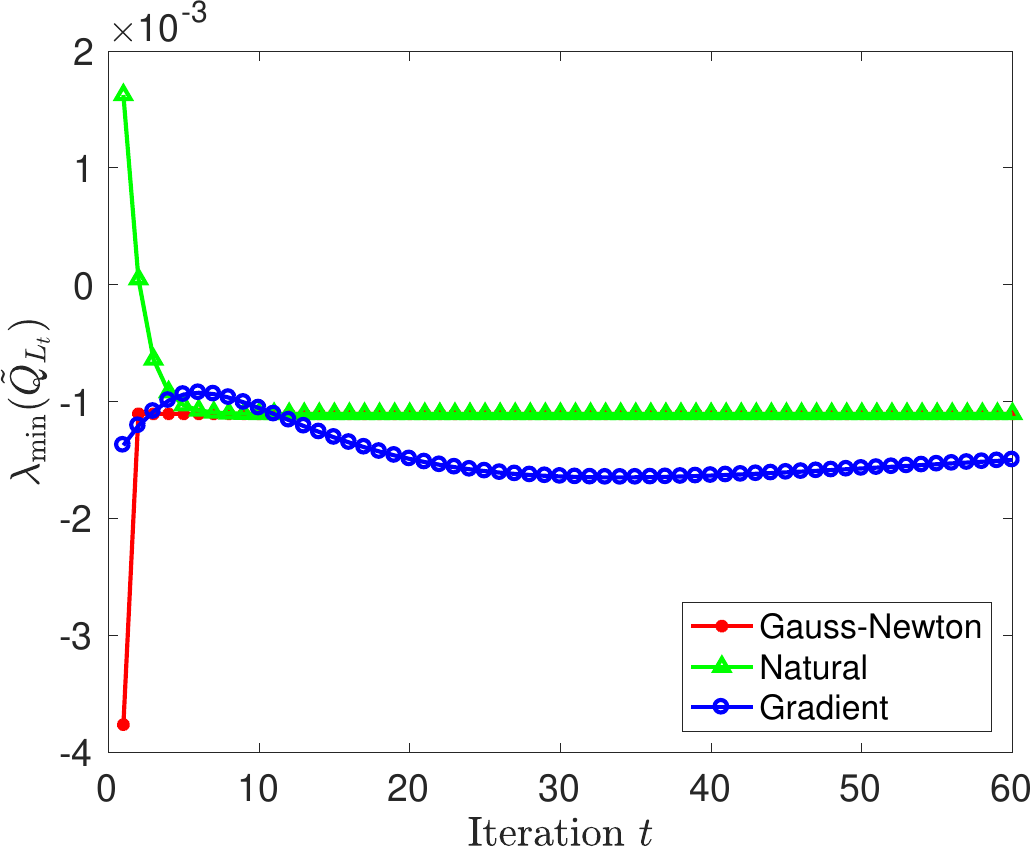}\\
		\hskip-8pt(a) $\cC(K(L),L)$ & \hskip 0pt(b) Grad. Mapp.   Norm Square  & \hskip2pt (c) $\lambda_{\min}(\tilde Q_L)$ 
	\end{tabular}
	\caption{Performance of the three GDA methods  for {\bf Case $2$} where Assumption \ref{assum:invertibility} ii) is not satisfied. (a) shows the monotone convergence of the expected cost $\cC(K(L),L)$ to the NE cost $\cC(K^*,L^*)$; (b) shows the convergence of the gradient mapping norm square; (c) shows the change of the smallest eigenvalue of $\tilde Q_L=Q-L^\top R^vL$.
%	 The algorithms are initialized randomly with stabilizing control pairs, and the stepsizes 
}
	\label{fig:case_2_GDA}
\end{figure*}

%% file: Minimax_LQR.bbl
\begin{thebibliography}{69}
\expandafter\ifx\csname natexlab\endcsname\relax\def\natexlab#1{#1}\fi
\expandafter\ifx\csname url\endcsname\relax
  \def\url#1{\texttt{#1}}\fi
\expandafter\ifx\csname urlprefix\endcsname\relax\def\urlprefix{}\fi

\bibitem[{Adolphs et~al.(2019)Adolphs, Daneshmand, Lucchi and
  Hofmann}]{adolphs2018local}
\textsc{Adolphs, L.}, \textsc{Daneshmand, H.}, \textsc{Lucchi, A.} and
  \textsc{Hofmann, T.} (2019).
\newblock Local saddle point optimization: {A} curvature exploitation approach.

\bibitem[{Al-Tamimi et~al.(2007)Al-Tamimi, Lewis and Abu-Khalaf}]{al2007model}
\textsc{Al-Tamimi, A.}, \textsc{Lewis, F.~L.} and \textsc{Abu-Khalaf, M.}
  (2007).
\newblock Model-free {Q}-learning designs for linear discrete-time zero-sum
  games with application to {$\cH$}-infinity control.
\newblock \textit{Automatica}, \textbf{43} 473--481.

\bibitem[{Balduzzi et~al.(2018)Balduzzi, Racaniere, Martens, Foerster, Tuyls
  and Graepel}]{balduzzi2018mechanics}
\textsc{Balduzzi, D.}, \textsc{Racaniere, S.}, \textsc{Martens, J.},
  \textsc{Foerster, J.}, \textsc{Tuyls, K.} and \textsc{Graepel, T.} (2018).
\newblock The mechanics of n-player differentiable games.
\newblock In \textit{International Conference on Machine Learning}.

\bibitem[{Banerjee and Peng(2003)}]{banerjee2003adaptive}
\textsc{Banerjee, B.} and \textsc{Peng, J.} (2003).
\newblock Adaptive policy gradient in multiagent learning.
\newblock In \textit{Conference on Autonomous Agents and Multiagent Systems}.
  ACM.

\bibitem[{Ba{\c{s}}ar and Bernhard(2008)}]{bacsar2008h}
\textsc{Ba{\c{s}}ar, T.} and \textsc{Bernhard, P.} (2008).
\newblock \textit{$\mathcal{H}_\infty$ Optimal Control and Related Minimax
  Design Problems: A Dynamic Game Approach}.
\newblock Springer Science \& Business Media.

\bibitem[{Bertsekas(2005)}]{bertsekas2005dynamic}
\textsc{Bertsekas, D.~P.} (2005).
\newblock \textit{Dynamic Programming and Optimal Control}, vol.~1.
\newblock Athena Scientific Belmont, MA.

\bibitem[{Bowling and Veloso(2001)}]{bowlingrational}
\textsc{Bowling, M.} and \textsc{Veloso, M.} (2001).
\newblock Rational and convergent learning in stochastic games.
\newblock In \textit{International Joint Conference on Artificial
  Intelligence}, vol.~17.

\bibitem[{Cartis et~al.(2010)Cartis, Gould and Toint}]{cartis2010complexity}
\textsc{Cartis, C.}, \textsc{Gould, N.~I.} and \textsc{Toint, P.~L.} (2010).
\newblock On the complexity of steepest descent, {N}ewton's and regularized
  {N}ewton's methods for nonconvex unconstrained optimization problems.
\newblock \textit{SIAM Journal on Optimization}, \textbf{20} 2833--2852.

\bibitem[{Cartis et~al.(2017)Cartis, Gould and Toint}]{cartis2017worst}
\textsc{Cartis, C.}, \textsc{Gould, N.~I.} and \textsc{Toint, P.~L.} (2017).
\newblock Worst-case evaluation complexity and optimality of second-order
  methods for nonconvex smooth optimization.
\newblock \textit{arXiv preprint arXiv:1709.07180}.

\bibitem[{Chen et~al.(2017)Chen, Lucier, Singer and Syrgkanis}]{chen2017robust}
\textsc{Chen, R.~S.}, \textsc{Lucier, B.}, \textsc{Singer, Y.} and
  \textsc{Syrgkanis, V.} (2017).
\newblock Robust optimization for non-convex objectives.
\newblock In \textit{Advances in Neural Information Processing Systems}.

\bibitem[{Cherukuri et~al.(2017)Cherukuri, Gharesifard and
  Cortes}]{cherukuri2017saddle}
\textsc{Cherukuri, A.}, \textsc{Gharesifard, B.} and \textsc{Cortes, J.}
  (2017).
\newblock Saddle-point dynamics: {C}onditions for asymptotic stability of
  saddle points.
\newblock \textit{SIAM Journal on Control and Optimization}, \textbf{55}
  486--511.

\bibitem[{Conitzer and Sandholm(2007)}]{conitzer2007awesome}
\textsc{Conitzer, V.} and \textsc{Sandholm, T.} (2007).
\newblock Awesome: A general multiagent learning algorithm that converges in
  self-play and learns a best response against stationary opponents.
\newblock \textit{Machine Learning}, \textbf{67} 23--43.

\bibitem[{Daskalakis and Panageas(2018)}]{daskalakis2018limit}
\textsc{Daskalakis, C.} and \textsc{Panageas, I.} (2018).
\newblock The limit points of (optimistic) gradient descent in min-max
  optimization.
\newblock In \textit{Advances in Neural Information Processing Systems}.

\bibitem[{Fazel et~al.(2018)Fazel, Ge, Kakade and Mesbahi}]{fazel2018global}
\textsc{Fazel, M.}, \textsc{Ge, R.}, \textsc{Kakade, S.} and \textsc{Mesbahi,
  M.} (2018).
\newblock Global convergence of policy gradient methods for the linear
  quadratic regulator 1467--1476.

\bibitem[{Graham(2018)}]{graham2018kronecker}
\textsc{Graham, A.} (2018).
\newblock \textit{Kronecker Products and Matrix Calculus with Applications}.
\newblock Courier Dover Publications.

\bibitem[{Grnarova et~al.(2017)Grnarova, Levy, Lucchi, Hofmann and
  Krause}]{grnarova2017online}
\textsc{Grnarova, P.}, \textsc{Levy, K.~Y.}, \textsc{Lucchi, A.},
  \textsc{Hofmann, T.} and \textsc{Krause, A.} (2017).
\newblock An online learning approach to generative adversarial networks.
\newblock \textit{arXiv preprint arXiv:1706.03269}.

\bibitem[{Hernandez-Leal et~al.(2017)Hernandez-Leal, Kaisers, Baarslag and
  de~Cote}]{hernandez2017survey}
\textsc{Hernandez-Leal, P.}, \textsc{Kaisers, M.}, \textsc{Baarslag, T.} and
  \textsc{de~Cote, E.~M.} (2017).
\newblock A survey of learning in multiagent environments: Dealing with
  non-stationarity.
\newblock \textit{arXiv preprint arXiv:1707.09183}.

\bibitem[{Heusel et~al.(2017)Heusel, Ramsauer, Unterthiner, Nessler and
  Hochreiter}]{heusel2017gans}
\textsc{Heusel, M.}, \textsc{Ramsauer, H.}, \textsc{Unterthiner, T.},
  \textsc{Nessler, B.} and \textsc{Hochreiter, S.} (2017).
\newblock {GANs} trained by a two time-scale update rule converge to a local
  {N}ash equilibrium.
\newblock In \textit{Advances in Neural Information Processing Systems}.

\bibitem[{Hu and Wellman(2003)}]{hu2003nash}
\textsc{Hu, J.} and \textsc{Wellman, M.~P.} (2003).
\newblock Nash \relax{Q}-learning for general-sum stochastic games.
\newblock \textit{Journal of Machine Learning Research}, \textbf{4} 1039--1069.

\bibitem[{Jacobson(1973)}]{jacobson1973optimal}
\textsc{Jacobson, D.} (1973).
\newblock Optimal stochastic linear systems with exponential performance
  criteria and their relation to deterministic differential games.
\newblock \textit{IEEE Transactions on Automatic control}, \textbf{18}
  124--131.

\bibitem[{Jacobson(1977)}]{jacobson1977values}
\textsc{Jacobson, D.} (1977).
\newblock On values and strategies for infinite-time linear quadratic games.
\newblock \textit{IEEE Transactions on Automatic Control}, \textbf{22}
  490--491.

\bibitem[{Jin et~al.(2019)Jin, Netrapalli and Jordan}]{jin2019minmax}
\textsc{Jin, C.}, \textsc{Netrapalli, P.} and \textsc{Jordan, M.~I.} (2019).
\newblock Minmax optimization: {S}table limit points of gradient descent ascent
  are locally optimal.
\newblock \textit{arXiv preprint arXiv:1902.00618}.

\bibitem[{Kakade(2002)}]{kakade2002natural}
\textsc{Kakade, S.~M.} (2002).
\newblock A natural policy gradient.
\newblock In \textit{Advances in Neural Information Processing Systems}.

\bibitem[{Khamaru and Wainwright(2018)}]{khamaru2018convergence}
\textsc{Khamaru, K.} and \textsc{Wainwright, M.~J.} (2018).
\newblock Convergence guarantees for a class of non-convex and non-smooth
  optimization problems.
\newblock \textit{arXiv preprint arXiv:1804.09629}.

\bibitem[{Konda and Tsitsiklis(2000)}]{konda2000actor}
\textsc{Konda, V.~R.} and \textsc{Tsitsiklis, J.~N.} (2000).
\newblock Actor-critic algorithms.
\newblock In \textit{Advances in Neural Information Processing Systems}.

\bibitem[{Konstantinov et~al.(1993)Konstantinov, Petkov and
  Christov}]{konstantinov1993perturbation}
\textsc{Konstantinov, M.~M.}, \textsc{Petkov, P.~H.} and \textsc{Christov,
  N.~D.} (1993).
\newblock Perturbation analysis of the discrete {R}iccati equation.
\newblock \textit{Kybernetika}, \textbf{29} 18--29.

\bibitem[{Krantz and Parks(2012)}]{krantz2012implicit}
\textsc{Krantz, S.~G.} and \textsc{Parks, H.~R.} (2012).
\newblock \textit{The Implicit Function Theorem: History, Theory, and
  Applications}.
\newblock Springer Science \& Business Media.

\bibitem[{Kwakernaak and Sivan(1972)}]{kwakernaak1972linear}
\textsc{Kwakernaak, H.} and \textsc{Sivan, R.} (1972).
\newblock \textit{Linear Optimal Control Systems}, vol.~1.
\newblock Wiley-Interscience New York.

\bibitem[{Lagoudakis and Parr(2002)}]{lagoudakis2002value}
\textsc{Lagoudakis, M.~G.} and \textsc{Parr, R.} (2002).
\newblock Value function approximation in zero-sum {M}arkov games.
\newblock In \textit{Conference on Uncertainty in Artificial Intelligence}.

\bibitem[{Lin et~al.(2018)Lin, Liu, Rafique and Yang}]{lin2018solving}
\textsc{Lin, Q.}, \textsc{Liu, M.}, \textsc{Rafique, H.} and \textsc{Yang, T.}
  (2018).
\newblock Solving weakly-convex-weakly-concave saddle-point problems as
  weakly-monotone variational inequality.
\newblock \textit{arXiv preprint arXiv:1810.10207}.

\bibitem[{Littman(1994)}]{littman1994markov}
\textsc{Littman, M.~L.} (1994).
\newblock Markov games as a framework for multi-agent reinforcement learning.
\newblock In \textit{International Conference on Machine Learning}.

\bibitem[{Lu et~al.(2018)Lu, Singh, Chen, Chen and Hong}]{lu2018understand}
\textsc{Lu, S.}, \textsc{Singh, R.}, \textsc{Chen, X.}, \textsc{Chen, Y.} and
  \textsc{Hong, M.} (2018).
\newblock Understand the dynamics of {GANs} via primal-dual optimization.

\bibitem[{Magnus and Neudecker(1985)}]{magnus1985matrix}
\textsc{Magnus, J.~R.} and \textsc{Neudecker, H.} (1985).
\newblock Matrix differential calculus with applications to simple, {H}adamard,
  and {K}ronecker products.
\newblock \textit{Journal of Mathematical Psychology}, \textbf{29} 474--492.

\bibitem[{Malik et~al.(2018)Malik, Pananjady, Bhatia, Khamaru, Bartlett and
  Wainwright}]{malik2018derivative}
\textsc{Malik, D.}, \textsc{Pananjady, A.}, \textsc{Bhatia, K.},
  \textsc{Khamaru, K.}, \textsc{Bartlett, P.~L.} and \textsc{Wainwright, M.~J.}
  (2018).
\newblock Derivative-free methods for policy optimization: Guarantees for
  linear quadratic systems.
\newblock \textit{arXiv preprint arXiv:1812.08305}.

\bibitem[{Mazumdar and Ratliff(2018)}]{mazumdar2018convergence}
\textsc{Mazumdar, E.} and \textsc{Ratliff, L.~J.} (2018).
\newblock On the convergence of competitive, multi-agent gradient-based
  learning.
\newblock \textit{arXiv preprint arXiv:1804.05464}.

\bibitem[{Mazumdar et~al.(2019)Mazumdar, Jordan and
  Sastry}]{mazumdar2019finding}
\textsc{Mazumdar, E.~V.}, \textsc{Jordan, M.~I.} and \textsc{Sastry, S.~S.}
  (2019).
\newblock On finding local {N}ash equilibria (and only local {N}ash equilibria)
  in zero-sum games.
\newblock \textit{arXiv preprint arXiv:1901.00838}.

\bibitem[{Mertikopoulos et~al.(2019)Mertikopoulos, Zenati, Lecouat, Foo,
  Chandrasekhar and Piliouras}]{mertikopoulos2019optimistic}
\textsc{Mertikopoulos, P.}, \textsc{Zenati, H.}, \textsc{Lecouat, B.},
  \textsc{Foo, C.-S.}, \textsc{Chandrasekhar, V.} and \textsc{Piliouras, G.}
  (2019).
\newblock Optimistic mirror descent in saddle-point problems: {G}oing the extra
  (gradient) mile.
\newblock In \textit{International Conference on Learning Representations}.

\bibitem[{Mnih et~al.(2016)Mnih, Badia, Mirza, Graves, Lillicrap, Harley,
  Silver and Kavukcuoglu}]{mnih2016asynchronous}
\textsc{Mnih, V.}, \textsc{Badia, A.~P.}, \textsc{Mirza, M.}, \textsc{Graves,
  A.}, \textsc{Lillicrap, T.}, \textsc{Harley, T.}, \textsc{Silver, D.} and
  \textsc{Kavukcuoglu, K.} (2016).
\newblock Asynchronous methods for deep reinforcement learning.
\newblock In \textit{International conference on machine learning}.

\bibitem[{Murty and Kabadi(1987)}]{murty1987some}
\textsc{Murty, K.~G.} and \textsc{Kabadi, S.~N.} (1987).
\newblock Some {NP}-complete problems in quadratic and nonlinear programming.
\newblock \textit{Mathematical Programming}, \textbf{39} 117--129.

\bibitem[{Nagarajan and Kolter(2017)}]{nagarajan2017gradient}
\textsc{Nagarajan, V.} and \textsc{Kolter, J.~Z.} (2017).
\newblock Gradient descent {GAN} optimization is locally stable.
\newblock In \textit{Advances in Neural Information Processing Systems}.

\bibitem[{Nesterov(2013)}]{nesterov2013introductory}
\textsc{Nesterov, Y.} (2013).
\newblock \textit{Introductory Lectures on Convex Optimization: {A} Basic
  Course}, vol.~87.
\newblock Springer Science \& Business Media.

\bibitem[{Nouiehed et~al.(2019)Nouiehed, Sanjabi, Lee and
  Razaviyayn}]{nouiehed2019solving}
\textsc{Nouiehed, M.}, \textsc{Sanjabi, M.}, \textsc{Lee, J.~D.} and
  \textsc{Razaviyayn, M.} (2019).
\newblock Solving a class of non-convex min-max games using iterative first
  order methods.
\newblock \textit{arXiv preprint arXiv:1902.08297}.

\bibitem[{O'Donoghue et~al.(2016)O'Donoghue, Munos, Kavukcuoglu and
  Mnih}]{o2016combining}
\textsc{O'Donoghue, B.}, \textsc{Munos, R.}, \textsc{Kavukcuoglu, K.} and
  \textsc{Mnih, V.} (2016).
\newblock Combining policy gradient and {Q}-learning.
\newblock \textit{arXiv preprint arXiv:1611.01626}.

\bibitem[{OpenAI(2018)}]{OpenAI_dota}
\textsc{OpenAI} (2018).
\newblock Openai five.
\newblock \url{https://blog.openai.com/openai-five/}.

\bibitem[{Papini et~al.(2018)Papini, Binaghi, Canonaco, Pirotta and
  Restelli}]{papini2018stochastic}
\textsc{Papini, M.}, \textsc{Binaghi, D.}, \textsc{Canonaco, G.},
  \textsc{Pirotta, M.} and \textsc{Restelli, M.} (2018).
\newblock Stochastic variance-reduced policy gradient.
\newblock \textit{arXiv preprint arXiv:1806.05618}.

\bibitem[{P{\'e}rolat et~al.(2018)P{\'e}rolat, Piot and
  Pietquin}]{perolat2018actor}
\textsc{P{\'e}rolat, J.}, \textsc{Piot, B.} and \textsc{Pietquin, O.} (2018).
\newblock Actor-critic fictitious play in simultaneous move multistage games.
\newblock In \textit{International Conference on Artificial Intelligence and
  Statistics}.

\bibitem[{P{\'e}rolat et~al.(2016)P{\'e}rolat, Piot, Scherrer and
  Pietquin}]{perolat2016use}
\textsc{P{\'e}rolat, J.}, \textsc{Piot, B.}, \textsc{Scherrer, B.} and
  \textsc{Pietquin, O.} (2016).
\newblock On the use of non-stationary strategies for solving two-player
  zero-sum markov games.
\newblock In \textit{Conference on Artificial Intelligence and Statistics}.

\bibitem[{Pinto et~al.(2017)Pinto, Davidson, Sukthankar and
  Gupta}]{pinto2017robust}
\textsc{Pinto, L.}, \textsc{Davidson, J.}, \textsc{Sukthankar, R.} and
  \textsc{Gupta, A.} (2017).
\newblock Robust adversarial reinforcement learning.
\newblock In \textit{International Conference on Machine Learning}.

\bibitem[{Rafique et~al.(2018)Rafique, Liu, Lin and Yang}]{rafique2018non}
\textsc{Rafique, H.}, \textsc{Liu, M.}, \textsc{Lin, Q.} and \textsc{Yang, T.}
  (2018).
\newblock Non-convex min-max optimization: {P}rovable algorithms and
  applications in machine learning.
\newblock \textit{arXiv preprint arXiv:1810.02060}.

\bibitem[{Sanjabi et~al.(2018)Sanjabi, Razaviyayn and Lee}]{sanjabi2018solving}
\textsc{Sanjabi, M.}, \textsc{Razaviyayn, M.} and \textsc{Lee, J.~D.} (2018).
\newblock Solving non-convex non-concave min-max games under
  {P}olyak-{$\L$}ojasiewicz condition.
\newblock \textit{arXiv preprint arXiv:1812.02878}.

\bibitem[{Schulman et~al.(2015)Schulman, Levine, Abbeel, Jordan and
  Moritz}]{schulman2015trust}
\textsc{Schulman, J.}, \textsc{Levine, S.}, \textsc{Abbeel, P.},
  \textsc{Jordan, M.} and \textsc{Moritz, P.} (2015).
\newblock Trust region policy optimization.
\newblock In \textit{International Conference on Machine Learning}.

\bibitem[{Schulman et~al.(2017)Schulman, Wolski, Dhariwal, Radford and
  Klimov}]{schulman2017proximal}
\textsc{Schulman, J.}, \textsc{Wolski, F.}, \textsc{Dhariwal, P.},
  \textsc{Radford, A.} and \textsc{Klimov, O.} (2017).
\newblock Proximal policy optimization algorithms.
\newblock \textit{arXiv preprint arXiv:1707.06347}.

\bibitem[{Silver et~al.(2016)Silver, Huang, Maddison, Guez, Sifre, Van
  Den~Driessche, Schrittwieser, Antonoglou, Panneershelvam, Lanctot
  et~al.}]{silver2016mastering}
\textsc{Silver, D.}, \textsc{Huang, A.}, \textsc{Maddison, C.~J.},
  \textsc{Guez, A.}, \textsc{Sifre, L.}, \textsc{Van Den~Driessche, G.},
  \textsc{Schrittwieser, J.}, \textsc{Antonoglou, I.}, \textsc{Panneershelvam,
  V.}, \textsc{Lanctot, M.} \textsc{et~al.} (2016).
\newblock Mastering the game of \relax{G}o with deep neural networks and tree
  search.
\newblock \textit{Nature}, \textbf{529} 484--489.

\bibitem[{Silver et~al.(2017)Silver, Schrittwieser, Simonyan, Antonoglou,
  Huang, Guez, Hubert, Baker, Lai, Bolton et~al.}]{silver2017mastering}
\textsc{Silver, D.}, \textsc{Schrittwieser, J.}, \textsc{Simonyan, K.},
  \textsc{Antonoglou, I.}, \textsc{Huang, A.}, \textsc{Guez, A.},
  \textsc{Hubert, T.}, \textsc{Baker, L.}, \textsc{Lai, M.}, \textsc{Bolton,
  A.} \textsc{et~al.} (2017).
\newblock Mastering the game of \relax{G}o without human knowledge.
\newblock \textit{Nature}, \textbf{550} 354--359.

\bibitem[{Srinivasan et~al.(2018)Srinivasan, Lanctot, Zambaldi, P{\'e}rolat,
  Tuyls, Munos and Bowling}]{srinivasan2018actor}
\textsc{Srinivasan, S.}, \textsc{Lanctot, M.}, \textsc{Zambaldi, V.},
  \textsc{P{\'e}rolat, J.}, \textsc{Tuyls, K.}, \textsc{Munos, R.} and
  \textsc{Bowling, M.} (2018).
\newblock Actor-critic policy optimization in partially observable multiagent
  environments.
\newblock In \textit{Advances in Neural Information Processing Systems}.

\bibitem[{Stoorvogel and Weeren(1994)}]{stoorvogel1994discrete}
\textsc{Stoorvogel, A.~A.} and \textsc{Weeren, A.~J.} (1994).
\newblock The discrete-time {R}iccati equation related to the
  $\mathcal{H}_{\infty}$ control problem.
\newblock \textit{IEEE Transactions on Automatic Control}, \textbf{39}
  686--691.

\bibitem[{Sun(1998)}]{sun1998perturbation}
\textsc{Sun, J.-G.} (1998).
\newblock Perturbation theory for algebraic {R}iccati equations.
\newblock \textit{SIAM Journal on Matrix Analysis and Applications},
  \textbf{19} 39--65.

\bibitem[{Sutton and Barto(2018)}]{sutton2018reinforcement}
\textsc{Sutton, R.~S.} and \textsc{Barto, A.~G.} (2018).
\newblock \textit{Reinforcement Learning: An Introduction}.
\newblock MIT press.

\bibitem[{Sutton et~al.(2000)Sutton, McAllester, Singh and
  Mansour}]{sutton2000policy}
\textsc{Sutton, R.~S.}, \textsc{McAllester, D.~A.}, \textsc{Singh, S.~P.} and
  \textsc{Mansour, Y.} (2000).
\newblock Policy gradient methods for reinforcement learning with function
  approximation.
\newblock In \textit{Advances in Neural Information Processing Systems}.

\bibitem[{Tu and Recht(2018)}]{tu2018gap}
\textsc{Tu, S.} and \textsc{Recht, B.} (2018).
\newblock The gap between model-based and model-free methods on the linear
  quadratic regulator: An asymptotic viewpoint.
\newblock \textit{arXiv preprint arXiv:1812.03565}.

\bibitem[{Tyrtyshnikov(2012)}]{tyrtyshnikov2012brief}
\textsc{Tyrtyshnikov, E.~E.} (2012).
\newblock \textit{A Brief Introduction to Numerical Analysis}.
\newblock Springer Science \& Business Media.

\bibitem[{Vinyals et~al.(2019)Vinyals, Babuschkin, Chung, Mathieu, Jaderberg,
  Czarnecki, Dudzik, Huang, Georgiev, Powell, Ewalds, Horgan, Kroiss,
  Danihelka, Agapiou, Oh, Dalibard, Choi, Sifre, Sulsky, Vezhnevets, Molloy,
  Cai, Budden, Paine, Gulcehre, Wang, Pfaff, Pohlen, Wu, Yogatama, Cohen,
  McKinney, Smith, Schaul, Lillicrap, Apps, Kavukcuoglu, Hassabis and
  Silver}]{alphastarblog}
\textsc{Vinyals, O.}, \textsc{Babuschkin, I.}, \textsc{Chung, J.},
  \textsc{Mathieu, M.}, \textsc{Jaderberg, M.}, \textsc{Czarnecki, W.~M.},
  \textsc{Dudzik, A.}, \textsc{Huang, A.}, \textsc{Georgiev, P.},
  \textsc{Powell, R.}, \textsc{Ewalds, T.}, \textsc{Horgan, D.},
  \textsc{Kroiss, M.}, \textsc{Danihelka, I.}, \textsc{Agapiou, J.},
  \textsc{Oh, J.}, \textsc{Dalibard, V.}, \textsc{Choi, D.}, \textsc{Sifre,
  L.}, \textsc{Sulsky, Y.}, \textsc{Vezhnevets, S.}, \textsc{Molloy, J.},
  \textsc{Cai, T.}, \textsc{Budden, D.}, \textsc{Paine, T.}, \textsc{Gulcehre,
  C.}, \textsc{Wang, Z.}, \textsc{Pfaff, T.}, \textsc{Pohlen, T.}, \textsc{Wu,
  Y.}, \textsc{Yogatama, D.}, \textsc{Cohen, J.}, \textsc{McKinney, K.},
  \textsc{Smith, O.}, \textsc{Schaul, T.}, \textsc{Lillicrap, T.},
  \textsc{Apps, C.}, \textsc{Kavukcuoglu, K.}, \textsc{Hassabis, D.} and
  \textsc{Silver, D.} (2019).
\newblock {AlphaStar: Mastering the Real-Time Strategy Game StarCraft II}.
\newblock
  \url{https://deepmind.com/blog/alphastar-mastering-real-time-strategy-game-starcraft-ii/}.

\bibitem[{Whittle(1981)}]{whittle1981risk}
\textsc{Whittle, P.} (1981).
\newblock Risk-sensitive linear/quadratic/{G}aussian control.
\newblock \textit{Advances in Applied Probability}, \textbf{13} 764--777.

\bibitem[{Zhang et~al.(2019{\natexlab{a}})Zhang, Hu and
  Ba\c{s}ar}]{zhang2018policymixed}
\textsc{Zhang, K.}, \textsc{Hu, B.} and \textsc{Ba\c{s}ar, T.}
  (2019{\natexlab{a}}).
\newblock Policy optimization for $\mathcal{H}_2$ linear control with
  $\mathcal{H}_\infty$ robustness guarantee: {I}mplicit regularization and
  global convergence.
\newblock \textit{arXiv preprint arXiv:1910.09496}.

\bibitem[{Zhang et~al.(2019{\natexlab{b}})Zhang, Koppel, Zhu and
  Ba{\c{s}}ar}]{zhang2019policy}
\textsc{Zhang, K.}, \textsc{Koppel, A.}, \textsc{Zhu, H.} and
  \textsc{Ba{\c{s}}ar, T.} (2019{\natexlab{b}}).
\newblock Global convergence of policy gradient methods to (almost) locally
  optimal policies.
\newblock \textit{arXiv preprint arXiv:1906.08383}.

\bibitem[{Zhang et~al.(2018{\natexlab{a}})Zhang, Yang and
  Ba\c{s}ar}]{zhang18cdc}
\textsc{Zhang, K.}, \textsc{Yang, Z.} and \textsc{Ba\c{s}ar, T.}
  (2018{\natexlab{a}}).
\newblock Networked multi-agent reinforcement learning in continuous spaces.
\newblock In \textit{Proceedings of the 57th IEEE Conference on Decision and
  Control}.

\bibitem[{Zhang et~al.(2018{\natexlab{b}})Zhang, Yang, Liu, Zhang and
  Ba\c{s}ar}]{zhang2018fully}
\textsc{Zhang, K.}, \textsc{Yang, Z.}, \textsc{Liu, H.}, \textsc{Zhang, T.} and
  \textsc{Ba\c{s}ar, T.} (2018{\natexlab{b}}).
\newblock Fully decentralized multi-agent reinforcement learning with networked
  agents.
\newblock In \textit{International Conference on Machine Learning}.

\bibitem[{Zhang et~al.(2018{\natexlab{c}})Zhang, Yang, Liu, Zhang and
  Ba{\c{s}}ar}]{zhang2018finite}
\textsc{Zhang, K.}, \textsc{Yang, Z.}, \textsc{Liu, H.}, \textsc{Zhang, T.} and
  \textsc{Ba{\c{s}}ar, T.} (2018{\natexlab{c}}).
\newblock Finite-sample analyses for fully decentralized multi-agent
  reinforcement learning.
\newblock \textit{arXiv preprint arXiv:1812.02783}.

\bibitem[{Zou et~al.(2019)Zou, Xu and Liang}]{zou2019finite}
\textsc{Zou, S.}, \textsc{Xu, T.} and \textsc{Liang, Y.} (2019).
\newblock Finite-sample analysis for {SARSA} and {Q}-learning with linear
  function approximation.
\newblock \textit{arXiv preprint arXiv:1902.02234}.

\end{thebibliography}
